\documentclass{article}
\usepackage[margin=1in]{geometry}

\usepackage[utf8]{inputenc} 
\usepackage[T1]{fontenc}    
\usepackage{hyperref}       
\usepackage{url}            
\usepackage{booktabs}       
\usepackage{amsfonts}       
\usepackage{nicefrac}       
\usepackage{microtype}      
\usepackage[table, dvipsnames]{xcolor}

\usepackage[normalem]{ulem}
\usepackage{amsmath}

\usepackage{natbib}

\usepackage{enumitem}

\usepackage{multirow}
\usepackage{makecell}
\usepackage{threeparttable}

\usepackage{array}
\newcommand{\PreserveBackslash}[1]{\let\temp=\\#1\let\\=\temp}
\newcolumntype{C}[1]{>{\PreserveBackslash\centering}p{#1}}
\usepackage{hhline}
\newcolumntype{?}{!{\vrule width 1pt}}

\usepackage{amsmath}
\usepackage{amsthm}
\usepackage{amssymb}
\usepackage{thmtools}
\usepackage{thm-restate}
\usepackage[title]{appendix}
\usepackage{comment}
\usepackage{dsfont}
\usepackage{cleveref}

\usepackage{algorithm}
\usepackage{algpseudocode}

\usepackage{tikz}
\usepackage{mathtools}
\newcommand{\rdarrow}[1][-45]{%
  \mathrel{%
    \text{$
     \begin{tikzpicture}[baseline = -0.5ex]
       \node[inner sep=0pt,outer sep=0pt,rotate = #1] (a) at (0,0)  {$\xrightarrow{}$};
    \end{tikzpicture}
    $}%
  }%
}%

\definecolor{darkblue}{rgb}{0,0,0.95}
\algdef{SE}[IF]{CIF}{ENDIF}%
   [2][default]{\textcolor{darkblue}{\algorithmicif\ #2\ \algorithmicthen}}%
   {\algorithmicend\ \algorithmicif}%

\usepackage{thmtools}
\usepackage{thm-restate}

\usepackage{subfigure}
\usepackage{wrapfig}

\def\Regret{\mathrm{Reg}}

\newtheorem{theorem}{Theorem}
\newtheorem{lemma}{Lemma}
\newtheorem{proposition}[lemma]{Proposition}

\newtheorem{claim}[lemma]{Claim}
\newtheorem{remark}{Remark}

\newtheorem{example}{Example}

\newtheorem{definition}{Definition}

\usepackage{bbm}

\newtheorem{theorem-rst}[theorem]{Theorem}
\newtheorem{lemma-rst}[lemma]{Lemma}
\newtheorem{proposition-rst}[lemma]{Proposition}
\newtheorem{assumption-rst}[lemma]{Assumption}
\newtheorem{claim-rst}[lemma]{Claim}
\newtheorem{corollary-rst}[lemma]{Corollary}

\usepackage{mathtools}
\DeclarePairedDelimiter\br{(}{)}
\DeclarePairedDelimiter\brs{[}{]}
\DeclarePairedDelimiter\brc{\{}{\}}

\DeclarePairedDelimiter\abs{\lvert}{\rvert}

\DeclareMathOperator*{\argmax}{arg\,max}

\newcommand{\E}{\mathbb{E}}
\newcommand{\R}{\mathbb{R}}
\newcommand{\G}{\mathbb{G}}
\newcommand{\D}{\mathcal{D}}
\newcommand{\F}{\mathcal{F}}
\newcommand{\K}{\mathcal{K}}

\newcommand{\Acal}{\mathcal{A}}

\newcommand{\Ocal}{\mathcal{O}}
\newcommand{\Kinf}[1]{\K_{\mathrm{inf}}^{#1}}

\newcommand{\Narms}{K}

\newcommand{\dr}[1]{\Delta_{#1}}
\newcommand{\Ind}[1]{\mathds{1}\brc*{#1}}

\newcommand{\nq}[2]{n^q_{#1}\br*{#2}}
\newcommand{\np}[2]{n_{#1}\br*{#2}}

\newcommand{\rOb}[1]{R_{#1}}
\newcommand{\rwd}[1]{\mu_{#1}}
\newcommand{\rEst}[2]{\hat{\mu}_{#1}\br*{#2}}
\newcommand{\rOpt}{\mu^*}
\newcommand{\Bq}[1]{B^q(#1)}

\newcommand{\unu}{{\underline{\nu}}} 

\newcommand{\kl}{\mathrm{KL}}
\newcommand{\klBin}{\mathrm{kl}}

\usepackage[colorinlistoftodos]{todonotes}

\makeatletter
\newcommand{\printfnsymbol}[1]{%
  \textsuperscript{\@fnsymbol{#1}}%
}

\def\showComments{} 

\ifdefined\showComments
    \newcommand{\comN}[1]{\textcolor{blue}{\{Nadav: #1\}}}
    \newcommand{\comY}[1]{\textcolor{red}{\{Yonathan: #1\}}}
    \newcommand{\comS}[1]{\textcolor{green}{\{Shie: #1\}}}
\else
    \newcommand{\comN}[1]{}
    \newcommand{\comY}[1]{}
    \newcommand{\comS}[1]{}
\fi

\usepackage{selectp}

\title{Query-Reward Tradeoffs in Multi-Armed Bandits}

\author{
  Nadav Merlis\\
  Technion -- Institute of Technology\\
 \texttt{merlis@campus.technion.ac.il} \\
  \and
  Yonathan Efroni\\
  Meta, New York\\
  \and
  Shie Mannor \\
  Technion -- Institute of Technology\\
  Nvidia Research, Israel\\
  }

\date{}

\begin{document}

\maketitle

\begin{abstract}
We consider a stochastic multi-armed bandit setting where reward must be actively queried for it to be observed. We provide tight lower and upper problem-dependent guarantees on both the regret and the number of queries. Interestingly, we prove that there is a fundamental difference between problems with a unique and multiple optimal arms, unlike in the standard multi-armed bandit problem. We also present a new, simple, UCB-style sampling concept, and show that it naturally adapts to the number of optimal arms and achieves tight regret and querying bounds.
\end{abstract}

\section{Introduction}

In the stochastic multi-armed bandit (MAB) problem \citep{robbins1952some}, an agent repeatedly selects actions (`arms') from a finite set and obtains their rewards, generated independently from arm-dependent distributions. The goal is to maximize the cumulative obtained reward, or, alternatively, minimize the regret. This setting has been extensively studied and its extensions are ubiquitous. In particular, many of its variants suggest different feedback models -- structured \citep{chen2016combinatorialA}, partial,  delayed and/or aggregated \citep{pike2018bandits}. Then, once the feedback model is fixed, rewards are always observed through this model. 

Nevertheless, such models ignore a key element in sequential decision-making: whatever the feedback model is, asking to observe rewards usually comes at some cost. In some cases, the cost is evident from the setting, e.g., if experimentation is required or the reward is manually labeled by an expert. In other settings, the feedback cost is more subtle. For example, when humans supply the reward feedback, incessantly asking for it may aggravate them. Thus, it is natural to allow agents to decide whether they ask for feedback and design algorithms to ask for feedback with care.

One existing approach to regulate feedback requests is to limit reward querying by a hard querying-budget constraint \citep{efroni2021confidence}. However, this approach is very wasteful, as it encourages exhausting the budget whenever possible, regardless of the querying costs. Instead, we should aim to query only when we gain valuable information and avoid querying otherwise. A goal of our work is to develop an algorithm that has good performance, while not violating querying constraints and not querying for reward unnecessarily. The importance of such a goal is illustrated in the following example.
\begin{example}[Restaurant recommendation problem] \label{example: restaurant recommendation problem}
Consider a restaurant recommendation problem that learns from user-rankings (`feedback').
Repeatedly asking for rankings will annoy users, so it is natural to cap the ranking requests (`budget') with an initially low cap that gradually increases to allow learning. Yet, even when the agent is allowed to ask for user feedback, such queries harm the user experience, so we should avoid asking for feedback when not needed. 

\end{example}
\begin{example}[Medical treatment] \label{example: medical treatment}
Consider a doctor treating patients. While patients are expected to recover after a few days, whether the doctor administered drugs or sent them to rest, it would be beneficial to perform a follow-up exam. to see how effective the treatment was in retrospect. However, doing follow-ups to all patients cost valuable medical resources, so we would like to do perform them only if when expect to get information on the best course of treatment.
\end{example}

In this work, we study tradeoffs between feedback querying and reward gain in stochastic MAB problems. To do so, we derive both lower and upper bounds to this problem (summarized in \Cref{table:results}). 

\emph{Lower bounds.} In \Cref{section: lower bounds}, we study asymptotic and finite-sample lower bounds for reward querying. These reveal fundamental tradeoffs between the querying profile and regret. Interestingly, the lower bounds highlight a clear separation -- absent in the usual MAB setting -- between problems where the optimal (highest rewarding) arm is unique and problems with multiple optimal arms.
    
\emph{Upper Bounds.} In \Cref{section: upper bounds}, we present and analyze a simple algorithm for efficient reward querying -- the BuFALU algorithm -- and study its problem-dependent behavior. Notably, unlike prior work, we show that BuFALU naturally adapts to problems with a unique optimal arm and avoids wasting its reward queries unnecessarily. We conclude with a numerical comparison of BuFALU to other alternatives, which highlights its advantages.  

\setlength {\tabcolsep}{2pt}
\begin{table*}
\centering
\begin{threeparttable}
\caption{
Result summary. 
$\epsilon(t)>0$ is a sequence that controls querying, $\nq{t}{a}$ is the number of queries from arm $a$ before time $t$ and $\Bq{T}$ is the total queries before $T$. For simplicity, the lower bounds are stated for Gaussian arms of fixed variance and $\epsilon(t)$ is assumed to strictly decrease to $0$.
}
\label{table:results}
{\footnotesize 
\begin{tabular}{|c|C{1.5cm}|c|c|}\hline
    \textbf{Regime} &  \textbf{Scenario} & \textbf{Lower Bounds} &\textbf{Upper Bounds} {\small (BuFALU)}\\ \hline 
    \multirow{5}{1.5cm}[-2ex]{\centering\textit{Asymptotic Bounds}} &
    Suboptimal Arms& 
    $\forall a\notin\Acal^*, \E\brs*{\nq{T}{a}} = \Omega\br*{\frac{\ln T}{\dr{a}^2}}^\dag$ & 
    \\ 
    \hhline{~---}
    & 
    \multirow{2}{1.5cm}[-.5ex]{\centering Unique Optimal} & 
    \multirow{2}{*}[-.5ex]{$\E\brs*{\nq{T}{a^*}} = \Omega\br*{\frac{\ln T}{\dr{\min}^2}}^\dag$}  & 
    $E\brs*{\Bq{T}}=\Ocal\br*{\sum_{a\notin\Acal^*}\frac{\ln T}{\dr{a}^2} + \frac{\ln T}{\dr{\min}^2}}$  \\ 
    \hhline{~~~~}
    & 
    & 
    & 
    $\Regret(T)=\Ocal\br*{\sum_{a\notin\Acal^*}\frac{\ln T}{\dr{a}}}^\ddag$  \\ 
    \hhline{~---}
    & 
    \multirow{2}{1.5cm}[-1.5ex]{\centering Multiple Optimal} & 
    \multirow{2}{*}[-.5ex]{$\E\brs*{\nq{T}{a^*}} = \omega\br*{\ln T}^\dag$}  & 
    $E\brs*{\Bq{T}}=\Ocal\br*{\sum_{a\notin\Acal^*}\frac{\ln T}{\dr{a}^2} + \abs*{\Acal^*}\frac{\ln T}{\epsilon(T)^2}}$  \\ 
    \hhline{~~~~}
    & 
    & 
    & 
    $\Regret(T)=\Ocal\br*{\sum_{a\notin\Acal^*}\frac{\ln T}{\dr{a}}}^\ddag$  \\ 
    \hline
    \multirow{2}{1.5cm}[-.25ex]{\centering\textit{Scarce Querying Bounds}} &
    $\E\brs*{\nq{T}{a^*}}$ & 
    $\Regret(T)= $  & 
    $B_a(t)\triangleq \E\brs*{\nq{T}{a}}=\Ocal\br*{\frac{\ln T}{\epsilon(T)^2}}$  \\ 
    \hhline{~~~~}
    & 
    $\le B_a(T)$ & 
    $\Omega\br*{\sum_{a\notin\Acal^*}\frac{\dr{a}}{\Narms}B_a^{-1}\br*{\frac{1}{\Narms\dr{a}^2}}}^\star$& 
    $\Regret(T)=\Ocal\br*{\sum_{a\notin\Acal^*}\dr{a}B_a^{-1}\br*{\frac{\ln T}{\dr{a}^2}}}^\S$\\ 
    \hline
\end{tabular}
}\vspace{-0.025cm}
\begin{tablenotes}
      \small 
      \item \hspace{-0.45cm} 
      $^\dag$ \Cref{theorem:lower bound asymptotic}
      \,\,  $^\ddag$ \Cref{theorem:upper bound}, $\bar{N}\br*{T,\dr{}}$ term 
      \,\, $^\star$ \Cref{corollary:lower bound finite regret}
      \,\,  $^\S$ \Cref{theorem:upper bound}, $L_\epsilon\br*{T,\dr{}}$ term 
\end{tablenotes}
\end{threeparttable}
\end{table*}
\setlength {\tabcolsep}{6pt}

\textbf{Related Work.} 
Surprisingly, the tradeoff between reward querying and performance in sequential learning has been, to a large extent, unexplored. 
To our knowledge, our work is the first to study problem-dependent anytime guarantees for the MAB problem when reward must be actively queried. 

Related to our work is the sequential budgeted learning framework in~\citep{efroni2021confidence}. There, queries must follow a hard time-dependent budget constraint. However, the suggested algorithm wastefully uses its entire budget. Unlike their algorithm, we develop an algorithm that can adapt to the problem instance and querying reward feedback for $O(\log(T))$ rounds. Moreover, we provide problem-dependent regret and querying guarantees, while~\citep{efroni2021confidence} only derived problem-independent regret guarantees. 
Another related setting is the problem of MAB with additional observations~\citep{yun2018multi}. There, agents can ask for observations from arms that were not played, but such queries are limited by a time-dependent budget. The rewards of played arms are always observed. In contrast, we limit the number of observations from \emph{played} arms. Thus, the settings are complementary. Previous works also studied adversarial settings where observing rewards incurs a cost \citep{seldin2014prediction}, or where the interaction ends when the budget is exhausted \citep{badanidiyuru2013bandits}. These models greatly differ from ours, as we work in a stochastic setting without explicit querying costs, and the interaction length is unaffected by querying constraints.

\section{Setting}
In our setting, an agent (\emph{bandit strategy}) selects one of $\Narms$ arms (\emph{actions}), each characterized by a reward distribution $\nu_a$ of expectation $\rwd{a}$. We refer to $\unu=\brc*{\nu_a}_{a\in\brs*{\Narms}}$ as the \emph{bandit instance}, where $\brs*{\Narms}\triangleq\brc*{1\dots,\Narms}$, and denote the set of all valid reward distributions by $\D$ (e.g., all distributions supported by $[0,1]$). When not clear from the context, we denote the expectation w.r.t. a specific bandit instance by $\E_{\unu}\brs*{\cdot}$. The optimal reward of a bandit instance is $\rOpt=\max_{a\in\brs*{\Narms}}\rwd{a}$, and we define the set of all optimal arms as $\Acal_* = \brc*{a: \rwd{a}=\rOpt}$. The suboptimality gap of an arm $a$ is $\dr{a}=\rOpt-\rwd{a}$, and we denote the maximal and minimal gaps by $\dr{\max}=\max_a\dr{a}$ and $\dr{\min}=\min_{a:\dr{a}>0}\dr{a}$.

At each round $t\ge1$, the agent plays a single arm $a_t\in\brs*{\Narms}$. Then, the arm generates a reward $\rOb{t}\sim\nu_{a_t}$, independently at random of other rounds. However, to observe this reward, the agent must actively query it by setting $q_t=1$; otherwise, the reward is not observed and $q_t=0$. For brevity, we say that the agent always observes $R_t\cdot q_t$. We denote the number of times an arm was played up to round $t$ by $\np{t}{a} = \sum_{s=1}^t\Ind{a_s=a}$, where $\Ind{\cdot}$ is the indicator function, and the number of times it was queried by $\nq{t}{a}=\sum_{s=1}^t\Ind{a_s=a,q_s=1}$. Notice that for an arm to be queried, it must first be played, and thus $\nq{t}{a}\le\np{t}{a}$. We similarly denote the total number of queries up to time $t$ by $\Bq{t} = \sum_{a=1}^{\Narms} \nq{t}{a}$ and sometimes limit it by some querying budget $B(t)$  (similarly to \citep{efroni2021confidence}). We assume that querying an arm incurs a unit querying-cost but remark that all bounds can be extended to arm-dependent costs. Finally, we define the empirical mean of an arm $a$, based on \emph{observed} samples up to round $t$, by $\rEst{t}{a}=\frac{1}{\nq{t}{a}}\sum_{s=1}^t R_s\Ind{a_s=a,q_s=1}$ and say that $\rEst{t}{a}=0$ if $\nq{t}{a}=0$. 

In this work, we evaluate agents by two metrics: by the \emph{expected number of reward queries}, $\E\brs*{\Bq{T}}$, and by the \emph{regret}, namely, the expected difference between the optimal reward and the reward of the played arms: $\Regret(T)
    = \E\brs*{\sum_{t=1}^T \br*{\rOpt - \rwd{a_t}}} = \E\brs*{\sum_{t=1}^T \dr{a_t}}$. 
Particularly, we analyze the tradeoffs between the two quantities and the effect of reducing the querying on the regret.

\section{Failures of Existing Approaches}
In this section, we elaborate on the failures of existing approaches to efficiently query rewards in the stochastic MAB setting in the presence of query restriction. 

\textbf{Failure of best-arm identification approach.}
The simplest (and most na\"ive) approach to incorporate reward querying to sequential decision-making is to query rewards using a best-arm identification (BAI) algorithm \citep{gabillon2012best,kalyanakrishnan2012pac}. Such algorithms usually adapt themselves to efficiently query arms. Then, it is intuitive to explore first using a BAI algorithm and then exploit its recommended arm. However, this is nontrivial in an anytime setting, where the interaction horizon is unknown. For simplicity, assume that a hard querying budget is given at its whole at the beginning of the interaction ($B(t)=B$ for all $t\geq 1$). For interactions of length $T\gg B$, we would like to use all the budget for exploration, while for $T<B$, we should save rounds for exploitation. In general, the number of exploration rounds must be adaptively determined, and it is unclear how to do so with off-the-shelf algorithms. The same holds for the accuracy and confidence of BAI algorithms. This is even less clear with time-dependent adversarial budgets, which prevent standard doubling tricks.

\textbf{Failure of confidence-budget-matching (CBM).}
Alternatively, one might use the confidence-budget matching mechanism \citep[CBM,][]{efroni2021confidence}, which is designed for anytime settings and time-dependent budgets. However, the algorithm wastefully and unnecessarily expends its querying budget. For example, when $B(t)\gg t$ the algorithm reduces to be UCB1 and queries every round. In contrast, in problems with a unique optimal arm, it is well-known that $\Ocal\br*{{\Narms\ln t}/{\dr{\min}^2}}$ queries are sufficient to identify the optimal arm with high enough probability of $\sim1/t$ \citep{gabillon2012best}. Thus, CBM \emph{should not} query every round (even though allowed), but rather on a logarithmic number of rounds, and fails to conserve queries.


\section{Lower Bounds}
\label{section: lower bounds}
In this section, we present lower bounds on the number of queries we must take from arms for the agent to `behave well'. Importantly, we will see a distinctly different behavior of the lower bounds when there is a unique or multiple optimal arms. This will encourage us to design an algorithm that adapts to both scenarios, as we do in the next section.

We require a few additional notations. Let $\Acal_*(\unu)$ be the set of optimal arms in instance $\unu$. Also, denote the Kullback-Leibler (KL) divergence between any two distributions $\nu_a$ and $\nu'_a$ by $\kl(\nu_a,\nu'_a)$ and the KL divergence between two Bernoulli random variables of expectations $p,q$ by $\klBin(p,q) = p\ln\frac{p}{q} + (1-p)\ln\frac{1-p}{1-q}$. 
Then, we define
\begin{align*}
    &\Kinf{+}(\nu,\mu,\D) = \inf\brc*{\kl(\nu,\nu'): \nu'\in\D, \E\brs*{\nu'}>\mu}\;, \\
    &\Kinf{-}(\nu,\mu,\D) = \inf\brc*{\kl(\nu,\nu'): \nu'\in\D, \E\brs*{\nu'}<\mu}\;, 
\end{align*}
where the infimum over an empty set is $+\infty$. If the infimum is zero, we let its inverse be $+\infty$. Intuitively, $\Kinf{+}(\nu_a,\mu,\D)$ represents the distance between a distribution $\nu_a$ to the closest distribution in $\D$ of expectation higher than $\mu$. Then, to distinguish between $\nu_a$ and \emph{any} distribution of expectation larger than $\mu$, we require a number of samples that is inversely proportional to $\Kinf{+}(\nu_a,\mu,\D)$. Similarly, $\Kinf{-}$ can be related to the closest distribution of a lower expectation.

Finally, let $\brc*{U_t}_{t\ge0}$ be a sequence of i.i.d. uniform random variables that encompass the internal randomization of agents. Then, a bandit strategy maps the history $I_t=\br*{U_0,a_1,q_1,\rOb{1}\cdot q_1,U_1,\dots,a_t,q_t,\rOb{t}\cdot q_t,U_t}$ into a next action $a_{t+1}$ and querying rule $q_{t+1}$.

\subsection{Asymptotic Lower Bounds}
Probably the most common assumption for a bandit strategy is \emph{consistency}, namely, that for \emph{any} bandit instance, the regret of the strategy is asymptotically sub-polynomial. 
\begin{definition}
A bandit strategy is called \emph{consistent} w.r.t. $\D$ if for any instance $\unu\in\D$, any suboptimal arm $a\notin\Acal_*(\unu)$ and any $\alpha\in(0,1]$ it holds that $\np{T}{a}=o(T^\alpha)$.
\end{definition}
For consistent strategies, we show that the following holds:
\begin{restatable}{theorem-rst}{lowerBoundAsymptotic}
\label{theorem:lower bound asymptotic}
Let $\nq{\infty}{a}=\liminf_{T\to\infty}\frac{\E_{\unu}\brs*{\nq{T}{a}}}{\ln T}$ for any $a\in\brs*{\Narms}$. If the bandit strategy is consistent w.r.t. $\D$, then for any bandit instance $\unu\in\D$, the following hold: 
\begin{enumerate}
    \item For any suboptimal arm $a\notin\Acal_*(\unu)$, $\nq{\infty}{a} \ge 1/{\Kinf{+}(\nu_a,\rOpt,\D)}$.
    \item Assume that $a^*$ is the unique optimal arm and let $a\ne a^*$ be a suboptimal arm. Also, denote the maximal suboptimal reward by $\mu^s = \max_{a\ne a^*}\rwd{a}$. Then $\nq{\infty}{a^*} \ge 1/{\Kinf{-}(\nu_{a^*},\mu^s,\D)}$ and for any $\mu\in\brs*{\mu^s,\rOpt}$, 
    \begin{align} 
        &\nq{\infty}{a} + \nq{\infty}{a^*} 
        \ge 1/{\max\brc*{\Kinf{+}(\nu_a,\mu,\D),\Kinf{-}(\nu_{a^*},\mu,\D)}} \;.\label{eq: lower bound unique optimal}
    \end{align}
    \item Assume that there are at least two optimal arms. Moreover, assume that (1) for all optimal arms $a\in\Acal_*(\unu)$, $\Kinf{-}(\nu_a,\rOpt,\D)=0$, or, alternatively, (2) there are at least two optimal arms for which $\Kinf{+}(\nu_a,\rOpt,\D)=0$. Then $\nq{\infty}{a}=\infty$ for some optimal arm $a\in\Acal_*(\unu)$.
\end{enumerate}
\end{restatable}
The proof is in \Cref{appendix:dependent lower bound}. Notice that the bounds of the theorem are \emph{asymptotic}, namely, for large enough $T$, an action $a$ is roughly queried $\nq{\infty}{a}\cdot\ln T$ times. Importantly, recall that $\nq{t}{a}\le \np{t}{a}$, so all results also hold for the number of plays. The theorem is divided into three parts. The first part is a natural extension of the classical problem-dependent lower bound for MABs \citep{lai1985asymptotically,burnetas1996optimal,garivier2019explore} and emphasizes that it is not enough to sufficiently \emph{play} suboptimal arms, but we rather must sufficiently \emph{query} them. The second and third parts discuss the querying requirements from \emph{optimal arms} when there is a unique or multiple optimal arms, respectively. This comes in contrast to the classical lower bounds that disregard querying, as playing optimal arms does not incur regret and can thus be ignored. 

Intuitively, when there is a \emph{unique optimal arm} $a^*$, the result first states that it must be distinguished from the highest suboptimal arm $a$. Yet, \Cref{eq: lower bound unique optimal} implies that by itself, this does not suffice. Instead, for any subotpimal arm $a$, both $a$ and $a^*$ must be sufficiently queried to \emph{separate} them; namely, identifying that $a^*$ is better than $a$. To see this, consider the (typical case) where $\Kinf{+}(\nu,\mu,\D)$ and $\Kinf{-}(\nu,\mu,\D)$ are continuous in $\mu$ and equal zero if $\E\brs*{\nu}=\mu$. Also, notice that $\Kinf{+}$ ($\Kinf{-}$) decreases (increases) in $\mu$. Then, there exists $\mu_0\in\br*{\rwd{a},\rOpt}$ such that $\Kinf{+}(\nu_a,\mu_0,\D)=\Kinf{-}(\nu_{a^*},\mu_0,\D)$, and this choice maximizes the r.h.s. of \eqref{eq: lower bound unique optimal}. In this case, a reasonable way to match the lower bound is to ensure that both $\nq{\infty}{a}$ and $\nq{\infty}{a^*}$ are roughly equal to $\br{\Kinf{+}(\nu,\mu_0,\D)}^{-1}$. This  separates the optimal arm $a^*$ from distributions $\brc*{\nu\in\D:\E\brs*{\nu}<\mu_0}$ and the suboptimal arm $a$ from $\brc*{\nu\in\D:\E\brs*{\nu}>\mu_0}$. Concretely, if $\D$ is the set of all Gaussian distributions of unit variance, then $\mu_0 = \rwd{a}+\dr{a}/2 = \rOpt-\dr{a}/2$, and we should estimate both $a$ and $a^*$ up to a precision of $\dr{a}/2$.

Finally, the last part of \Cref{theorem:lower bound asymptotic} treats problems with \emph{multiple optimal arms}. Specifically, it states that a logarithmic queries might suffice only if $\Kinf{+}(\nu_a,\rOpt,\D)=0$ for \emph{at most} a single optimal arm and $\Kinf{-}(\nu_a,\rOpt,\D)>0$ for \emph{at least} one optimal arm. Notably, for standard distribution sets $\D$, for any value of $\rOpt$, either $\Kinf{+}(\nu,\rOpt,\D)=0$ or $\Kinf{-}(\nu,\rOpt,\D)=0$ for all $\nu\in\D$ with $\E\brs*{\nu}=\rOpt$ -- at least one of the conditions hold, so multiple optimal arms should be queried super-logarithmically. Intuitively, when either of the conditions hold, it is impossible to determine if arms are strictly optimal or near-optimal with arbitrarily small gaps using logarithmic queries. We illustrate that both conditions are necessary by a concrete example of a distribution set $\D$ in \Cref{appendix:lower bound example}.

To summarize, \Cref{theorem:lower bound asymptotic} draws a remarkable distinction between problems with a unique optimal arm -- where log-querying suffices -- and ones with multiple optimal arms, where consistent strategies must query super-logarithmically. Similar phenomena do not exist in standard MAB problems and a key contribution of this theorem is the refined characterization of the conditions for it to occur.  

We end this section by remarking that our problem-dependent lower bounds can also be applied to problems with (implicit or explicit) querying costs. Due to space limitations, we defer this discussion to \Cref{appendix:querying costs}. There, we show that algorithms must avoid unnecessary reward queries, as otherwise, even the most minuscule (fixed) querying cost would result in a linear regret. Such scenario clearly highlights the advantage in reducing the feedback queries as much as possible.

\subsection{Lower Bounds for Scarce Querying}
We now discuss the best possible performance in the limit of scarce querying. To derive these bounds, we require the bandit strategy to be \emph{better-than-uniform}:
\begin{definition}
A bandit strategy is called \emph{better-than-uniform} on $\D$ if for any instance $\unu\in\D$ and any $T\ge1$, it holds that $\sum_{a\in\Acal_*(\unu)} \E_{\unu}\brs*{\np{T}{a}} \ge \frac{\abs*{\Acal_*(\unu)}}{\Narms}T$.
\end{definition}
This definition is weaker than the one in \citep{garivier2019explore}, which requires \emph{all} optimal arms to be played at least $T/\Narms$ times in expectation. This does not affect the result but allows focusing on playing a \emph{specific} optimal arm, which better fits querying-aware settings. For such strategies, the following holds (see proof in \Cref{appendix:lower bound linear}, which resembles the one of Theorem 2 in \citealt{garivier2019explore}):
\begin{restatable}{proposition-rst}{lowerBoundLinear}
\label{prop:lower bound linear}
For any bandit instance $\unu\in\D$, any strategy that is better than uniform on $\D$, any arm $a\notin\Acal_*(\unu)$ and any $T\ge1$, it holds that
\begin{align*}
   &\E_{\unu}\brs*{\np{T}{a}} 
   \ge \frac{T}{\Narms}\br*{1 - \sqrt{2\min\brc*{\Narms\E_{\unu}\brs*{\nq{T}{a}},T}\Kinf{+}(\nu_a,\rOpt,\D)}}.
\end{align*}
If $\Kinf{+}(\nu_a,\rOpt,\D)=\infty$, we define the r.h.s. to be zero. Specifically, if $\Kinf{+}(\nu_a,\rOpt,\D)<\infty$ and $\E_{\unu}\brs*{\nq{T}{a}}\le \frac{1}{8\Narms \Kinf{+}(\nu_a,\rOpt,\D)}$, then $\E_{\unu}\brs*{\np{T}{a}} \ge {T}/\br*{2\Narms}$.
\end{restatable}
This proposition implies that until we sufficiently query arms, the regret must be linear, and has an important implication on the regret when feedback is limited (see proof in \Cref{appendix:lower bound linear}): \begin{restatable}{corollary-rst}{lowerBoundLinearRegret}
\label{corollary:lower bound finite regret}
For any $t\ge1$ and $a\in\brs*{\Narms}$, denote $B_a(t) = \E_{\unu}\brs*{\nq{t}{a}}$. Also, let $B_a^{-1}(x) = \sup\brc*{t\in\mathbb{N}: B_a(t)\le x}$ for $x\ge B(1)$ and otherwise $B_a^{-1}(x)=0$. Then, for any instance $\unu\in\D$ and any better-than-uniform strategy on $\D$,
\begin{align*}
    &\forall T\ge \max_{a\notin\Acal_*(\unu)}B_a^{-1}\br*{\frac{1}{8\Narms \Kinf{+}(\nu_a,\rOpt,\D)}},
    \quad\Regret(T)\ge \sum_{a\notin\Acal_*(\unu)} \frac{\dr{a}}{2\Narms}B_a^{-1}\br*{\frac{1}{8\Narms \Kinf{+}(\nu_a,\rOpt,\D)}}.
\end{align*}
\end{restatable}
We remark that the same result also holds when we only have bounds on the number of queries ($\E_{\unu}\brs*{\nq{t}{a}}\le B_a(t)$ for some positive nondecreasing $B_a(t)$), e.g., per-arm querying budget. When querying is scarce, this might lead to an exponential lower bound, as shown in the following example:
\begin{example}[Querying profiles]
Let the set of all valid arm distributions $\D$ be the set of Gaussian distributions with unit variance. Then, it can be easily verified that for all $a\in\brs*{\Narms}$, 
$\Kinf{+}(\nu_a,\rOpt,\D) = \frac{(\rOpt-\rwd{a})^2}{2}\!=\!\frac{\dr{a}^2}{2}$. Thus, by \Cref{corollary:lower bound finite regret}, for any $T\!\ge\! B_a^{-1}\br*{\frac{1}{4\Narms\dr{\min}^2}}$, the regret is lower bounded by 
\begin{equation*}
    \Regret(T) \ge \sum_{a\notin\Acal_*(\unu)}\frac{\dr{a}}{2\Narms}B_a^{-1}\br*{\frac{1}{4\Narms \dr{a}^2}}\enspace.
\end{equation*}
Consider the following querying profiles:
\begin{enumerate}
    \item If the per-arm queries are polynomial $B_a(T)\!=\!T^\alpha/\Narms$ for $\alpha\!\in\![0,1]$, then the lower bound is 
    $\Regret(T) \!=\! \Omega\br*{\sum_{a\notin\Acal_*(\unu)}\dr{a}^{1-\frac{2}{\alpha}}}$, 
    i.e., the lower bound is inversely polynomial in the gaps.
    \item If the per-arm queries are poly-log and gap-oblivious, i.e., $B_a(T) = (\ln T)^\alpha/\Narms$, then the bound is exponential in the inverse-gaps: $\Regret(T) = \Omega\br*{\sum_{a\notin\Acal_*(\unu)}\dr{a}\exp\brc*{\dr{a}^{-2/\alpha}}}$.
    \item If the per-arm queries are logarithmic but gap-aware, i.e., $B_a(T)= \ln T/\dr{a}^2$, then the lower bound is linear in the gaps: $\Regret(T) = \Omega\br*{\sum_{a\notin\Acal_*(\unu)}\dr{a}/\Narms}$. Typically, $\dr{a}\le1$, and the bound is effectively constant. Notably, it suffices to know a lower bound on the gaps: if $\dr{a}\ge\epsilon$ for all $a\notin\Acal_*(\unu)$, then a total of $B(T)=\br*{\Narms\ln T}/\epsilon^2$ queries is adequate.   
\end{enumerate}
\end{example}
The last bound hints that logarithmic queries suffice, so it might be appealing to pre-allocate a logarithmic querying budget. However, agents typically do not know the gap values, and allocating a logarithmic querying budget might lead to a prohibitively large regret -- exponential in the gaps. Notably, this is the case even when the optimal arm is unique. On the other hand, if the agent does know the minimal gap, it can be used to drastically reduce querying, even with multiple optimal arms.

\section{Upper Bounds}
\label{section: upper bounds}

\begin{algorithm}[t]
\caption{\underline{Bu}dget-\underline{F}eedback \underline{A}ware \underline{L}ower \underline{U}pper  Confidence Bound (BuFALU)} \label{alg: BuFALU}
\begin{algorithmic}[1]
\State \textbf{Define: } $UCB_t(a) = \rEst{t-1}{a} + \sqrt{\frac{3\ln t}{2\nq{t-1}{a}}}$, $LCB_t(a) = \rEst{t-1}{a} - \sqrt{\frac{3\ln t}{2\nq{t-1}{a}}}$ (`Hoeffding')
\For{$t=1,...,\Narms$}
\State Play $a_t = t$ and ask for feedback ($q_t=1$); observe $\rOb{t}$ and update $\nq{t}{a_t}, \rEst{t}{a_t}$
\EndFor
\For{$t=\Narms+1,...,T$}
\State Observe $\epsilon(t)\ge0$
\State Set $l_t, u_t, c_t$ according to \Cref{eq:arm choice}
\If{$UCB_t(u_t)\le LCB_t(l_t)$ or $UCB_t(c_t)-LCB_t(l_t)\le \epsilon(t)$} 
    \State Play $a_t=l_t$ and do not ask for feedback ($q_t=0$) 
\Else
    \State Play $a_t=c_t$ and ask for feedback ($q_t=1$); observe $\rOb{t}$ and update $\nq{t}{a_t}, \rEst{t}{a_t}$
\EndIf
\EndFor
\end{algorithmic}
\end{algorithm} 

In the last section, we showed that algorithms should allocate enough (possibly super-logarithmic) reward queries to optimal arms. We also showed that the regret is linear until we query feedback from suboptimal arms for sufficiently many times. Further, in the important case of a unique optimal arm, we showed that the optimal arm must be separated from all suboptimal arms, but argued it should still be logarithmically queried. 
We now leverage these insights to design an algorithm that conservatively queries rewards. For simplicity, we focus on problems with rewards in $[0,1]$. In \Cref{appendix: upper bounds}, we prove our results for any confidence intervals under mild assumptions (including Bernstein bounds).

To be more feedback-conscious, we adopt a confidence-based approach. Namely, we assume that each arm  $a$ is equipped with a confidence interval of width $CI_t(a) = UCB_t(a)-LCB_t(a)$, such that the true mean of arms is within the confidence interval with high probability, i.e., $\rwd{a}\in\brs*{LCB_t(a),UCB_t(a)}$. Importantly, confidence intervals can be used to \emph{quantify} the suboptimality of actions. For example, for any arm $a$, we can upper bound its suboptimality gap w.h.p. by 
\begin{align*}
    \dr{a} = \rOpt - \rwd{a} 
    &\le UCB_t(a^*) - LCB_t(a) 
    \le \max_{a'} UCB_t(a') - LCB_t(a)\enspace,
\end{align*}
where $a^*$ is an optimal arm. Then, if $a\!\in\!\argmax_{a'}UCB_t(a')$, we get $\dr{a}\!\le\! CI_t(a)$. In general, by controlling the widths of confidence intervals, we can identify good arms and bound their sub-optimality. However, narrowing these intervals requires feedback, which we assume to be limited.

Our approach carefully shrinks the confidence intervals, with the goal to separate a unique optimal arm (whenever such exists), while controlling the number of reward queries. We divide our algorithm into two steps -- action-selection and confidence-control. For action-selection, inspired by the lower bounds, we aim to \emph{separate} the confidence interval of a unique optimal arm $a^*$ from the rest of the arms, i.e., $LCB_t(a^*)\ge UCB_t(a)$ for all suboptimal $a$. If we succeeded in doing so, we can safely declare that $a^*=\argmax_a LCB_t(a)\triangleq l_t$ and play $l_t$ without reward querying. Letting $u_t \in \argmax_{a\ne l_t} UCB_t(a)$, this case happens if $UCB_t(u_t)\le LCB_t(l_t)$. When this condition does not hold, we would like to separate $u_t$ from $l_t$ by actively shrinking the confidence intervals of both arms. An efficient way to do so is to play (and query) the arm with the wider confidence interval:
\begin{align}
    &c_t \in\argmax_{a\in\brc*{u_t,l_t}} CI_t(a) \quad \mathrm{for} \quad l_t\in\argmax_a LCB_t(a), \quad u_t\in\argmax_{a\ne l_t} UCB_t(a)\;. \label{eq:arm choice}
\end{align}
Yet, this separation might be costly in queries and is impossible if there are multiple optimal arms. 
To avoid this, we regulate the number of queries by controlling the widths of the confidence intervals. Let $\epsilon(t)$ be a variable that controls the widths of the confidence intervals. We allow our algorithm to query for reward only if $UCB_t(c_t)-LCB_t(l_t)>\epsilon(t)$; namely, we query only if $c_t$ might be distinctively better than $l_t$, which is a natural candidate for the optimal action. For generality, we allow $\epsilon(t)$ to be externally controlled (be chosen adversarially), to allow external effects on the querying rule (e.g., querying budget), but in practice, it can be oftentimes chosen by the algorithm designer.

Lastly, if $UCB_t(c_t)-LCB_t(l_t)\le\epsilon(t)$, we revert to playing the `default action' $l_t$ without reward querying. We later show that doing so is safe, in a sense that $\dr{l_t}\le \epsilon(t)$ (\Cref{lemma: gap-query} in \Cref{appendix: problem dependent proof}). Combining both playing and querying schemes leads to the Budget-Feedback Aware Lower-Upper Confidence Bound algorithm (BuFALU), depicted in \Cref{alg: BuFALU}. 
Before stating the regret and querying guarantees of BuFALU, we present two notable quantities on which the bounds will depend:
\begin{align*}
    &L_\epsilon(T,\dr{}) = \sum_{t=1}^T\Ind{\epsilon(t)\ge\dr{}}  \quad \mathrm{and} \quad 
    \bar{N}\br*{T,\dr{}} = \max_{t\in\brs*{T}}\frac{6\ln t}{\max\brc*{\dr{}^2,\epsilon^2(t)}}\enspace.
\end{align*}
The first quantity $L_\epsilon(T,\dr{})$ counts the number of rounds that $\epsilon(t)$ exceeds a fixed confidence level $\dr{}$ until time $T$. To understand it, keep in mind that when a reward is not queried, we play $a_t = l_t$ with the guarantee that $\dr{l_t}\le \epsilon(t)$ (by \Cref{lemma: gap-query}). In turn, at such rounds, we might play any arm $a$ for which $\dr{a}\le\epsilon(t)$. Thus, $L_\epsilon(T,\dr{a})$ represents the maximal number of rounds that a suboptimal arm $a$ can be played when reward are not queried. A notable case is when $\epsilon(t)$ is nonincreasing, and then $L_\epsilon(T,\dr{})$ can be conveniently bounded by $L_\epsilon(T,\dr{})\le \epsilon^{-1}(\dr{}) \triangleq \sup\brc*{t\ge1: \epsilon(t)\ge \dr{}}$. 

To understand $\bar{N}$, recall that $\nq{t-1}{a} = \frac{6\ln t}{\mu^2}$ is the number of samples required for the confidence interval $CI_t(a)$ to be smaller than $\mu$. Thus, $\bar{N}\br*{T,\dr{}}$ represents the maximal number of queries required to shrink the confidence intervals to a confidence level of $\max\brc*{\epsilon(t),\dr{}}$. In particular, for any suboptimal arm $a$ and $\dr{}=\dr{a}$, this number of queries suffices to either identify that $a$ is suboptimal (confidence smaller than $\dr{a}$) or to stop sampling it due to the confidence constraint. Importantly, if $\epsilon(t) = \sqrt{6\Narms\ln t/B(t)}$ for a positive nondecreasing budget $B(t)$ (or, alternatively, if $\epsilon(t)$ is  nonincreasing), see that $\bar{N}\br*{T,\dr{}} = \frac{6\ln T}{\max\brc*{\dr{}^2,\epsilon^2(T)}}$. Then, we get $\bar{N}(T,0)\le B(T)/\Narms$, so this term can ensure that the budget constraints are never violated. We now state the problem-dependent bounds for BuFALU (see proof in \Cref{appendix: problem dependent proof}). The bounds hold for oblivious adversarially chosen $\epsilon(t)$. If the adversary is adaptive, or $\epsilon(t)$ is stochastic, the results hold by taking an expectation on all bounds. 
\begin{theorem}
\label{theorem:upper bound}
Assume that the rewards are bounded in $[0,1]$. Also, let $T\ge1$ and assume that $\brc*{\epsilon(t)}_{t\in\brs*{T}}$ is some nonnegative sequence. Then, when running \Cref{alg: BuFALU}, the following hold:
\begin{enumerate}
    \item For all $a\in\brs*{\Narms}$, it holds that $\nq{T}{a} \le \bar{N}\br*{T,0} +1$.
    \item If there are multiple optimal arms ($\abs{\Acal_*}>1$), then
    \begin{align*}
        &\Regret(T)\le \sum_{a\notin\Acal_*} \dr{a}\br*{\bar{N}\br*{T,\dr{a}}+L_\epsilon(T,\dr{a})} +3\Narms\dr{\max}, \\
        &\E\brs*{\Bq{T}} \le \sum_{a\notin\Acal_*} \bar{N}\br*{T,\dr{a}}+ \abs{\Acal_*}\bar{N}\br*{T,0} +  3\Narms.
    \end{align*}\vspace{-.2cm}
    \item If the optimal arm $a^*$ is unique, then 
    \begin{align*}
        &\Regret(T)\le \sum_{a\ne a^*} \dr{a}\br*{\bar{N}\br*{T,\frac{\dr{a}}{2}}+L_\epsilon(T,\dr{a})}+3\Narms\dr{\max}, \\
        &\E\brs*{\Bq{T}} \le \sum_{a\ne a^*} \bar{N}\br*{T,\frac{\dr{a}}{2}} + \bar{N}\br*{T,\frac{\dr{\min}}{2}} +  3\Narms.
    \end{align*}\vspace{-.2cm}
\end{enumerate}
\end{theorem}
Moreover, BuFALU enjoys the following \emph{problem-independent} bound (see proof in \Cref{appendix: problem independent proof}):
\begin{proposition}
\label{prop: problem independent upper}
When running \Cref{alg: BuFALU} with any sequence $\epsilon(t)\ge0$, for any $T\ge1$, it holds that $\Regret(T)\le 4\sqrt{6\Narms T\ln T} +\sum_{t=1}^T\epsilon(t) + 3\Narms\dr{\max}$.
\end{proposition}

\subsection{Discussion and Comparisons}
Given a positive nondecreasing (possibly adversarial) querying budget $B(t)$, a natural parameter choice is $\epsilon(t)\!=\!\sqrt{6\Narms\ln t/B(t)}$. Then, the first part of \Cref{theorem:upper bound} provides strict (almost sure) querying guarantee of $B(T)+\Narms$ queries, while \Cref{prop: problem independent upper} achieves a regret bound of $\Regret(T)=\Ocal\br*{\sqrt{\Narms T\ln T} + \sum_{t=1}^T \sqrt{\frac{\Narms\ln T}{B(T)}}}$ -- the same bounds as CBM-UCB \citep{efroni2021confidence}. On the other hand, BuFALU also enjoys problem-dependent regret and querying guarantees. In particular, when the optimal arm is unique, the number of queries is logarithmic -- usually far less than the allocated budget.

On the other hand, other choices of $\epsilon(t)$ allow flexibility in the algorithm design. Setting $\epsilon(t)=0$ represents rounds with free queries, while $\epsilon(t)>CI_t(c_t)$ blocks querying. Also, if the minimal gap $\dr{\min}$ is known, then setting $\epsilon(t)\to\dr{\min}^-$ allows logarithmic querying (even with multiple optimal arms) while retaining logarithmic regret. 

Next, we compare \Cref{theorem:upper bound} to the lower bounds of \Cref{section: lower bounds}, and for simplicity, assume that $\epsilon(t)$ is nonincreasing. First, the per-arm query bound of  $B_a(t)=\Ocal\br*{\ln t/\epsilon^2(t)}$ implies that $\epsilon^{-1}\br*{\dr{a}}\approx B_a^{-1}\br*{\ln T/\dr{a}^2}$. Then, by bounding $L_\epsilon(t,\dr{})\le \epsilon^{-1}(\dr{})$, and for $\Kinf{+}(\nu_a,\rOpt,\D)\approx\dr{a}^2$, the regret term of $\dr{a}L_\epsilon(T,\dr{a})$ corresponds with the lower bound of \Cref{corollary:lower bound finite regret}. 
Moreover, for any $\epsilon(t)\!\rdarrow\!0$, the regret is logarithmic, and all suboptimal arms are logarithmically sampled. When $\Kinf{+}(\nu_a,\rOpt,\D)\approx\dr{a}^2$, this also matches the asymptotic lower bound up to absolute constants. Yet, if there are multiple optimal arms, they will be queried $\Ocal\br*{\frac{\ln T}{\epsilon^2(T)}}$ times each -- super-logarithmically. 

Finally, recall that UCB methods for the standard MAB problem typically lead to count (and query) bounds of $\frac{6\ln T}{\dr{a}^2}$ (see, e.g., Theorem 2.1 of \citealt{bubeck2012regret} with $\alpha=3$). Indeed, the suboptimal queries depend on $\bar{N}(T,\dr{a})\le \frac{6\ln T}{\dr{a}^2}$ when there are multiple optimal arms. In contrast, when the optimal arm is unique, the bounds depend on $\bar{N}(T,\dr{a}/2)\approx 4\bar{N}(T,\dr{a})$. This factor might be explained by the second part of the lower bound (\Cref{theorem:lower bound asymptotic}); to logarithmically query the unique optimal arm, it must be completely separated from all suboptimal arms. Then, the term $\dr{a}/2$ is the result of separating both the optimal arm and suboptimal arms from their middle point $(\rwd{a}+\rOpt)/2$. 

\subsection{Numerical Illustration}
\label{subsection:illustrations}
\begin{figure}[t]
\centering
\subfigure{
\includegraphics[trim=0 15 15 15,clip,width=0.3375\linewidth]{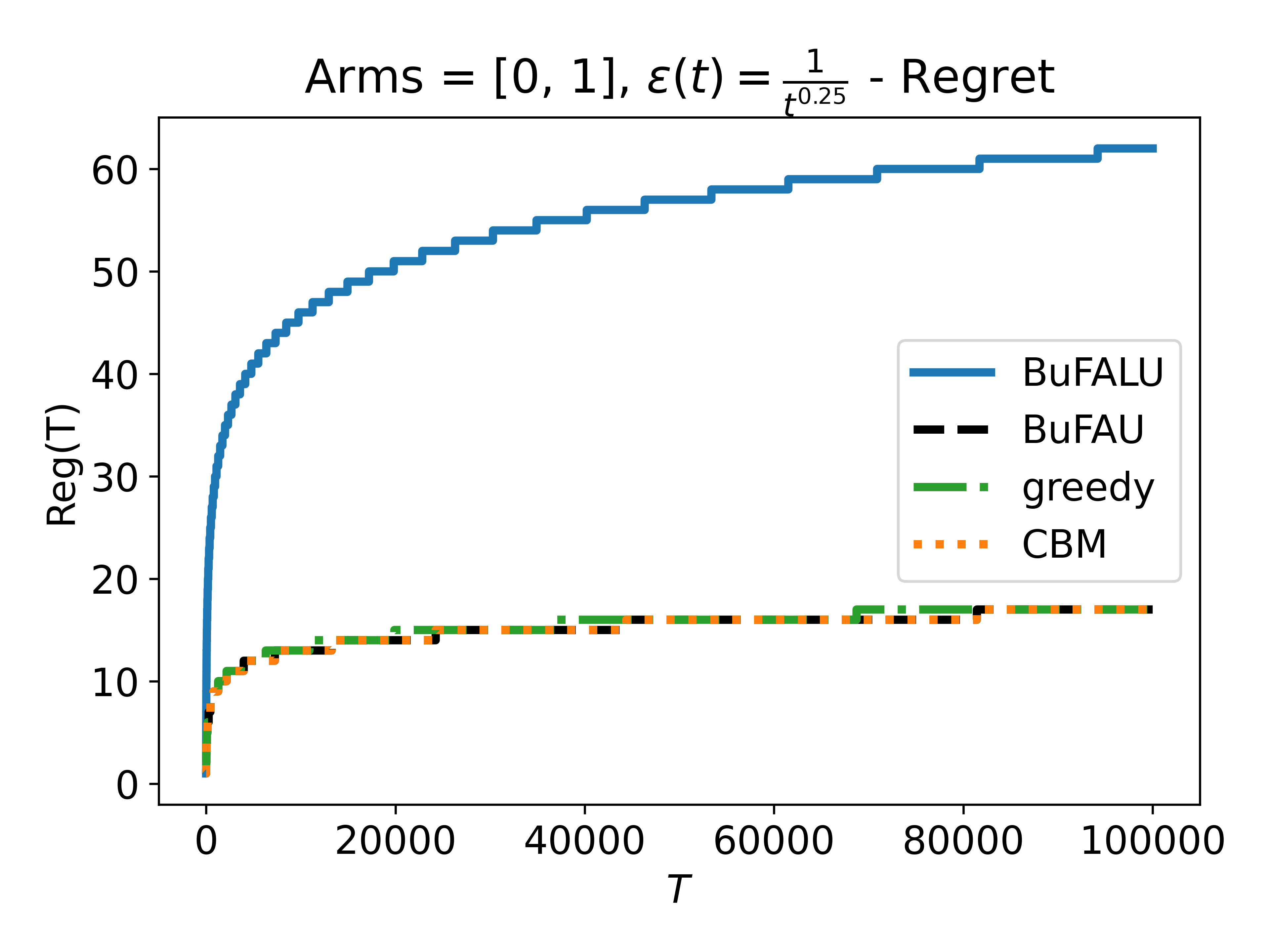}
}
\hspace{0.05\linewidth}
\subfigure{
\includegraphics[trim=0 15 15 15,clip,width=0.3375\linewidth]{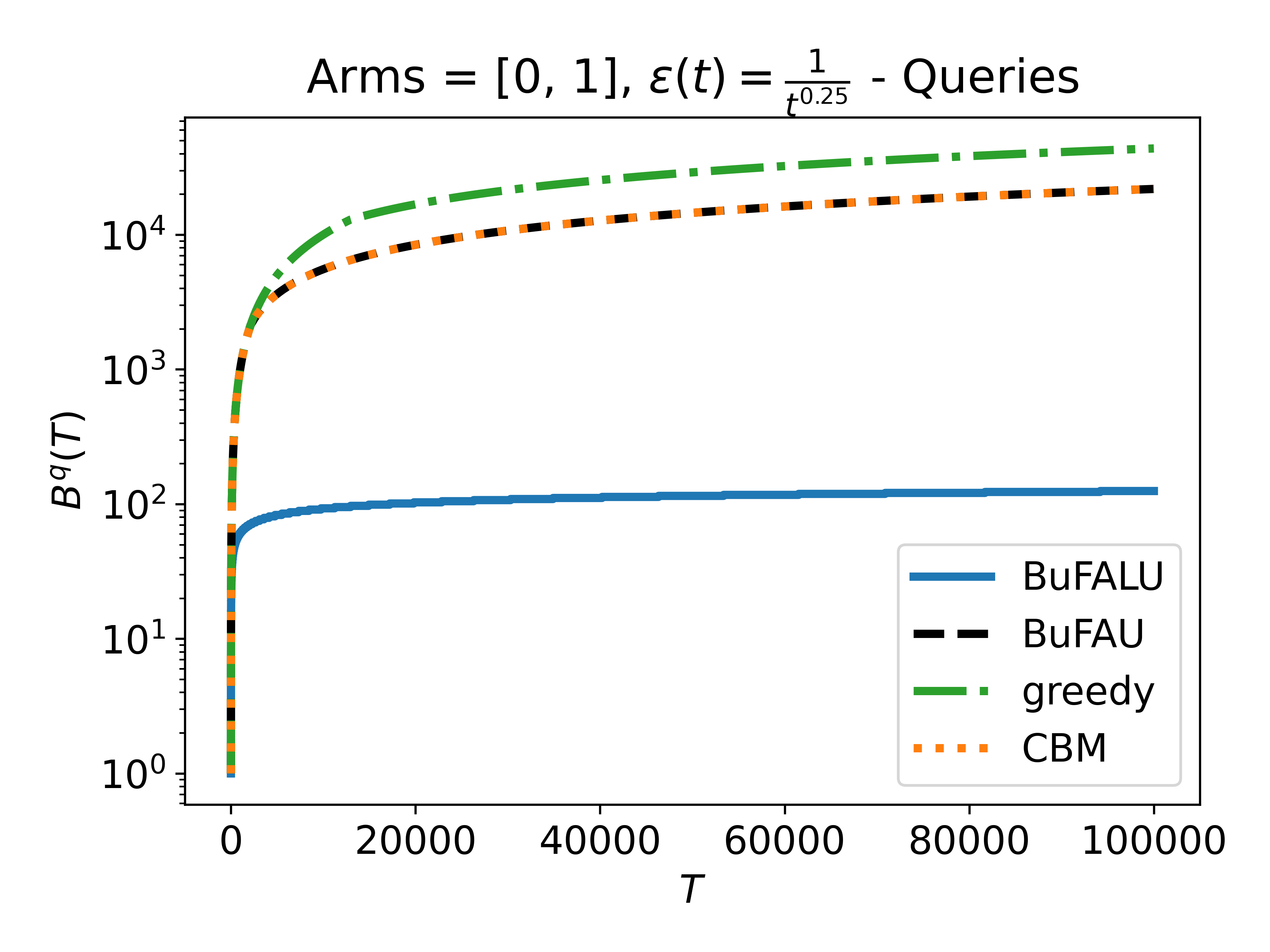}
} \\
\vspace{-0.15cm}
\subfigure{
\includegraphics[trim=0 15 15 15,clip,width=0.3375\linewidth]{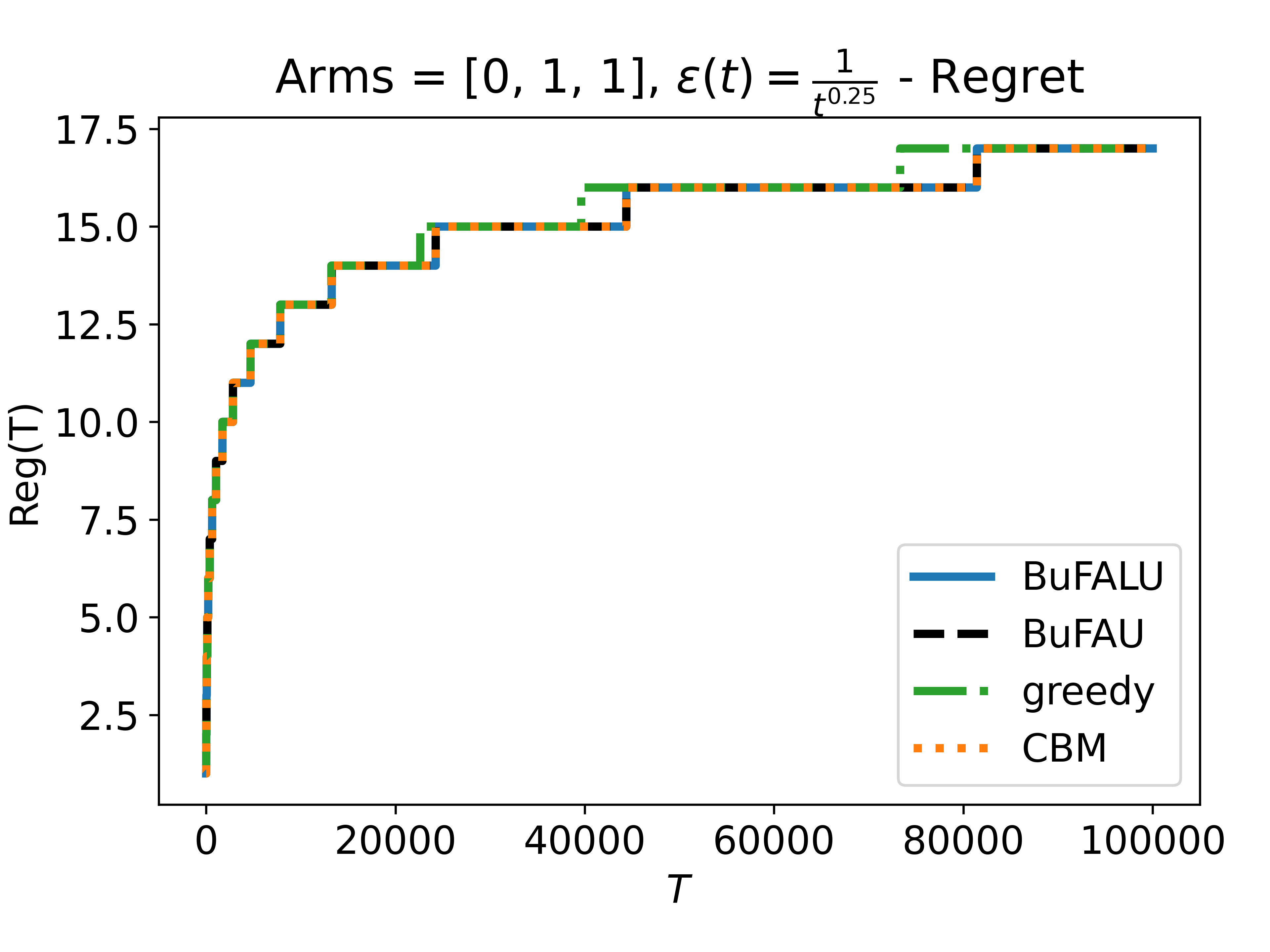}
}
\hspace{0.05\linewidth}
\subfigure{
\includegraphics[trim=0 15 15 15,clip,width=0.3375\linewidth]{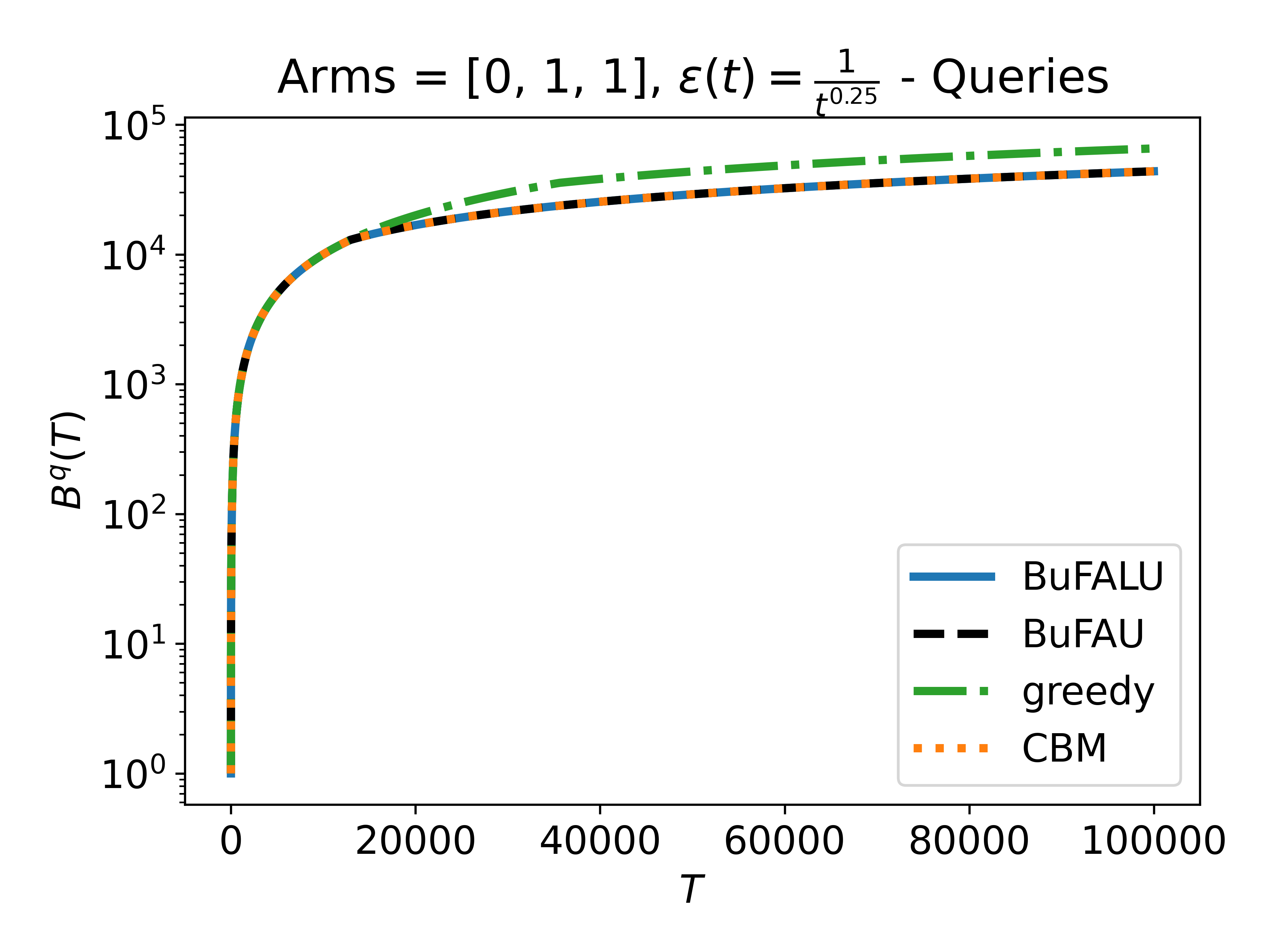}
}\vspace{-0.1cm}
\caption{Performance in two deterministic instances. Left column: regret, right column: number of queries. In the top row, there is a unique optimal arm, while in the bottom one there are two. All evaluations used $\epsilon(t)=t^{-1/4}$ and tested one seed (deterministic problems).}
\label{figure:experiments}
\end{figure}

In this section, we present a simple numerical illustration of the behavior of BuFALU in the presence of a unique or multiple optimal arms. To best capture the difference between the scenarios, we evaluate the algorithm on two \emph{deterministic} MAB instances. In the first instance, the optimal arm is unique (with $\rOpt=1$) and there exists a single suboptimal arm (with $\mu_1=0$). The second instance is the same, except for an additional optimal arm. We compare BuFALU to a few natural baselines. The first uses the same querying mechanism as BuFALU but sets $c_t\in\argmax_aUCB_t(a)$. This baseline, called BuFAU, will allow us to understand the contribution of the action choice vs. the querying rule. 
The second baseline is CBM-UCB \citep{efroni2021confidence}, whose problem-independent guarantees are similar to BuFALU. The final baseline is a greedy algorithm that receives a total budget as guaranteed by \Cref{theorem:upper bound} (i.e., $\frac{6\Narms\ln T}{\epsilon^2(t)}+\Narms$). If the budget is not exhausted, it plays (and queries) the maximal UCB; otherwise, it plays (without querying) the arm with the maximal empirical mean. All algorithms are evaluated with $\epsilon(t)=t^{-1/4}$ for $100,000$ steps. Then, \Cref{theorem:upper bound} guarantees that BuFALU cannot query any arm more than $\sim20,000$ times. We refer to \Cref{appendix:baselines} for a full description of the baselines and remark that evaluations in stochastic 5-armed problems yield similar insights (see \Cref{appendix:experiments}). The results are depicted in \Cref{figure:experiments}.

The simulated behavior of BuFALU validates the characterization of \Cref{theorem:upper bound}: when there are multiple optimal arms, BuFALU uses all available budget to query them and achieves the same regret and querying performance as the baselines. In contrast, when the optimal arm is unique, BuFALU only sparingly asks for feedback -- only $\sim 1/150$ fraction of the cumulative number of queries of other baselines. Its regret, however, is larger by a factor of $4$, the same factor that we get by replacing $\bar{N}(T,\dr{a})$ with $\bar{N}(T,\dr{a}/2)$. 
Notably, if queries have a cost $c$, the query reduction is clearly worth this degradation. In fact, for any per-query cost of $c>0.0025$ (namely, $0.25\%$ of the optimal rewards), BuFALU outperforms all baselines. 
Lastly, comparing BuFALU to BuFAU, we conclude that the key part in BuFALU is the choice of $c_t$, which allows separating a unique optimal arm from all other arms.


\section{Summary and Future Work}
\label{section:summary}
In this work, we analyzed MAB problems where rewards must be queried to be observed. We proved lower bounds in this setting and highlighted the fundamental difference between problems with a unique and multiple optimal arms. We also presented BuFALU, to which we proved problem-dependent regret and querying bounds and showed its adaptivity to the number of optimal arms.

In the standard MAB setting, there are a few interesting directions for improving our results. First, although our analysis supports arbitrary sequences of $\epsilon(t)$, the maximization in $\bar{N}$ might be loose when $\epsilon(t)$ briefly drops. We believe that better characterizing these cases is imperative when working with general (possibly decreasing) budgets. 
Second, when the optimal arm is unique, our regret bounds degrade by a constant factor of $4$. Yet, the lower bounds hint that this factor does not have to affect all arms (and might, for example, be present mainly in the queries of the optimal arm). Improving this factor might require changes to the algorithm and/or confidence intervals and is left for future work. Finally, while we limited the number of queries from played arms, \citealt{yun2018multi} allowed budget-limited observations from unplayed arms. Although combining the settings is natural, the individual solutions are very different, and we leave this for future work.

Moreover, we only tackled the standard MAB setting. Extending our work to other settings might lead to nontrivial challenges. In large or continuous problem, e.g., combinatorial bandits, \citep{chen2016combinatorialA}, linear bandits \citep{abbasi2011improved} and reinforcement learning  \citep[RL,][]{jaksch2010near,azar2013minimax,simchowitz2019non}, the dichotomy to unique and multiple optimal arms might be too coarse. Then, it might be more relevant to characterize the behavior using the size or structure of the set of optimal actions. Specifically in RL, recent studies show that the presence of multiple optimal arms greatly affects the problem-dependent regret even without querying constraints \citep{xu2021fine,tirinzoni2021fully}, and it is worthwhile to further study it when feedback is limited. 
Alternatively, one could extend our setting to nonstationary environments, where feedback is required to track environmental changes, and to settings with some adversity in the rewards (e.g., adversarial corruptions \cite{lykouris2018stochastic}).

Finally, our work raises important questions in the low-budget regime (e.g., logarithmic budget). There, minimizing the regret seems hopeless in the presence of multiple optimal arms, and weaker optimality notions can be considered. One option is lenient regret criteria \citep{merlis2020lenient}, which do not incur regret when playing near-optimal arms. Then, it might be possible to perform well even when $\epsilon(t)$ does not decrease to zero, but this warrants further study.


\bibliographystyle{plainnat}
\bibliography{references}

\clearpage

\appendix

\section{Lower bounds}
\label{appendix:dependent lower bound}

We start by stating a variant of the fundamental inequality of \citep{garivier2019explore}, proved by \cite{efroni2021confidence} for the case where feedback is not always observed.
\begin{lemma}
\label{lemma: kl counts inequality}
For any $T\ge1$, any $\sigma(I_T)$-measurable random variable $Z$ with values in $[0,1]$ and any two bandit instances $\unu$ and $\unu'$, it holds that
\begin{equation}
    \label{eq: kl counts inequality}
    \sum_{a=1}^{\Narms}\E_\unu\brs*{\nq{T}{a}}\kl(\nu_a,\nu'_a) \ge \klBin\br*{\E_\unu\brs*{Z},\E_{\unu'}\brs*{Z}}
\end{equation}
\end{lemma}
As was shown in \citep{garivier2019explore}, a similar inequality allows elegantly deriving lower bounds under various regularity assumptions on the bandit strategy. 
\subsection{Proof of Theorem~\ref{theorem:lower bound asymptotic}}
\lowerBoundAsymptotic*
\begin{proof}
Parts of the proof rely on techniques from Theorem 1 in \citep{garivier2019explore}. First notice that for any $p,q\in\brs*{0,1}$, it holds that 
\begin{align}
    \klBin(p,q) &= \underbrace{p\ln\frac{1}{q}}_{\ge0} + \underbrace{(1-p)\ln\frac{1}{1-q}}_{\ge0} + \underbrace{p\ln p + (1-p)\ln(1-p)}_{\ge -\ln2} \nonumber \\
    & \ge \max\brc*{p\ln\frac{1}{q}, (1-p)\ln\frac{1}{1-q}} - \ln 2 \label{eq:klBin log lower bound}
\end{align}

\paragraph{Proof of Part 1.}  

Let $a\notin\Acal_*(\unu)$ be some suboptimal arm. Furthermore, let $\unu'$ be a modified bandit instance such that $\nu'_k=\nu_k$ for all $k\ne a$ and $\nu'_a\in\D$ is some distribution with $\E\brs{\nu'_a}>\rOpt$ (if such a distribution does not exist, then $\Kinf{+}(\nu_a,\rOpt,\D)=\infty$ and the bound trivially holds).  Then, by \Cref{lemma: kl counts inequality} with $Z= \frac{\np{T}{a}}{T}\in\brs*{0,1}$ and \Cref{eq:klBin log lower bound}, it holds that 
\begin{align}
    \label{eq: lower bound suboptimal 1}
    \E_\unu\brs*{\nq{T}{a}}\kl(\nu_a,\nu'_a)
    &\ge \klBin\br*{\frac{\E_{\unu}\brs*{\np{T}{a}}}{T},\frac{\E_{\unu'}\brs*{\np{T}{a}}}{T}} \nonumber\\
    &\ge \br*{1-\frac{\E_{\unu}\brs*{\np{T}{a}}}{T}}\ln\frac{1}{1-\frac{\E_{\unu'}\brs*{\np{T}{a}}}{T}} - \ln2\enspace.
\end{align}
Since the bandit strategy is consistent and $a$ is suboptimal for instance $\unu$, it holds that $\lim_{T\to\infty}\frac{\E_{\unu}\brs*{\np{T}{a}}}{T} = 0$. Moreover, $a$ is the unique optimal arm of bandit instance $\unu'$. Therefore, the consistency implies that for any $\alpha\in(0,1]$ and for sufficiently large $T$,
\begin{align*}
    1-\frac{\E_{\unu'}\brs*{\np{T}{a}}}{T} 
    = \frac{T - \E_{\unu'}\brs*{\np{T}{a}}}{T}
    = \frac{\sum_{a'\ne a}\E_{\unu'}\brs*{\np{T}{a'}}}{T} \le \frac{T^{\alpha}}{T} 
    = T^{\alpha-1}\enspace.
\end{align*}
Combining both into \cref{eq: lower bound suboptimal 1}, we get that for any $\alpha\in(0,1]$ and any $\nu'_a$ with $\E\brs{\nu'_a}>\rOpt$,
\begin{align*}
    \liminf_{T\to\infty}\frac{\E_\unu\brs*{\nq{T}{a}}}{\ln T}
    &\ge \frac{1}{\kl(\nu_a,\nu'_a)}\liminf_{T\to\infty}\frac{1}{\ln T}\br*{\br*{1-\frac{\E_{\unu}\brs*{\np{T}{a}}}{T}}\ln\frac{1}{1-\frac{\E_{\unu'}\brs*{\np{T}{a}}}{T}} - \ln2}\\
    & \ge \frac{1}{\kl(\nu_a,\nu'_a)}\liminf_{T\to\infty}\frac{1}{\ln T}\ln\frac{1}{T^{\alpha-1}}\\
    & = \frac{1-\alpha}{\kl(\nu_a,\nu'_a)}\enspace,
\end{align*}
and since it holds for any $\alpha\in(0,1]$, we have that 
\begin{align*}
    \liminf_{T\to\infty}\frac{\E_\unu\brs*{\nq{T}{a}}}{\ln T}\ge \frac{1}{\kl(\nu_a,\nu'_a)}\enspace.
\end{align*}
By taking the supremum in the right-hand side over all distributions $\nu'_a\in\D$ with $\E\brs{\nu'_a}>\rOpt$, we conclude this part of the proof.


\paragraph{Proof of Part 2.} 

Let $\mu\in\brs*{\mu^s,\rOpt}$ and assume that there exist two distributions $\nu'_a,\nu'_{a^*}\in\D$ such that $\E\brs{\nu'_a}\ge\mu$ and $\E\brs{\nu'_{a^*}}<\mu$ (otherwise, either $\Kinf{+}(\nu_a,\mu,\D)=\infty$ or $\Kinf{-}(\nu_{a^*},\mu,\D)=\infty$ and the bound trivially holds). Also, define a new bandit instance $\unu'$ for which $\nu'_k=\nu_k$ for all $k\ne a,a^*$ and arms $a,a^*$ are distributed according to $\nu'_a$ and $\nu'_{a^*}$, respectively. By \Cref{lemma: kl counts inequality} with $Z= \frac{\np{T}{a^*}}{T}\in\brs*{0,1}$ and \Cref{eq:klBin log lower bound}, it holds that 
\begin{align}
    \label{eq: lower bound optimal 1}
    \E_\unu\brs*{\nq{T}{a}}\kl(\nu_a,\nu'_a) + \E_\unu\brs*{\nq{T}{a^*}}\kl(\nu_{a^*},\nu'_{a^*})
    &\ge \klBin\br*{\frac{\E_{\unu}\brs*{\np{T}{a^*}}}{T},\frac{\E_{\unu'}\brs*{\np{T}{a^*}}}{T}} \nonumber\\
    &\ge \frac{\E_{\unu}\brs*{\np{T}{a^*}}}{T}\ln\frac{1}{\frac{\E_{\unu'}\brs*{\np{T}{a^*}}}{T}} - \ln2\enspace.
\end{align}
Since the bandit strategy is consistent and $a^*$ is the unique optimal arm for instance $\unu$, for any $a'\ne a^*$ we have that $\E_{\unu}\brs*{\np{T}{a'}}/T\to 0$ and thus
\begin{align*}
    \lim_{T\to\infty}\frac{\E_{\unu}\brs*{\np{T}{a^*}}}{T}
    = \lim_{T\to\infty}\frac{T - \sum_{a'\ne a^*}\E_{\unu}\brs*{\np{T}{a'}}}{T}
    = 1 - \sum_{a'\ne a^*}\lim_{T\to\infty}\frac{\E_{\unu}\brs*{\np{T}{a'}}}{T}
    =1\enspace.
\end{align*}
Moreover, $a^*$ is strictly suboptimal in bandit instance $\unu'$. Therefore, the consistency implies that for any $\alpha\in(0,1]$ and for sufficiently large $T$,
\begin{align*}
    \frac{\E_{\unu'}\brs*{\np{T}{a^*}}}{T} 
    \le \frac{T^{\alpha}}{T} 
    = T^{\alpha-1}\enspace.
\end{align*}
Therefore, for any $\alpha\in(0,1]$, we can bound
\begin{align}
    \liminf_{T\to\infty}\frac{1}{\ln T}\br*{ \frac{\E_{\unu}\brs*{\np{T}{a^*}}}{T}\ln\frac{1}{\frac{\E_{\unu'}\brs*{\np{T}{a^*}}}{T}} - \ln2}
     \ge \liminf_{T\to\infty}\frac{1}{\ln T}\ln\frac{1}{T^{\alpha-1}}
     = 1-\alpha\enspace, \label{eq: lower bound optimal 2}
\end{align}
and since it holds for any $\alpha\in(0,1]$, the same result holds for $\alpha=0$.

For the l.h.s., we divide into two cases:

\underline{Case I:} Letting $a\in \argmax_{a'\ne a^*}\rwd{a'}$ be an arm such that $\rwd{a}=\mu^s$ while choosing $\mu=\mu^s$ and $\nu'_a=\nu_a$, we get that $\kl(\nu_a,\nu'_a)=0$. Furthermore, since $a$ is strictly suboptimal and $\E\brs{\nu'_{a^*}}<\mu^s<\rOpt$, we have that $\kl(\nu_{a^*},\nu'_{a^*})>0$. Dividing by it in \eqref{eq: lower bound optimal 1} and combining with \eqref{eq: lower bound optimal 2}, we get that 
\begin{align*}
    \liminf_{T\to\infty}\frac{\E_\unu\brs*{\nq{T}{a^*}}}{\ln T}\ge \frac{1}{\kl(\nu_{a^*},\nu'_{a^*})}\enspace.
\end{align*}
By taking the supremum in the right-hand side over all distributions $\nu'_{a^*}\in\D$ with $\E\brs{\nu'_{a^*}}<\mu^s$, we get the first desired result.
    
\underline{Case II:} For any $\mu\in\brs*{\mu^s,\rOpt}$, we apply H\"older's inequality and bound the l.h.s. of \Cref{eq: lower bound optimal 1} by
\begin{align*}
    \E_\unu\brs*{\nq{T}{a}}\kl(\nu_a,\nu'_a)& + \E_\unu\brs*{\nq{T}{a^*}}\kl(\nu_{a^*},\nu'_{a^*})\\
    &\le \br*{\E_\unu\brs*{\nq{T}{a}} + \E_\unu\brs*{\nq{T}{a^*}}}\max\brc*{\kl(\nu_a,\nu'_a),\kl(\nu_{a^*},\nu'_{a^*})}\enspace.
\end{align*}
Importantly, notice that $\kl(\nu_{a^*},\nu'_{a^*})>0$ (since $\E\brs*{\nu'_{a^*}}<\mu\le\rOpt$)  and we can divide by the maximum. Combining with \eqref{eq: lower bound optimal 2}, we get 
\begin{align*}
\liminf_{T\to\infty}\frac{\E_\unu\brs*{\nq{T}{a}}+\E_\unu\brs*{\nq{T}{a^*}}}{\ln T}
&\ge \frac{1}{\max\brc*{\kl(\nu_a,\nu'_a),\kl(\nu_{a^*},\nu'_{a^*})}}\enspace.
\end{align*}
Taking the supremum over all distributions $\nu'_a,\nu'_{a^*}\in\D$ with expectations $\E\brs*{\nu'_a}>\mu$ and $\E\brs*{\nu'_{a^*}}<\mu$ leads to the second stated result and concludes this part of the proof.

We remark that a more general lower bound can be written without applying H\"older's inequality:
\begin{align} 
    \label{eq: lower bound unique optimal improved}
    \Kinf{+}(\nu_a,\mu,\D)\nq{\infty}{a} + \Kinf{-}(\nu_{a^*},\mu,\D)\nq{\infty}{a^*} \ge 1 \enspace
\end{align}
for any $\mu\in\brs*{\mu^s,\rOpt}$, where we define $0\cdot (+\infty)\ge1$. Notably, this definition makes sure that the bound holds whenever any of the quantities at its l.h.s. are infinite, so it is only needed to be proven when all quantities are finite. Then, starting from \eqref{eq: lower bound optimal 1}, dividing by $\ln T$, taking $\liminf$ and using \eqref{eq: lower bound optimal 2}, we get that
\begin{align*}
    \kl(\nu_a,\nu'_a)\nq{\infty}{a}  +\kl(\nu_{a^*},\nu'_{a^*})\nq{\infty}{a^*}\ge1 \enspace.
\end{align*}
Next, if both $\Kinf{+}(\nu_a,\mu,\D),\Kinf{-}(\nu_{a^*},\mu,\D)<+\infty$, we can take the infimum at the l.h.s. knowing that it is not over an empty set. Moreover, for $\nq{\infty}{a},\nq{\infty}{a^*}<+\infty$, the inequality is preserved after the infimum, which leads to \Cref{eq: lower bound unique optimal improved}.

\paragraph{Proof of Part 3. } 

\underline{Assume that condition (1) holds}. For any $a\in\Acal_*(\unu)$, define a bandit instance $\unu'$ such that $\nu'_k=\nu_k$ for all $k\ne a$ and $\nu'_a\in\D$ is some distribution such that $\E\brs*{\nu'_a}<\rOpt$ (such a distribution must exist, as otherwise, $\Kinf{-}(\nu_a,\rOpt,\D)=\infty$ and condition (1) does not hold). Then, by \Cref{lemma: kl counts inequality} with $Z= \frac{\np{T}{a}}{T}\in\brs*{0,1}$ and \Cref{eq:klBin log lower bound}, it holds that 
\begin{align*}
    \E_\unu\brs*{\nq{T}{a}}\kl(\nu_a,\nu'_a)
    \ge \klBin\br*{\frac{\E_{\unu}\brs*{\np{T}{a}}}{T},\frac{\E_{\unu'}\brs*{\np{T}{a}}}{T}}
    \ge \frac{\E_{\unu}\brs*{\np{T}{a}}}{T}\ln\frac{1}{\frac{\E_{\unu'}\brs*{\np{T}{a}}}{T}} - \ln2\enspace.
\end{align*}
As $a$ is strictly suboptimal in bandit instance $\unu'$, and since the bandit strategy is consistent, for any $\alpha\in(0,1]$ and for sufficiently large $T$, it holds that
\begin{align*}
    \frac{\E_{\unu'}\brs*{\np{T}{a}}}{T} 
    \le \frac{T^{\alpha}}{T} 
    = T^{\alpha-1}\enspace.
\end{align*}
Then, for large enough $T$, we have that 
\begin{align*}
    \E_\unu\brs*{\nq{T}{a}}\kl(\nu_a,\nu'_a)
    \ge \frac{\E_{\unu}\brs*{\np{T}{a}}}{T}\ln\frac{1}{T^{\alpha-1}} - \ln2
    = (1-\alpha)\frac{\E_{\unu}\brs*{\np{T}{a}}}{T}\ln T - \ln2\enspace.
\end{align*}
Since the number of arms is finite, for large enough $T$, the inequality holds for all $a\in\Acal_*(\unu)$. Then, summing over all inequalities yields
\begin{align}
    \label{eq: multiple optimal lower bound lower shifts 1}
    \sum_{a\in\Acal_*(\unu)}\E_\unu\brs*{\nq{T}{a}}\kl(\nu_a,\nu'_a)
    \ge (1-\alpha)\frac{\sum_{a\in\Acal_*(\unu)}\E_{\unu}\brs*{\np{T}{a}}}{T}\ln T - \Narms\ln2\enspace.
\end{align}
To further analyze the r.h.s. of the inequality, recall that the bandit strategy is consistent; therefore, the set of the optimal arm is sampled linearly, i.e.,
\begin{align*}
   \lim_{T\to\infty}\frac{\sum_{a\in\Acal_*(\unu)}\E_{\unu}\brs*{\np{T}{a}}}{T}
    &= \lim_{T\to\infty}\frac{T - \sum_{a'\notin \Acal_*(\unu)}\E_{\unu}\brs*{\np{T}{a'}}}{T}
    = 1 - \sum_{a'\notin \Acal_*(\unu)}\lim_{T\to\infty}\frac{\E_{\unu}\brs*{\np{T}{a'}}}{T}
    =1\,,
\end{align*}
where the last equality is by the consistency, which implies that $\E_{\unu}\brs*{\np{T}{a'}}/T\to0$ for any $a\notin\Acal_*(\unu)$. For the l.h.s., we apply H\"older's inequality and get 
\begin{align*}
    \sum_{a\in\Acal_*(\unu)}\E_\unu\brs*{\nq{T}{a}}\kl(\nu_a,\nu'_a) \le \max_{a\in\Acal_*(\unu)}\kl(\nu_a,\nu'_a)\sum_{a\in\Acal_*(\unu)}\E_\unu\brs*{\nq{T}{a}}\enspace.
\end{align*}
Substituting both into \Cref{eq: multiple optimal lower bound lower shifts 1}, reorganizing and taking the limit, we get
\begin{align*}
    \liminf_{T\to\infty}\frac{\sum_{a\in\Acal_*}\E_{\unu}\brs*{\nq{T}{a}}}{\ln T} \ge (1-\alpha)\frac{1}{\max_{a\in\Acal_*(\unu)}\kl(\nu_a,\nu'_a)}\enspace.
\end{align*}
Taking the supremum over all distributions $\brc*{\nu'_a}_{a\in\Acal_*(\unu)}$ in $\D$ such that $\E\brs*{\nu'_a}<\rOpt$  and noting that the bound holds for any $\alpha\in(0,1]$ leads to the following bound:
\begin{align*}
    \liminf_{T\to\infty}\frac{\sum_{a\in\Acal_*}\E_{\unu}\brs*{\nq{T}{a}}}{\ln T} \ge \frac{1}{\max_{a\in\Acal_*}\Kinf{-}(\nu_a,\rOpt,\D)}\enspace,
\end{align*}
Importantly, this bound implies that if $\Kinf{-}(\nu_a,\rOpt,\D)=0$ for all $a\in\Acal_*(\unu)$, then the r.h.s. is infinite and there exists at least on optimal arm $a\in\Acal_*(\unu)$ for which $\nq{\infty}{a}=\liminf_{T\to\infty}\frac{\E_{\unu}\brs*{\nq{T}{a}}}{\ln T}=\infty$. Moreover, one can easily verify that the derivation did not actually require condition (1); therefore, this might serve as a lower bound when the condition does not hold. Finally, as in \eqref{eq: lower bound unique optimal improved}, if we define $0\cdot(+\infty)\ge1$, then using the exact same arguments, a more general version of the bound would be 
\begin{align} 
    \label{eq: lower bound multiple optimal below improved}
    \sum_{a\in\Acal_*(\unu)}\Kinf{-}(\nu_a,\rOpt,\D)\nq{\infty}{a} \ge 1 \enspace.
\end{align}


\underline{Assume that condition (2) holds}. For any $a\in\Acal_*(\unu)$, define a bandit instance $\unu'$ such that $\nu'_k=\nu_k$ for all $k\ne a$ and $\nu'_a\in\D$ is some distribution such that $\E\brs*{\nu'_a}>\rOpt$ (we will later take an infiumum over all such distributions, and if such a distribution does not exists, then the value $\Kinf{+}(\nu_a,\rOpt,\D)=\infty$ will lead to a trivial bound). Then, by \Cref{lemma: kl counts inequality} with $Z= \frac{\np{T}{a}}{T}\in\brs*{0,1}$ and \Cref{eq:klBin log lower bound}, it holds that 
\begin{align*}
    \E_\unu\brs*{\nq{T}{a}}\kl(\nu_a,\nu'_a)
    &\ge \klBin\br*{\frac{\E_{\unu}\brs*{\np{T}{a}}}{T},\frac{\E_{\unu'}\brs*{\np{T}{a}}}{T}}\\
    &\ge \br*{1-\frac{\E_{\unu}\brs*{\np{T}{a}}}{T}}\ln\frac{1}{1-\frac{\E_{\unu'}\brs*{\np{T}{a}}}{T}} - \ln2\enspace.
\end{align*}
As $a$ is the unique optimal arm for bandit instance $\unu'$, and since the bandit strategy is consistent, for any $\alpha\in(0,1]$ and for sufficiently large $T$ it holds that
\begin{align*}
    1-\frac{\E_{\unu'}\brs*{\np{T}{a}}}{T} 
    = \frac{T - \E_{\unu'}\brs*{\np{T}{a}}}{T}
    = \frac{\sum_{a'\ne a}\E_{\unu'}\brs*{\np{T}{a'}}}{T} \le \frac{T^{\alpha}}{T} 
    = T^{\alpha-1}\enspace.
\end{align*}
Then, for large enough $T$, we have that 
\begin{align*}
    \E_\unu\brs*{\nq{T}{a}}\kl(\nu_a,\nu'_a)
    &\ge \br*{1-\frac{\E_{\unu}\brs*{\np{T}{a}}}{T}}\ln\frac{1}{T^{\alpha-1}} - \ln2\\
    &= (1-\alpha)\br*{1-\frac{\E_{\unu}\brs*{\np{T}{a}}}{T}}\ln T - \ln2\enspace.
\end{align*}
Next, let $a,b\in\Acal_*(\unu)$ be two optimal arms. For large enough $T$, the inequality holds for both arms, and summing over their respective inequalities yields
\begin{align}
    \label{eq: multiple optimal lower bound upper shifts 1}
    \E_\unu\brs*{\nq{T}{a}}\kl(\nu_a,\nu'_a) &+ \E_\unu\brs*{\nq{T}{b}}\kl(\nu_b,\nu'_b)\nonumber\\
    &\ge (1-\alpha)\br*{2-\frac{\E_{\unu}\brs*{\np{T}{a}}}{T}-\frac{\E_{\unu}\brs*{\np{T}{b}}}{T}}\ln T - 2\ln2\enspace.
\end{align}
To further analyze the r.h.s. of the inequality, notice that 
\begin{align*}
    2-\frac{\E_{\unu}\brs*{\np{T}{a}}}{T}-\frac{\E_{\unu}\brs*{\np{T}{b}}}{T}
    \ge 2 - \frac{\sum_{a'} \E_{\unu}\brs*{\np{T}{a'}}}{T}
     = 2 - \frac{T}{T} 
     =1
\end{align*}
For the l.h.s., we bound using H\"older's inequality:
\begin{align*}
    \E_\unu\brs*{\nq{T}{a}}\kl(\nu_a,\nu'_a) &+ \E_\unu\brs*{\nq{T}{b}}\kl(\nu_b,\nu'_b)\\
    &\le \max\brc*{\kl(\nu_a,\nu'_a),\kl(\nu_b,\nu'_b)}\br*{\E_\unu\brs*{\nq{T}{a}} + \E_\unu\brs*{\nq{T}{b}}}\enspace.
\end{align*}
Substituting both into \Cref{eq: multiple optimal lower bound upper shifts 1}, reorganizing and taking the limit, we get
\begin{align*}
    \liminf_{T\to\infty}\frac{\E_\unu\brs*{\nq{T}{a}} + \E_\unu\brs*{\nq{T}{b}}}{\ln T} \ge (1-\alpha)\frac{1}{\max\brc*{\kl(\nu_a,\nu'_a),\kl(\nu_b,\nu'_b)}}\enspace.
\end{align*}
Taking the supremum over all distributions $\nu'_a,\nu'_b\in\D$ such that $\E\brs*{\nu'_a}>\rOpt, \E\brs*{\nu'_b}>\rOpt$  and noting that the bound holds for any $\alpha\in(0,1]$ leads to
\begin{align*}
    \liminf_{T\to\infty}\frac{\E_{\unu}\brs*{\nq{T}{a}}+\E_{\unu}\brs*{\nq{T}{b}}}{\ln T} \ge \frac{1}{\max\brc*{\Kinf{+}(\nu_a,\rOpt,\D),\Kinf{+}(\nu_{b},\rOpt,\D)}}\enspace.
\end{align*}
As we previously remarked, if one of the infimums is over an empty set, the r.h.s. equals zero and the bound trivially holds. 
Finally, as in the case of condition (1), if there exist two optimal arms $a,b\in\Acal_*(\unu)$ for which $\Kinf{+}(\nu_{a},\rOpt,\D)=\Kinf{+}(\nu_{b},\rOpt,\D)=0$, then the r.h.s. of this bound equals infinity, and for at least one of these arms $a'\in\brc*{a,b}$, it holds that $\nq{\infty}{a'}=\liminf_{T\to\infty}\frac{\E_{\unu}\brs*{\nq{T}{a'}}}{\ln T}=\infty$. Furthermore, this bound can serve as a lower bound even when condition (2) does not hold. Finally, as in \eqref{eq: lower bound unique optimal improved} and \eqref{eq: lower bound multiple optimal below improved}, if we define $0\cdot(+\infty)\ge1$, a more general version of the bound would be 
\begin{align} 
    \label{eq: lower bound multiple optimal above improved}
    \Kinf{+}(\nu_a,\rOpt,\D)\nq{\infty}{a} + \Kinf{+}(\nu_b,\rOpt,\D)\nq{\infty}{b}  \ge 1 \enspace.
\end{align}
\end{proof}

\clearpage


\subsection{Additional Intuition Behind the Lower Bound for multiple optimal arms}
\label{appendix:lower bound example}

The following example illustrates the intuition behind the conditions for the third part of \Cref{theorem:lower bound asymptotic}:
\begin{example}
\label{example:lower bound asymptotic}
Let $x\in(0,1)$. Define $\D_1$ as the set of all distributions with the discrete support $\brc*{0,1}$ and expectations in $[0,x]$. Also, define $\D_2$ as the set of all distributions with the discrete support $\brc*{x,1}$. Particularly, for any $\nu\in\D_1$, $\E\brs*{\nu}\le x$, and for any $\nu\in\D_2$, $\E\brs*{\nu}\ge x$. Finally, let $\D=\D_1\cup\D_2$. Now let $\unu$ be a bandit instance with arm distributions in $\D$ such that $\rOpt=x$, i.e., the optimal reward is $x$. Then, the distribution of optimal arms is either $Ber(x)$ (if $\nu_{a^*}\in\D_1$) or outputs $x$ w.p. 1  (if $\nu_{a^*}\in\D_2$), and suboptimal arms are in $\D_1$. Consider the following instances:
\begin{enumerate}
    \item \textbf{All optimal arms are in $\D_1$.} 
    Recall that $\Omega\br*{\frac{\ln{1}/{\delta}}{\epsilon^2}}$ samples are required to distinguish between $Ber(x)$ and $Ber(x-\epsilon)$ w.p. $1-\delta$. Then, any fixed logarithmic budget cannot identify whether an arm is optimal or suboptimal with arbitrarily close mean $x-\epsilon$, and at least one optimal arm must be queried 
    super-logarithmically
    identify with high certainty that $\rOpt=x$. 
    Notice that in this case, $\Kinf{-}(\nu_a,x,\D)=0$ for all optimal arms. 
    \item \textbf{One optimal arm belongs to $\D_2$ and all other optimal arms are in $\D_1$.} 
    In this case, an agent can identify the optimal arm in $\D_2$ by a single sample, since it outputs $x$ w.p. 1 and this value is not supported by any distribution in $\D_1$. Similarly, an optimal arm can be related to $\D_1$ when outputs $0$, which only requires $\Ocal\br*{\frac{\ln{1}/{\delta}}{\ln{1}/\br{1-x}}}$ samples with certainty $1-\delta$. All suboptimal arms can be similarly identified. Then, with a logarithmic number of queries, agents can identify that only one arm belongs to $\D_2$, and as $\E\brs*{\nu}\le\E\brs*{\nu'}$ for any $\nu\in\D_1$ and $\nu'\in\D_2$, exploiting this arm is always optimal. Thus, logarithmic queries suffice.
    \item \textbf{At least two optimal arms belong to $\D_2$.} 
    In this case, agents must identify that the optimal mean $\rOpt$ is not higher than $x$, and as in the case where all optimal arms belong to $\D_1$, 
    a logarithmic number of queries will not suffice, as it only allows identifying the mean of an arm with a fixed accuracy. Then, all optimal arms in $\D_2$, except for a single arm, must be sampled more than a logarithmic number of times to identify that their mean is exactly $x$.
    In contrast, the remaining arm in $\D_2$ can be queried once to identify that it is in $\D_2$. Afterwards, as $\E\brs*{\nu}\ge x$ for any $\nu\in\D_2$ and $\rOpt=x$, this arm can be safely exploited. This corresponds to the case where $\Kinf{+}(\nu_a,\rOpt,\D)=0$ for at least two optimal arms.
\end{enumerate}
\end{example}

While the example illustrates that the conditions in the Theorem are sufficient, we actually believe that they are also necessary. In particular, when both conditions do not hold, we believe that an action-elimination algorithm that uses KL-based confidence interval \citep{garivier2018kl} should allow logarithmically querying the optimal arms. To do so, the algorithm will have to prioritize the optimal arm that might have similar distributions with higher expectations,  if exists (e.g., the optimal arm in $\D_2$ in the second part of the example). However, this is not a formal proof, which we leave for future work.

\clearpage


\subsection{Implications of the Problem-Dependent Lower Bounds to Querying Costs}
\label{appendix:querying costs}

We now discuss the implications of the asymptotic lower bounds when querying rewards incur a cost $c>0$. The cost might be known (e.g., payment for labeling) or unknown (e.g., user irritation in recommender systems). In this setting, it is natural to modify the regret to be \emph{query-aware}:
\begin{align*}
    \Regret^q(T)
    = \underbrace{\E\brs*{\sum_{t=1}^T \br*{\rOpt - \rwd{a_t}}}}_{\text{reward suboptimality}} + \underbrace{\E\brs*{\sum_{t=1}^T c\Ind{q_t=1}}}_{\text{cost penalty}} 
    &= \sum_{a=1}^\Narms \br*{\dr{a}\E\brs*{\np{T}{a}}+c\E\brs*{\nq{T}{a}}} \\[-.4cm]
    & \ge \sum_{a=1}^\Narms \br*{\dr{a}+c}\E\brs*{\nq{T}{a}}\enspace,
\end{align*}
where the inequality is since actions must be played to be queried. 
Assume that the regret is sub-polynomial for any instance in $\D$; in particular, the strategy is consistent and \Cref{theorem:lower bound asymptotic} holds.

Now, if the optimal arm is unique, then strategies must separate the optimal arm from any suboptimal arm (by \Cref{theorem:lower bound asymptotic}, part 2). However, doing so with super-logarithmic queries leads to suboptimal regret due to the querying costs. Thus, optimal strategies should query all arms logarithmically, \emph{including the optimal arm}. The best balance between queries from the optimal arm and increased plays from suboptimal arms depends on the values of $c$ and the gaps. However, \textbf{reward degradation is unavoidable} for any $c>0$.  

On the other hand, if there are multiple optimal arms and the conditions of \Cref{theorem:lower bound asymptotic} (part 3) hold, then optimal arms must be queried super-logarithmically for a strategy to be consistent. Thus, \textbf{no consistent strategy can achieve logarithmic query-aware regret}, and every strategy must either suffer from high querying costs or low rewards.

\clearpage


\subsection{Proofs of Proposition~\ref{prop:lower bound linear} and Corollary~\ref{corollary:lower bound finite regret}}
\label{appendix:lower bound linear}

\lowerBoundLinear*
\begin{proof}
We follow the proof of Theorem 2 of \citep{garivier2019explore}. Let $a\notin\Acal_*(\unu)$ and let $\unu'$ be some modified problem such that $\nu'_k=\nu_k$ for all $k\ne a$ and $\nu'_a\in\D$ is such that $\E\brs{\nu'_a}>\rOpt$ (if no such distribution exists, then $\Kinf{+}(\nu_a,\rOpt,\D)=\infty$ and the r.h.s. of the lower bound is defined as $0$, so the bound trivially holds). Furthermore, notice that the desired lower bound is always smaller than $\frac{T}{\Narms}$; therefore, if $\E_{\unu}\brs*{\np{T}{a}}\ge \frac{T}{\Narms}$ then the bound holds. Thus, we can assume for the rest of the proof that $\E_{\unu}\brs*{\np{T}{a}}< \frac{T}{\Narms}$. Finally, since the strategy is better than uniform on $\D$ and $a$ is the unique optimal arm in $\unu'$, it holds that $\E_{\unu'}\brs*{\np{T}{a}}\ge \frac{T}{\Narms}$. Then, by \Cref{lemma: kl counts inequality} with $Z= \frac{\np{T}{a}}{T}\in\brs*{0,1}$, we have that
\begin{align}
    \label{eq:lower bound linear 1}
    \E_\unu\brs*{\nq{T}{a}}\kl(\nu_a,\nu'_a)
    \ge \klBin\br*{\frac{\E_{\unu}\brs*{\np{T}{a}}}{T},\frac{\E_{\unu'}\brs*{\np{T}{a}}}{T}}
    \ge \klBin\br*{\frac{\E_{\unu}\brs*{\np{T}{a}}}{T},\frac{1}{\Narms}}\enspace,
\end{align}
where the last inequality is since $\frac{\E_{\unu}\brs*{\np{T}{a}}}{T}< \frac{1}{\Narms}$, $\frac{\E_{\unu'}\brs*{\np{T}{a}}}{T}\ge \frac{1}{\Narms}$ and for $p\le q$, the function $\klBin(p,q')$ is increasing in $q'\in\brs*{p,q}$. Next, recall that by the local refinement of Pinsker's inequality (e.g., Lemma 6 in \citep{garivier2019explore}), for any $0\le p < q \le1$, we have that 
\begin{align*}
    \klBin(p,q) \ge \frac{1}{2q}(q-p)^2
\end{align*}
Substituting into \Cref{eq:lower bound linear 1}, we get 
\begin{align*}
    \E_\unu\brs*{\nq{T}{a}}\kl(\nu_a,\nu'_a)
    \ge \frac{\Narms}{2}\br*{\frac{1}{\Narms} - \frac{\E_{\unu}\brs*{\np{T}{a}}}{T}}^2\enspace,
\end{align*}
which can be reorganized (using $\frac{1}{\Narms} \ge \frac{\E_{\unu}\brs*{\np{T}{a}}}{T}$) to
\begin{align*}
    \E_{\unu}\brs*{\np{T}{a}} \ge \frac{T}{\Narms}\br*{1 - \sqrt{2\Narms\E_\unu\brs*{\nq{T}{a}}\kl(\nu_a,\nu'_a)}}\enspace.
\end{align*}
Moreover, since $\E_{\unu}\brs*{\np{T}{a}}\le \frac{T}{K}$, we also have that $\E_{\unu}\brs*{\nq{T}{a}}\le \frac{T}{K}$. This leads to a bound of 
\begin{align*}
    \E_{\unu}\brs*{\np{T}{a}} \ge \frac{T}{\Narms}\br*{1 - \sqrt{2T\cdot\kl(\nu_a,\nu'_a)}}\enspace.
\end{align*}
Taking the maximum between both previous bounds results with
\begin{align*}
    \E_{\unu}\brs*{\np{T}{a}} \ge \frac{T}{\Narms}\br*{1 - \sqrt{2\min\brc*{\Narms\E_{\unu}\brs*{\nq{T}{a}},T}\kl(\nu_a,\nu'_a)}}\enspace,
\end{align*}
and taking the supremum over all distributions $\nu'_a\in\D$ such that $\E\brs*{\nu'_a}>\rOpt$ leads to the desired result. 
In particular, notice that the result when $\Kinf{+}(\nu_a,\rOpt,\D)<\infty$ and $\E_{\unu}\brs*{\nq{T}{a}}\le \frac{1}{8\Narms \Kinf{+}(\nu_a,\rOpt,\D)}$ directly follows from the general bound.
\end{proof}

\clearpage
\lowerBoundLinearRegret*
\begin{proof}
For generality, we only assume that $\E_{\unu}\brs*{\nq{t}{a}} \le B_a(t)$, for some positive nondecreasing function $\brc*{B_a(t)}_{t\ge1}$. 
First, notice that since we defined $B_a^{-1}(x)=0$ when $B(1)> x$, then $B_a^{-1}(x)$ always exists, but might be equal to $+\infty$ if $B(t)\le x$ for all $t\in\mathbb{N}$. On the other hand, if $B_a^{-1}(x)=+\infty$, then the bound trivially holds (as there is no $T$ for which the result must hold). Thus, w.l.o.g., we assume that $\sup\brc*{t\in\mathbb{N}: B_a(t)\le x}<+\infty$, and thus $B_a^{-1}(x) = \max\brc*{t\in\mathbb{N}: B_a(t)\le x}$.

Moreover, by definition, for any suboptimal arm with $\frac{1}{8\Narms \Kinf{+}(\nu_a,\rOpt,\D)}<B(1)$, it holds that  $B_a^{-1}\br*{\frac{1}{8\Narms \Kinf{+}(\nu_a,\rOpt,\D)}} =0$, and such arms do not affect both the time constraint nor the regret bound. Therefore, we also assume w.l.o.g. that $\frac{1}{8\Narms \Kinf{+}(\nu_a,\rOpt,\D)}\ge B(1)$ for all suboptimal arms. Particularly, since $B(1)>0$, this condition also implies that 
$\Kinf{+}(\nu_a,\rOpt,\D)<\infty$ for all $a\notin\Acal_*(\unu)$.

Denote $T_a = B_a^{-1}\br*{\frac{1}{8\Narms \Kinf{+}(\nu_a,\rOpt,\D)}}$. Then, by definition, we have that 
\begin{align*}
    \E_{\unu}\brs*{\nq{T_a}{a}}\le B_a(T_a)\le \frac{1}{8\Narms \Kinf{+}(\nu_a,\rOpt,\D)}\enspace.
\end{align*} In turn, \Cref{prop:lower bound linear}, leads to the bound $\E_{\unu}\brs*{\np{T_a}{a}} \ge \frac{T_a}{2\Narms}$. Finally, as $\E_{\unu}\brs*{\np{T}{a}}$ is nondecreasing in $T$, for any $T\ge \max_{a\notin\Acal_*(\unu)} T_a$, we have that
\begin{align*}
    \Regret(T) 
    &= \sum_{a\notin\Acal_*(\unu)} \dr{a} \E_{\unu}\brs*{\np{T}{a}}\\
    &\ge \sum_{a\notin\Acal_*(\unu)} \dr{a}\E_{\unu}\brs*{\np{T_a}{a}}\\
    &\ge \sum_{a\notin\Acal_*(\unu)} \dr{a}\frac{T_a}{2\Narms}\\
    &= \sum_{a\notin\Acal_*(\unu)} \frac{\dr{a}}{2\Narms}B_a^{-1}\br*{\frac{1}{8\Narms \Kinf{+}(\nu_a,\rOpt,\D)}}\enspace.
\end{align*}
\end{proof}
\clearpage


\section{Upper Bounds}
\label{appendix: upper bounds}
In this part of the appendix, we prove the upper bounds of \Cref{section: upper bounds}. In particular, to make the proof as general as possible, we prove the upper bounds for any confidence intervals that follow some mild regularity assumptions. First recall that we denoted the history of the bandit process by $I_t=\br*{U_0,a_1,q_1,\rOb{1}\cdot q_1,U_1,\dots,a_t,q_t,\rOb{t}\cdot q_t,U_t}$ (see \Cref{section: lower bounds}), and we further define $\F_t = \sigma(I_t)$. Then, we define regular confidence intervals as follows:
\begin{definition}
\label{def:confidence regularity}
Let $T\ge1$ and let $\brc*{\brs*{LCB_t(a),UCB_t(a)}}_{t\in\brs*{T},a\in\brs*{\Narms}}$ be a sequence of confidence intervals such that  $LCB_t(a),UCB_t(a)$ are predictable w.r.t. $\F_t$ for any $t\in\brs*{T}$ and $a\in\brs*{\Narms}$. Then, the confidence intervals are called \emph{regular} if there exists a sequence of events $\brc*{\G_t\in\F_{t-1}}_{t\in\brs*{T}}$ (`good events') such that regardless of the bandit strategy, the following hold:
\begin{enumerate}
    \item For any $t\in\brs*{T}$ and $a\in\brs*{\Narms}$, it holds that $CI_t(a) \triangleq UCB_t(a)-LCB_t(a)\ge0$.
    \item For any $t\in\brs*{T}$, if $\G_t$ holds, then $\rwd{a}\in\brs*{LCB_t(a),UCB_t(a)}$ for all $a\in\brs*{\Narms}$.
    \item The failure probabilities are bounded by $\sum_{t=1}^T \Pr\brc*{\bar{\G}_t}\le C(T)$.
    \item \label{cond:expectation} For any $a\in\brs*{\Narms}$, there exists a function $N_a^g:\brs*{T}\times [0,\dr{a}]\mapsto[0,\infty)$ such that if $\G_t$ holds, then for any $N_a^g(t,\mu)<\nq{t-1}{a}\le t-1$, it holds that $CI_t(a)< \mu$. Moreover, $N_a^g(t,\mu)\le t-1$ for all $a,t$ and $\mu$ and we define $N_a^g(t,0)=t-1$.
    \item \label{cond:almost sure} There exists a function $N^{as}:\brs*{T}\times [0,\infty)\mapsto[0,\infty)$ such that for any $a\in\brs*{\Narms}$ and  $N^{as}(t,\mu)<\nq{t-1}{a}\le t-1$, it holds that $CI_t(a)< \mu$. Moreover, $N_a^g(t,\mu)\le t-1$ for all $a,t$ and $\mu$ and we define $N^{as}(t,0)=t-1$.
\end{enumerate}
Finally, w.l.o.g., we assume that for any $t\ge1$, $a\in\brs*{\Narms}$ and $\mu\in[0,\dr{a}]$, it holds that $N_a^g(t,\mu)\le N^{as}(t,\mu)$, as by definition, we can always replace $N_a^g(t,\mu)$ by $\min\brc*{N_a^g(t,\mu),N^{as}(t,\mu)}$. Equivalently, any bound dependence in $N_a^g(t,\mu)$ can be replaced by $\min\brc*{N_a^g(t,\mu),N^{as}(t,\mu)}$
\end{definition}
Each of the conditions is a reasonable requirement from confidence intervals. The first condition requires that $\brs*{LCB_t(a),UCB_t(a)}$ will be a nonempty interval. The second condition requires it to contain the true mean under some good event, while the third makes sure that the good events hold with sufficiently high probability. The last two conditions characterize the quantities that will affect the performance when using the confidence intervals: condition four quantifies the number of samples $N_a^g(t,\mu)$ required for an arm $a$ to be well-concentrated (up to a confidence level $\mu$) at time $t$ \emph{when the good event holds}. Importantly, it will determine in-expectation regret and querying bounds. We emphasize that $N_a^g$ might depend on the specific arm distribution $\nu_a$ (e.g., be variance-dependent), but cannot depend on any random quantity. The last condition quantifies the number of samples $N^{as}(t,\mu)$ required for the confidence intervals to be smaller than $\mu$ at time $t$ \emph{regardless of the good event}. In particular, this condition must hold for all arms and will determine the \emph{almost-sure} querying guarantees. In practice, all these requirements are extremely mild and hold for most standard confidence intervals, including Hoeffding-based confidence intervals (as in the main paper) and Bernstein-based confidence intervals. We refer the readers to \Cref{appendix:regularity proofs} for the regularity proofs of these confidence intervals. We now state a more general version of \Cref{theorem:upper bound} that characterizes the performance of \Cref{alg: BuFALU} when used with regular confidence intervals:
\begin{theorem}
\label{theorem:upper bound general}
Let $T\ge1$ and let $\brc{\epsilon(t)}_{t\in\brs*{T}}$ be some nonnegative sequence. Also, let $L_\epsilon(T,\dr{}) = \sum_{t=1}^T \Ind{\epsilon(t)\ge\dr{}}$ be the number of times $\epsilon(t)$ exceeds a confidence-level $\dr{}\ge0$ until $T$ and define
\begin{align*}
    \bar{N}_a^g(T,\mu) = \max_{t\in\brs*{T}}N_a^g\br*{t,\max\brc*{\mu,\epsilon(t)}},\quad 
    \bar{N}^{as}(T) = \max_{t\in\brs*{T}}N^{as}(t,\epsilon(t))\enspace.
\end{align*}

Then, when applying \Cref{alg: BuFALU} with regular confidence intervals, the following hold:
\begin{enumerate}
    \item For all $a\in\brs*{\Narms}$, it holds that $\nq{T}{a} \le \bar{N}^{as}(T) +1$.
    \item If there are multiple optimal arms $(\abs{\Acal_*}>1)$, then
    \begin{align*}
        &\Regret(T)\le \sum_{a\notin\Acal_*} \dr{a}\br*{\bar{N}_a^g(T,\dr{a})+L_\epsilon(T,\dr{a})} +\br*{\Narms+C(T)}\dr{\max}, \\
        &\E\brs*{\Bq{T}} \le \sum_{a=1}^{\Narms} \bar{N}_a^g(T,\dr{a}) +  \br*{\Narms+C(T)}\enspace.
    \end{align*}
    \item If the optimal arm $a^*$ is unique, then 
    \begin{align*}
        &\Regret(T)\le \sum_{a\notin\Acal_*} \dr{a}\br*{\bar{N}_a^g\br*{T,\frac{\dr{a}}{2}} + L_\epsilon(T,\dr{a})} +\br*{\Narms+C(T)}\dr{\max}, \\
        &\E\brs*{\Bq{T}} \le \sum_{a\ne a^*} \bar{N}_a^g\br*{T,\frac{\dr{a}}{2}} + \bar{N}_{a^*}^g\br*{T,\frac{\dr{\min}}{2}} +  \br*{\Narms+C(T)}\enspace.
    \end{align*}
\end{enumerate}
\end{theorem}
Notably, the bounds hold for any sequence of nonnegative confidence levels $\brc*{\epsilon(t)}_{t\ge1}$. This stands in sharp contrast to the results of~\cite{efroni2021confidence}, which assumed that the budget is nondecreasing, a requirement that limits the valid sequences of $\epsilon(t)$. However the cost for this generaliztion is the maximization over $t\in\brs*{T}$ in $\bar{N}_a^g(T,\mu)$ and $\bar{N}^{as}(T)$. In particular, the maximization ensures that regardless of $t$, these quantities upper bound the number of samples required for the confidence intervals to be small. Nonetheless, for reasonable choices of $\epsilon(t)$, $\bar{N}_a^g(T,\mu)$ and $\bar{N}^{as}(T)$ can be easily calculated. We refer the readers to \Cref{lemma: hoeffding regular} and \Cref{lemma: bernstein regular}, where we present natural choices of $\epsilon$ for Hoeffding and Bernstein confidence intervals, respectively, and explicitly bound $\bar{N}_a^g$ and $\bar{N}^{as}$. Specifically, one favorable choice is to require $\epsilon(t)$ to be nonincreasing. Then, bounding $L_\epsilon$ by $L_\epsilon(T,\dr{})\le\epsilon^{-1}(\dr{}) \triangleq \sup\brc*{t\ge1: \epsilon(t)\ge \dr{}}$ is (asymptotically) tight and much easier to understand and compute. Finally, we remark that for ease of writing, in both aforementioned lemmas, $B(t)$ is the \emph{per-arm querying budget}. Then, results for total querying budget (as we immediately discuss) are simply obtained by replacing $B(t) \to B(t)/\Narms$. 

When applied with Hoeffding-based confidence intervals, the theorem reduces to the results in \Cref{theorem:upper bound}, and in particular, choosing $\epsilon(t) = \sqrt{\frac{6\Narms\ln t}{B(t)}}$ for a nondecreasing positive budget $B(t)$ also provides budget guarantees, namely, $\Bq{T}\le B(T)+\Narms$. Interestingly, a similar choice with Bernstein-based confidence intervals (explicitly, $\epsilon(t) = \sqrt{\frac{6\Narms\ln t}{B(t)-\Narms}}+14\frac{\Narms\ln t}{B(t)-\Narms}$ for $B(t)>\Narms$) provides regret bounds that depend on the variance of arms $\brc*{V_a}_{a\in\brs*{\Narms}}$ and expected per-arm querying bounds of $\Ocal\br*{4V_a B(t)/\Narms}$. Importantly, since $V_a\le\frac{1}{4}$, this improves the Hoeffding-based bounds and might be dramatically lower when the variance is low.

\clearpage
\subsection{Proof of Theorem \ref{theorem:upper bound general}}
\label{appendix: problem dependent proof}
Before proving the theorem, we start by proving a few key properties on the querying mechanism and played actions of \Cref{alg: BuFALU}:
\begin{lemma}
\label{lemma: gap-query}
Let $t\in\brc*{\Narms+1,\dots,T}$ and assume that \Cref{alg: BuFALU} is run with regular confidence intervals and $\epsilon(t)\ge0$. Then, the following hold: (i) If $q_t=1$, then $CI_t(a_t)>\epsilon(t)$. 
(ii) If the good event $\G_t$ holds and $q_t=0$, then $\dr{a_t}\le \epsilon(t)$. 
\end{lemma}
\begin{proof}

 \hfill
 
\textbf{Part (i).} When $q_t=1$, recall that $a_t=c_t\in\brc*{l_t,u_t}$. We divide into two cases. If $a_t=l_t$, then a necessary condition for querying is that
\begin{align*}
    \epsilon(t) < UCB_t(c_t) - LCB_t(l_t) = UCB_t(l_t)-LCB_t(l_t) = CI_t(l_t)\enspace,
\end{align*}
as required. Otherwise, $a_t=u_t$. Then, by the definition of $l_t$, we have that $LCB_t(u_t)\le LCB_t(l_t)$ and the querying condition implies that 
\begin{align*}
    \epsilon(t) 
    < UCB_t(c_t) - LCB_t(l_t) 
    = UCB_t(u_t)-LCB_t(l_t)
    \le UCB_t(u_t) - LCB_t(u_t)
    = CI_t(u_t),
\end{align*}
which concludes this part of the proof. 

\textbf{Part (ii).} When $q_t=0$, we have that $a_t=l_t$; therefore, we need to show that $\dr{l_t}\le \epsilon(t)$. For $q_t=0$ to hold, at least one of the following two options must occur: 
\begin{enumerate}
    \item First option: $UCB_t(u_t)\le LCB_t(l_t)$, which implies that $UCB_t(a)\le LCB_t(l_t)$ for all $a\ne l_t$. In particular, since the good event $\G_t$ holds, we have that $LCB_t(l_t)\le \rwd{l_t}$ and $UCB_t(a)\ge \rwd{a}$ for all $a\ne l_t$, and thus, $\rwd{a}\le\rwd{l_t}$ for all $a\ne l_t$. Therefore, $l_t$ is an optimal arm and $\dr{l_t}=0\le \epsilon(t)$. 
    \item Second option: $UCB_t(c_t) - LCB_t(l_t)\le\epsilon(t)$. Assume w.l.o.g. that $l_t$ is not an optimal arm, as otherwise $\dr{l_t}=0$ and the claim naturally holds, and let $a^*\ne l_t$ be an optimal arm. Specifically, under the good event $\G_t$, it holds that 
\begin{align*}
    UCB_t(u_t) = \max_{a\ne l_t} UCB_t(a) \ge UCB_t(a^*) \ge\rwd{a^*} = \rOpt\enspace.
\end{align*}
We now divide into the cases where $c_t=u_t$ and $c_t=l_t$. 
If $c_t=u_t$, we use the fact that under $\G_t$, $LCB_t(l_t)\le \rwd{l_t}$, and we get that
\begin{align*}
    \dr{l_t} = \rOpt - \rwd{l_t} \le UCB_t(u_t)-LCB_t(l_t) \le \epsilon(t)\enspace,
\end{align*}
where the last inequality is by the assumption that $UCB_t(c_t) - LCB_t(l_t)\le\epsilon(t)$ for $c_t=u_t$. On the other hand, if $c_t=l_t$, this assumption is equivalent to $CI_t(l_t)\le \epsilon(t)$. In turn, since $l_t=c_t\in\argmax_{a\in\brc*{u_t,l_t}} CI_t(a)$, it also implies that $CI_t(u_t)\le \epsilon(t)$. Finally, under the good event, and since $l_t\in\argmax_a LCB_t(a)$, we have that 
\begin{align*}
    \rwd{l_t} \ge LCB_t(l_t) \ge LCB_t(u_t)
\end{align*}
Then, recalling that $UCB_t(u_t)\ge\rOpt$ and $CI_t(u_t)\le\epsilon(t)$ leads to the result since
\begin{align*}
    \dr{l_t} = \rOpt - \rwd{l_t} \le UCB_t(u_t) - LCB_t(u_t)
    = CI_t(u_t)\le \epsilon(t)\enspace.
\end{align*} 
\end{enumerate}
\end{proof}

\clearpage

Given this lemma, we can now prove \Cref{theorem:upper bound general}:
\begin{proof}

We remark that at the first $\Narms$ rounds, each arm is played and queried once. Then, it can be verified that all bounds hold for $T\le \Narms$. Thus, throughout the proof, we assume w.l.o.g. that $T>\Narms$.

\textbf{Almost-sure querying bound:} Assume in contradiction that $\nq{T}{a}>\bar{N}^{as}(T)+1$ for some $a\in\brs*{\Narms}$. Therefore, there exists a time index $t\in\brc*{\Narms+1,\dots,T}$ such that action $a$ was queried ($a_t=a$ and $q_t=1$) and $\nq{t-1}{a}>\bar{N}^{as}(T)\ge N^{as}(t,\epsilon(t))$, where the last inequality is by the definition of $\bar{N}^{as}$. In particular, since $\nq{t-1}{a}\le t-1$, this condition cannot hold if $N^{as}(t)=t-1$. Therefore, $N^{as}(t,\epsilon(t))<t-1$, and by the definition of $N^{as}$, we have that $CI_t(a)<\epsilon(t)$. This come in contradiction to the fact that $a_t=a$ was queried, since part (i) of \Cref{lemma: gap-query} implies that $CI_t(a)=CI_t(a_t)>\epsilon(t)$. This proves that $\nq{T}{a}\le \bar{N}^{as}(T)+1$ for all $a\in\brs*{\Narms}$ and concludes the first part of the proof.

\textbf{Count decomposition:} To derive the expected regret and querying bounds, we start with a general decomposition that will be relevant for all the required results. Specifically, the number of plays of any arm $a$ under the good event $\G_t$ (defined in \Cref{def:confidence regularity}) can be bounded as follows:
\begin{align} 
    \sum_{t=1}^T \Ind{\G_t,a_t=a}
    = \sum_{t=1}^T \Ind{\G_t,a_t=a,q_t=0} + \sum_{t=1}^T \Ind{\G_t,a_t=a,q_t=1}\enspace. \label{eq:count decomp}
\end{align}
By \Cref{lemma: gap-query} (part (i)), the first term can be bounded by 
\begin{align} 
    \sum_{t=1}^T \Ind{\G_t,a_t=a,q_t=0} 
    &= \sum_{t=1}^T \Ind{\G_t,l_t=a,q_t=0} \tag{$a_t=l_t$ when $q_t=0$} \\
    &\le \sum_{t=1}^T \Ind{\G_t,\dr{a}\le \epsilon(t)} \tag{By \Cref{lemma: gap-query}} \\
    & \le L_\epsilon(T,\dr{a})\enspace, \label{eq:count decomp no query}
\end{align}
where the last relation is by the definition of $L_\epsilon$. For the second term of \eqref{eq:count decomp}, let $\dr{}\ge0$ be some parameter that will be determined later. Then, under the good event, we bound
\begin{align}
    \sum_{t=1}^T \Ind{\G_t,a_t=a,q_t=1} 
    &= \underbrace{\sum_{t=1}^T \Ind{\G_t,a_t=a,q_t=1,\nq{t-1}{a} \le \bar{N}_a^g(T,\dr{})}}_{\le \bar{N}_a^g(T,\dr{})+1}  \nonumber\\
    &\quad+ \sum_{t=1}^T \underbrace{\Ind{\G_t,a_t=a,q_t=1,\nq{t-1}{a} > \bar{N}_a^g(T,\dr{})}}_{(*)=0\textrm{ for an appropriate $\dr{}$}} \nonumber\\
    & \le \bar{N}_a^g(T,\dr{})+1\enspace,\label{eq:count decomp query}
\end{align}
which leads to a total bound of 
\begin{align}
    \sum_{t=1}^T \Ind{\G_t,a_t=a}\le \bar{N}_a^g(T,\dr{})+1 + L_\epsilon(T,\dr{a})\enspace. \label{eq: count decomp result}
\end{align}

\textbf{Proving the bound of \Cref{eq:count decomp query}.} The bound on the first term holds since $\nq{t}{a}$ starts from zero and increases by one every time action $a$ was played and queried. We now show that depending on the assumptions and specific arms, $\dr{}$ can be chosen such that under $\G_t$, the events $(*)$ of the second term cannot occur. This is trivially true if $t\le \Narms$, since each arm $a$ is queried at time $t=a$ with $\nq{t-1}{a}=0$ and $N_a^{\dr{a}}(a)\ge0$. Therefore, w.l.o.g., we focus on $t>\Narms$. Moreover, by the definition of $\bar{N}_a^g$, we have that $\bar{N}_a^g(T,\dr{})\ge N_a^g(t,\max\brc*{\dr{},\epsilon(t)})$. Then, it suffices to show that under $\G_t$, action $a$ cannot be queried if $\nq{t-1}{a}>N_a^g(t,\max\brc*{\dr{},\epsilon(t)})$. 

To show this, first recall that $\nq{t-1}{a}\le t-1$; thus, if $N_a^g(T,\max\brc*{\dr{},\epsilon(t)})=t-1$, this condition can never hold. Otherwise, $N_a^g(T,\max\brc*{\dr{},\epsilon(t)})<t-1$, and by the regularity of the confidence interval, when $\G_t$ holds and $\nq{t-1}{a}>N_a^g(T,\max\brc*{\dr{},\epsilon(t)})$ we have that 
\begin{align*}
    CI_t(a_t) = CI_t(a) < \max\brc*{\dr{},\epsilon(t)}\enspace. 
\end{align*}
Importantly, if $\max\brc*{\dr{},\epsilon(t)}=\epsilon(t)$, then the condition of $CI_t(a_t)<\epsilon(t)$ implies that an action cannot be queried (namely, $q_t\ne1$), by the first part of \Cref{lemma: gap-query}. Therefore, w.l.o.g., we assume that $\max\brc*{\dr{},\epsilon(t)}=\dr{}$ and prove by contradiction that for the right choice of $\dr{}$, $a_t=a$ cannot be queried (under $\G_t$) if $CI_t(a)<\dr{}$. This will imply that all indicators in $(*)$ are equal to zero and will conclude the proof of \Cref{eq:count decomp query}. To do so, divide into the cases where $a$ is optimal or suboptimal and problems with unique or multiple optimal arms. In all cases, assume in contradiction that $a$ is queried and recall that it implies that $a=a_t=c_t\in\brc*{l_t,u_t}$.
\begin{enumerate}[label=(\roman*)]
    \item \underline{$a$ is an optimal arm and $\dr{}=0$}. \\
    By the regularity of the confidence intervals, $CI_t(a_t)\ge0$. Therefore, the condition $CI_t(a_t)<0=\dr{}$ can never hold.
    
    \item \underline{$a$ is strictly suboptimal, $a_t=a=l_t$ and $\dr{}\le\dr{a}$}. \\
    Let $a^*$ be any optimal arm $a^*$ (which is different than $l_t=a$ since it is suboptimal). Then, the good event implies that $UCB_t(a^*)\ge\rOpt$, and thus
    \begin{align*}
        UCB_t(u_t) = \max_{a'\ne l_t}UCB_t(a') \ge UCB_t(a^*)\ge \rOpt\enspace.
    \end{align*}
    On the other hand, the good event also implies that $LCB_t(l_t)\le \rwd{l_t}=\rwd{a}$, and by definition of $l_t$, it holds that $LCB_t(u_t)\le LCB_t(l_t)\le \rwd{a}$. Combining both inequalities, we get that
    \begin{align*}
        CI_t(u_t) = UCB_t(u_t) - LCB_t(u_t) \ge \rOpt - \rwd{a}=\dr{a}\enspace.
    \end{align*}
    However, since $a_t=l_t$ was played and queried, by the definition of $c_t$, it must hold that $CI_t(l_t)\ge CI_t(u_t)\ge\dr{a}$, in contradiction to the fact that $CI_t(l_t)=CI_t(a)<\dr{}\le \dr{a}$.
    
    \item \underline{$a$ is strictly suboptimal, $a_t=a=u_t$ and $\dr{}\le\dr{a}$: multiple optimal arms}. \\
    Since there are at least two optimal arm, there exists at least one optimal arm $a^*$ such that $l_t\ne a^*$. Specifically, under the good event, we have that $UCB_t(a^*)\ge \rOpt$, and then
    \begin{align*}
        UCB_t(u_t) = \max_{a'\ne l_t} UCB_t(a') \ge UCB_t(a^*) = \rOpt\enspace.
    \end{align*}
    Moreover, under $\G_t$, we have that $LCB_t(u_t)\le \rwd{u_t}$, and combining both leads to
    \begin{align*}
        CI_t(u_t) = UCB_t(u_t) - LCB_t(u_t) \ge \rOpt - \rwd{u_t} = \dr{u_t}\enspace,
    \end{align*}
    in contradiction to the fact that $u_t=a$ and $CI_t(a)<\dr{}\le\dr{a}$. 
    
    \item \underline{$a$ is strictly suboptimal, $a_t=a=u_t$ and $\dr{}\le\dr{a}/2$: unique optimal arm $a^*$}. \\
    Since $a_t=c_t=u_t$, it holds that  $CI_t(u_t)<\dr{}\le \frac{\dr{u_t}}{2}$. In particular, $\G_t$ implies that
    \begin{align*}
        UCB_t(u_t) = LCB_t(u_t) + CI_t(u_t) \le \rwd{u_t}+CI_t(u_t) < \rwd{u_t} + \frac{\dr{u_t}}{2} = \rOpt-\frac{\dr{u_t}}{2}\enspace.
    \end{align*}
    On the other hand, under $\G_t$, it holds that $UCB_t(a^*)\ge \rOpt>UCB_t(u_t)$. Then, the scenario of $a=u_t$ can only happen if $l_t=a^*$ (as $u_t$ maximizes the UCB only on actions different than $l_t$), and we get that
    \begin{align*}
        LCB_t(l_t) 
        &= UCB_t(l_t)-CI_t(l_t) 
        = UCB_t(a^*) - CI_t(l_t)
        \ge \rOpt - CI_t(l_t)\enspace.
    \end{align*}
    Finally, since $a_t=c_t=u_t$, and by the definition of $c_t$, it holds that $CI_t(l_t) \le CI_t(u_t)<\frac{\dr{u_t}}{2}$, and we have that
    \begin{align*}
        LCB_t(l_t) 
        > \rOpt -\frac{\dr{u_t}}{2}
        > UCB_t(u_t)\enspace,
    \end{align*}
    in contradiction to the querying rule.
    
    \item \underline{$a=a^*$ is a unique optimal arm and $\dr{}\le\dr{\min}/2$}. \\
    In this case, by the good event and the requirement that $CI_t(a)=CI_t(a^*)<\dr{}$,
    \begin{align*}
        LCB_t(a^*) = UCB_t(a^*) - CI_t(a^*) \ge \rOpt - CI_t(a^*) >\rOpt - \dr{}\ge \rOpt - \frac{\dr{\min}}{2}\enspace.
    \end{align*}
    Specifically, $LCB_t(a^*)>\rwd{a}\ge LCB_t(a)$ for all $a\ne a^*$ and thus $l_t=a^*$. Furthermore, by the uniqueness of the optimal arm, $u_t\in\argmax_{a\ne a^*} UCB_t(a)$ is a strictly suboptimal arm. Therefore, under $\G_t$, we have that
    \begin{align*}
        UCB_t(u_t) = LCB_t(u_t) + CI_t(u_t) \le \rwd{u_t} +CI_t(u_t) \le \rOpt-\dr{\min} +CI_t(u_t)
    \end{align*}
    where the last inequality is since $u_t$ is a suboptimal arm (and so, $\dr{u_t}\ge\dr{\min}$). Finally, recall that $q_t=1$ only if $UCB_t(u_t)>LCB_t(l_t)$. However, the previous inequalities imply that this condition can only hold if $CI_t(u_t)>\frac{\dr{\min}}{2}>CI_t(a^*)=CI_t(l_t)$, and then $c_t=\argmax_{a\in\brc*{u_t,l_t}} CI_t(a) =u_t$. Thus, if an arm is indeed queried, it is a strictly suboptimal arm, in contradiction to the requirement that $a=a^*$ is queried.
\end{enumerate}
Overall, when there are multiple optimal arms, cases (i)-(iii) all hold for the choice $\dr{}=\dr{a}$. When the optimal arm is unique, cases (ii), (iv) and (v) hold with $\dr{}=\dr{a}/2$ for suboptimal arms and $\dr{}=\dr{\min}/2$ for the optimal arm.

\textbf{Regret and queries bounds.} We now show how the bounds of \eqref{eq:count decomp query} and \eqref{eq: count decomp result} can be used to derive the desired regret and querying bounds. To bound the expected regret, we use \Cref{eq: count decomp result} with $\dr{}=\mu_a$ for all suboptimal arms $a\notin\Acal_*$, where $\mu_a=\dr{a}$ if there are multiple optimal arms and $\mu_a=\dr{a}/2$ if the optimal arm is unique. Doing so while using the failure probabilities of the good event $\G_t$, we get:
\begin{align*}
    \Regret(T) = \sum_{t=1}^T \E\brs*{\dr{a_t}}
    &\le \sum_{t=1}^T \E\brs*{\dr{a_t}\Ind{\G_t}} + \dr{\max}\sum_{t=1}^T \E\brs*{\Ind{\bar{\G}_t}} \\
    & = \sum_{a\notin\Acal_*} \dr{a}\E\brs*{\sum_{t=1}^T \Ind{\G_t,a_t=a}} + \dr{\max}\underbrace{\sum_{t=1}^T \Pr\brc*{\bar{\G}_t}}_{\le C(T)} \\
    & \le \sum_{a\notin\Acal_*} \dr{a}\br*{\bar{N}_a^g(T,\mu_a)+ L_\epsilon(T,\dr{a}) + 1} + \dr{\max}C(T) \\
    & \le \sum_{a\notin\Acal_*} \dr{a}\br*{\bar{N}_a^g(T,\mu_a)+ L_\epsilon(T,\dr{a})} + \br*{\Narms+C(T)}\dr{\max}\enspace.
\end{align*}
Notice that substituting $\mu_a$ leads to the desired bound, whether there is a unique or multiple optimal arms. We similarly use \Cref{eq:count decomp query} to bound the expected number of queries. We still choose the same values of $\mu_a$ for suboptimal arms, but for optimal arms, we let $\mu_a=0=\dr{a}$ for all $a\in\Acal_*$ when there are multiple optimal arms and $\mu_{a^*}=\dr{\min}/2$ for a unique optimal arm. Then, as in the regret bound, we get
\begin{align*}
    \E\brs*{\Bq{T}} 
    = \sum_{t=1}^T \E\brs*{\Ind{q_t=1}} 
    &\le \sum_{t=1}^T \E\brs*{\Ind{\G_t,q_t=1}} +  \sum_{t=1}^T \E\brs*{\Ind{\bar{\G}_t}}\\
    &= \sum_{a\in\brs*{\Narms}}\E\brs*{\sum_{t=1}^T \Ind{\G_t,a_t=a,q_t=1}} +  \underbrace{\sum_{t=1}^T \Pr\brc*{\bar{\G}_t}}_{\le C(T)}\\
    &\le \sum_{a\in\brs*{\Narms}}(\bar{N}_a^g(T,\mu_a)+1) +  C(T) \\
    & =\sum_{a\in\Acal_*}\bar{N}_a^g(T,\mu_a) + \sum_{a\notin\Acal_*}\bar{N}_a^g(T,\mu_a) + \Narms +  C(T)\enspace.
\end{align*}
and one can easily verify that substituting $\mu_a$ leads to both desired bounds.
\end{proof}
\begin{remark}
\label{remark:adaptive adversary}
Notice that the bounds of \eqref{eq:count decomp query} and \eqref{eq: count decomp result} hold even if $\epsilon(t)$ is $\F_t$ predictable, e.g., if the sequence $\brc*{\epsilon(t)}_{t\ge1}$ is chosen by an adaptive adversary. Specifically, when this is the case, all bounds remain the same, except an expectation that should be taken on $\bar{N}_a^g(T,\mu)$ and $L_\epsilon(T,\dr{})$. Similarly, an expectation can be taken over any stochastic source that affects $\epsilon(t)$ and is independent of $\F_t$.
\end{remark}


\clearpage
\subsection{Problem-Independent Upper Bound}
\label{appendix: problem independent proof}
In this section, we generalize \Cref{prop: problem independent upper} to the setting presented in \Cref{appendix: upper bounds}.
\begin{proposition}
\label{prop: problem independent upper general}
Under the notations of \Cref{theorem:upper bound general}, let $T\ge1$ and assume that $\epsilon(t)\ge0$ for all $t\in\brs*{T}$ and that \Cref{alg: BuFALU} is applied with regular confidence intervals. Also, assume that there exist a function $M(t,\dr{0})$ such that for all $\dr{0}>0$ and all arms $a\in\brs*{\Narms}$ with $\dr{a}> \dr{0}$, it holds that $\dr{a}\bar{N}_a^{g}(T,\dr{a}) \le M(T,\dr{0})$. Then, 
\begin{align*}
    \Regret(T)\le \inf_{\dr{0}>0}\brc*{2\Narms M\br*{T,\frac{\dr{0}}{2}} + \dr{0}T} + \sum_{t=1}^T\epsilon(t) + \br*{\Narms+C(T)}\dr{\max}\enspace.
\end{align*}
\end{proposition}
See that for Hoeffding-based confidence intervals, we have that $\bar{N}_a^g(T,\mu) \le \frac{6\ln T}{\mu^2}$ (by \Cref{lemma: hoeffding regular}), so for any $\dr{a}>\dr{0}$, it holds that $\dr{a}\bar{N}_a^g(T,\dr{a}) = \frac{6\ln T}{\dr{a}} \le \frac{6\ln T}{\dr{0}} \triangleq M\br*{T,\dr{0}}$. Then, the infimum in the bound is achieved for $\dr{0} = \sqrt{\frac{24\Narms\ln T}{T}}$, which leads to the bound in the main paper (\Cref{prop: problem independent upper}). Alternatively, when using Bernstein-type confidence bounds, we can bound the variance of arms by $V_a\le\frac{1}{4}$ and obtain a similar bound (by \Cref{lemma: bernstein regular}): 
\begin{align*}
    \dr{a}\bar{N}_a^g(T,\dr{a}) \le \frac{24V_a\ln T}{\dr{a}}+52\ln T +\dr{a} \le \frac{6\ln T}{\dr{0}}+52\ln T +\dr{\max} \triangleq M\br*{T,\dr{0}}\enspace.
\end{align*}
On the other hand, if the variance of all suboptimal arms is upper bounded by $V_a\le V$, this can be used to achieve an improved variance-dependent regret bound.
\begin{proof}
As in the problem-dependent bound of \Cref{theorem:upper bound general}, we decompose the regret by:
\begin{align*}
    \Regret(T) = \sum_{t=1}^T \E\brs*{\dr{a_t}}
    &\le \sum_{t=1}^T \E\brs*{\dr{a_t}\Ind{\G_t}} + \dr{\max}\sum_{t=1}^T \E\brs*{\Ind{\bar{\G}_t}} \\
    & \le \sum_{a\notin\Acal_*} \dr{a}\E\brs*{\sum_{t=1}^T \Ind{\G_t,a_t=a}} + \dr{\max}C(T)\enspace.
\end{align*}
Next, for any $\dr{0}>0$, we can bound the regret by 
\begin{align}
    \Regret(T) 
    &\le \underbrace{\sum_{a:\dr{a}>\dr{0}} \dr{a}\E\brs*{\sum_{t=1}^T \Ind{\G_t,a_t=a}}}_{(i)} + \underbrace{\sum_{a:\dr{a}\in(0,\dr{0}]} \dr{a}\E\brs*{\sum_{t=1}^T \Ind{\G_t,a_t=a}}}_{(ii)}\nonumber\\
    &\quad +\dr{\max}C(T) \enspace. \label{eq:decomp problem independent}
\end{align}
To bound term $(i)$ we further decompose it to
\begin{align*}
    (i)
    & = \E\brs*{\sum_{a:\dr{a}>\dr{0}} \dr{a}\sum_{t=1}^T \Ind{\G_t,a_t=a,q_t=1}} + \E\brs*{\sum_{a:\dr{a}>\dr{0}}\sum_{t=1}^T  \dr{a}\Ind{\G_t,a_t=a,q_t=0}} \\
    & \le \E\brs*{\underbrace{\sum_{a:\dr{a}>\dr{0}} \dr{a}\sum_{t=1}^T \Ind{\G_t,a_t=a,q_t=1}}_{(*)}} + \E\brs*{\underbrace{\sum_{t=1}^T  \dr{a_t}\Ind{\G_t,q_t=0}}_{(**)}}
\end{align*}

For term $(*)$, if there are multiple optimal arms, then by \Cref{eq:count decomp query} with $\dr{}=\dr{a}$,
\begin{align*}
    \sum_{a:\dr{a}>\dr{0}} \dr{a}\sum_{t=1}^T \Ind{\G_t,a_t=a,q_t=1}
    \le \sum_{a:\dr{a}>\dr{0}} \dr{a}\br*{\bar{N}_a^{g}(T,\dr{a})+1}
    \le \Narms M\br*{T,\frac{\dr{0}}{2}} +\Narms \dr{\max}\enspace,
\end{align*}
where the last inequality is since the gaps of all arms in the summation are larger than $\dr{0}$, and specifically larger than $\dr{0}/2$, and by bounding the number of arms with gaps larger than $\dr{0}$ by $\Narms$. Alternatively, if there is a unique optimal arm, we can bound using \Cref{eq:count decomp query} with $\dr{}=\dr{a}/2$:
\begin{align*}
    \sum_{a:\dr{a}>\dr{0}} \dr{a}\sum_{t=1}^T \Ind{\G_t,a_t=a,q_t=1}
    &\le \sum_{a:\dr{a}>\dr{0}} \dr{a}\br*{\bar{N}_a^{g}\br*{T,\frac{\dr{a}}{2}}+1}\\
    &\le 2\sum_{a:\dr{a}>\dr{0}} \frac{\dr{a}}{2}\bar{N}_a^{g}\br*{T,\frac{\dr{a}}{2}} + \Narms\dr{\max}\\
    &\le 2\Narms M\br*{T,\frac{\dr{0}}{2}} +\Narms \dr{\max}\enspace,
\end{align*}
where the last inequality is again by the definition of $M\br*{T,\frac{\dr{0}}{2}}$. Combining both cases, we get 
\begin{align*}
    (*) \le 2 \Narms M\br*{T,\frac{\dr{0}}{2}} +\Narms \dr{\max}\enspace.
\end{align*}
The remaining summation $(**)$ can be bounded using the fact that if $q_t=0$, then $\dr{a_t}\le\epsilon(t)$ (by \Cref{lemma: gap-query}), and thus
\begin{align*}
    (**)=\sum_{t=1}^T \dr{a_t}\Ind{\G_t,q_t=0}
    &\le \sum_{t=1}^T \dr{a_t}\Ind{\dr{a_t}\le\epsilon(t)} 
    \le \sum_{t=1}^T \epsilon(t)\Ind{\dr{a_t}\le\epsilon(t)}
     \le \sum_{t=1}^T \epsilon(t)\enspace,
\end{align*}
which lead to the bound of 
\begin{align*}
    (i)\le  2\Narms M\br*{T,\frac{\dr{0}}{2}} + \sum_{t=1}^T \epsilon(t) + \Narms \dr{\max}
\end{align*}
Finally, we bound term $(ii)$ as follows:
\begin{align*}
    \sum_{a:\dr{a}\in(0,\dr{0}]} \dr{a}\E\brs*{\sum_{t=1}^T \Ind{\G_t,a_t=a}} 
    &\le \dr{0} \E\brs*{\sum_{a=1}^\Narms\sum_{t=1}^T \Ind{\G_t,a_t=a}} 
    = \dr{0}\E\brs*{\sum_{t=1}^T \Ind{\G_t}} 
    \le \dr{0}T \enspace.
\end{align*}
Substituting $(i)$ and $(ii)$ into \Cref{eq:decomp problem independent} and taking the infimum over all possible choices of $\dr{0}>0$ leads to the desired bound.
\end{proof}
\clearpage

\subsection{Regularity Proofs}
\label{appendix:regularity proofs}
\subsubsection{Hoeffding-Based Confidence Intervals}
\label{appendix:regularity hoeffding}
Hoeffding-based confidence intervals are probably the most commonly used confidence intervals for bounded (and subgaussian) rewards. Specifically, for rewards bounded in $[0,1]$, they are defined by 
\begin{align*}
    UCB_t(a) = \rEst{t-1}{a} + \sqrt{\frac{3\ln t}{2\nq{t-1}{a}}}
    \qquad \textrm{and}\qquad 
    LCB_t(a) = \rEst{t-1}{a} - \sqrt{\frac{3\ln t}{2\nq{t-1}{a}}}\enspace,
\end{align*}
and if $\nq{t-1}{a}=0$, we define $UCB_t(a)=+\infty$ and $LCB_t(a)=-\infty$. In the following, we prove that such confidence intervals are regular:
\begin{lemma}
\label{lemma: hoeffding regular}
Assume that the rewards are bounded in $\brs*{0,1}$. Then, for any $T\ge1$, Hoeffding-based confidence intervals are regular w.r.t. the events $\G_t=\brc{\forall a\in\brs*{\Narms}: \rwd{a}\in\brs*{LCB_t(a),UCB_t(a)}}$ and the functions
\begin{align*}
    C(t) = 2\Narms, \qquad
     N_a^g(t,\mu) = N^{as}(t,\mu) = \min\brc*{\frac{6\ln t}{\mu^2},t-1} \enspace,
\end{align*}
Specifically, $\bar{N}_a^g(T,0) = \bar{N}^{as}(T)$. Moreover, for any $\mu>0$, it holds that $\bar{N}_a^g(T,\mu)\le \frac{6\ln T}{\mu^2}$. Finally, if $B(t)$ is a positive nondecreasing sequence and $\epsilon(t)=\sqrt{\frac{6\ln t}{B(t)}}$, or, alternatively, $\epsilon(t)$ is a nonnegative nonincreasing sequence, then $\bar{N}_a^g(T,\mu) = \min\brc*{\frac{6\ln T}{\max\brc*{\mu^2,\epsilon^2(t)}},T-1}$.
\end{lemma}
\begin{proof}
Notice that $CI_t(a)=\sqrt{\frac{6\ln t}{\nq{t-1}{a}}}$. We now check the conditions by their order:
\begin{enumerate}
    \item The condition that $CI_t(a)\ge0$ trivially holds for any $\nq{t-1}{a}\ge1$ and holds by definition when $\nq{t-1}{a}=0$ (and $CI_t(a)=\infty$).
    \item By definition, for any $t\ge1$, if $\G_t$ holds, then $\rwd{a}\in\brs*{LCB_t(a),UCB_t(a)}$ for all $a\in\brs*{\Narms}$, as required.
    
    \item We now bound $\Pr\brc*{\bar{\G}_t}$. For any $a\in\brs*{\Narms}$, let $X_{a}(1),\dots X_a(T)$ be i.i.d. random variables of the same distribution as arm $a$, and we let the observed reward at the $\nq{t}{a}^{th}$ time the arm was played be $R_t=X_a(\nq{t}{a})$. Then, we can write:
    \begin{align*}
        \Pr\brc*{\bar{\G}_t}
        &= \Pr\brc*{\exists a\in\brs*{\Narms}: \rwd{a}\notin\brs*{LCB_t(a),UCB_t(a)}}\\
        &= \Pr\brc*{\exists a\in\brs*{\Narms}: \abs*{\rEst{t-1}{a}-\rwd{a}}>\sqrt{\frac{3\ln t}{2\nq{t-1}{a}}}}\\
        & \overset{(1)}{\le} \sum_{a=1}^\Narms \Pr\brc*{\abs*{\rEst{t-1}{a}-\rwd{a}}>\sqrt{\frac{3\ln t}{2\nq{t-1}{a}}}} \\
        & \overset{(2)}{\le} \sum_{a=1}^\Narms \sum_{n=1}^{t-1} \Pr\brc*{\abs*{\rEst{t-1}{a}-\rwd{a}}>\sqrt{\frac{3\ln t}{2\nq{t-1}{a}}},\nq{t-1}{a}=n}\\
        & \overset{(3)}{=} \sum_{a=1}^\Narms \sum_{n=1}^{t-1} \Pr\brc*{\abs*{\frac{1}{n}\sum_{k=1}^n X_a(k)-\rwd{a}}>\sqrt{\frac{3\ln t}{2n}},\nq{t-1}{a}=n}\\
        & \le \sum_{a=1}^\Narms \sum_{n=1}^{t-1} \Pr\brc*{\abs*{\frac{1}{n}\sum_{k=1}^n X_a(k)-\rwd{a}}>\sqrt{\frac{3\ln t}{2n}}}\\
        & \overset{(4)}{\le} \sum_{a=1}^\Narms \sum_{n=1}^{t-1} \frac{1}{t^3}\\
        & \le \frac{\Narms}{t^2}\enspace.
    \end{align*}
    Relations $(1)$ and $(2)$ are due to the union bound over actions $a\in\brs*{\Narms}$ and  $\nq{t-1}{a}\in\brc*{0,\dots,t-1}$, respectively. Moreover, notice that the relevant event cannot occur when $\nq{t-1}{a}=0$, so we only treat values larger than $1$. In $(3)$ we substituted $\nq{t-1}{a}=n$ and wrote the empirical means using $X_{a}(t)$. Finally, $(4)$ relies on Hoeffding's inequality for i.i.d. bounded random variables supported by $[0,1]$. Summing over all time indices leads to the desired result:
    \begin{align*}
        \sum_{t=1}^T \Pr\brc*{\bar{\G}_t}
        \le \sum_{t=1}^T \frac{\Narms}{t^2}
        \le \sum_{t=1}^\infty \frac{\Narms}{t^2}
        \le 2\Narms
        =C(T)\enspace.
    \end{align*}
    \item If $N_a^g(t,\mu)=t-1$, then the event that $N_a^g(t,\mu)<\nq{t-1}{a}\le t-1$ can never happen, so assume w.l.o.g. that $N_a^g(t,\mu)<t-1$ and $N_a^g(t,\mu)=\frac{6\ln t}{\mu^2}$. In particular, $\mu>0$ and $t>1$, and it can be easily verified that whether the good event holds or not, if $\nq{t-1}{a}>N_a^g(t,\mu)= \frac{6\ln t}{\mu^2}$, then 
    \begin{align*}
        CI_t(a) =\sqrt{\frac{6\ln t}{\nq{t-1}{a}}}
        < \sqrt{\frac{6\ln t}{N_a^g(t,\mu)}} = \mu\enspace.
    \end{align*}
    \item The proof follows exactly as the previous part, as its result did not depend on the good event.
\end{enumerate}
We now prove the claims at the end of the lemma. First, for any nonnegative sequence $\brc*{\epsilon(t)}_{t\ge1}$, it holds that
\begin{align*}
    \bar{N}_a^g(T,\mu) 
    &= \max_{t\in\brs*{T}} N_a^g(t,\max\brc*{\mu,\epsilon(t)}) \\
    &= \max_{t\in\brs*{T}} \brc*{\min\brc*{\frac{6\ln t}{\max\brc*{\mu^2,\epsilon^2(t)}},t-1}}\\
    &\le \max_{t\in\brs*{T}} \brc*{\min\brc*{\frac{6\ln t}{\mu^2},t-1}} \\
    & \le \min\brc*{\frac{6\ln T}{\mu^2},T-1}
\end{align*}
where the last inequality is since the logarithmic and constant functions are nondecreasing in $t$, and the minimum of nondecreasing functions is nondecreasing (\Cref{claim:minimum increasing}). Finally, we write
\begin{align*}
    N_a^g(t,\max\brc*{\mu,\epsilon(t)})
     = \min\brc*{\frac{6\ln t}{\mu^2},\frac{6\ln t}{\epsilon^2(t)},t-1}\enspace.
\end{align*}
Moreover, if either $\epsilon(t)=\sqrt{\frac{6\ln t}{B(t)}}$, for nondecreasing positive $B(t)$, or $\epsilon(t)\ge0$ is nonincreasing, then $\frac{6\ln t}{\epsilon^2(t)}$ is nondecreasing in $t$. Thus, $ N_a^g(t,\max\brc*{\mu,\epsilon(t)})$ is a minimum of nondecreasing functions and is nondecreasing by itself (by \Cref{claim:minimum increasing}). In turn, this implies that
\begin{align*}
    \bar{N}_a^g(T,\mu) 
    &= \max_{t\in\brs*{T}} N_a^g(t,\max\brc*{\mu,\epsilon(t)})
    = N_a^g(T,\max\brc*{\mu,\epsilon(T)})\\
    &= \min\brc*{\frac{6\ln T}{\max\brc*{\mu^2,\epsilon^2(t)}},T-1}\enspace.
\end{align*}
\end{proof}

\clearpage
\subsubsection{Bernstein-Based Confidence Intervals}
\label{appendix:regularity bernstein}
Bernstein-based confidence intervals are confidence intervals that depend on the variance on the reward distributions. Specifically, we apply Empirical-Bernstein bounds \citep{maurer2009empirical}, which depend on the empirical variance of arms and obviate the need to know the true variances of arms. Formally, if all arm distributions are bounded in $[0,1]$ and of variances $\brc*{V_a}_{a\in\brs*{\Narms}}$, we denote the unbiased empirical estimator of the variance by
\begin{align*}
    \hat{V}_t(a) = \frac{\nq{t}{a}}{\nq{t}{a}-1}\br*{\frac{1}{\nq{t}{a}}\sum_{s=1}^t R_t^2\Ind{a_t=a} - \br*{\rEst{t}{a}}^2}\enspace,
\end{align*}
where we define $\hat{V}_t(a)=0$ if $\nq{t}{a}<2$. Then, Bernstein-based confidence intervals are defined by
\begin{align*}
    &UCB_t(a) = \rEst{t-1}{a} + \sqrt{\frac{6\hat{V}_{t-1}(a)\ln t}{\nq{t-1}{a}}} + \frac{7\ln t}{\nq{t-1}{a}-1}\enspace, \\
    &LCB_t(a) = \rEst{t-1}{a} - \sqrt{\frac{6\hat{V}_{t-1}(a)\ln t}{\nq{t-1}{a}}} - \frac{7\ln t}{\nq{t-1}{a}-1},
\end{align*}
where if $\nq{t-1}{a}\le 1$, we define $UCB_t(a)=+\infty$ and $LCB_t(a)=-\infty$ (and in general, we let $\frac{\ln t}{\nq{1}{a}}=\frac{\ln t}{\nq{t}{a}-1}=+\infty$ for $\nq{t}{a}\le1$). We now prove that these confidence intervals are regular:
\begin{lemma}
\label{lemma: bernstein regular}
Assume that  the rewards are bounded in $\brs*{0,1}$ and let $T\ge1$. Then, Bernstein-based confidence interval are regular w.r.t. the events
\begin{align*}
    \G_t=\brc*{\forall a\in\brs*{\Narms}: \rwd{a}\in\brs*{LCB_t(a),UCB_t(a)},\; \abs*{\sqrt{\hat{V}_{t-1}(a)} - \sqrt{V_a}}\le \sqrt{\frac{6\ln t}{\nq{t-1}{a}-1}}}
\end{align*}
and the functions
\begin{align*}
    C(t) = 12\Narms, \qquad
     N_a^g(t,\mu) = \min\brc*{\frac{24V_a\ln t}{\mu^2}+\frac{52\ln t}{\mu}+1,t-1}\enspace.
\end{align*}
For $N^{as}(t,\mu)$, we allow two different options:
\begin{align*}
    &N_1^{as}(t,\mu) = \min\brc*{\frac{6\ln t}{\mu^2}+\frac{28\ln t}{\mu}+1,t-1}, 
    \qquad \textrm{or}\\
    &N_2^{as}(t,\mu) = \min\brc*{\br*{\sqrt{\frac{3\ln t}{2\mu^2}}+\sqrt{\frac{3\ln t}{2\mu^2} + \frac{14\ln t}{\mu}}}^2+1,t-1}
\end{align*}
Specifically, we have that $\bar{N}_a^g(T,\mu)\le N_a^g(T,\mu)$. Finally, we suggest two choices for $\epsilon(t)$:
\begin{itemize}
    \item For nonincreasing sequence $\epsilon(t)\ge0$, we have that $\bar{N}_a^g(T,\mu)= N_a^g(T,\max\brc*{\mu,\epsilon(T)})$ and $\bar{N}_i^{as}(T)= N_i^{as}(T,\epsilon(T))$ for $i=1,2$.
    \item If $B(t)$ is a nondecreasing sequence such that $B(1)>1$ and $\epsilon(t) = \sqrt{\frac{6\ln t}{B(t)-1}}+14\frac{\ln t}{B(t)-1}$, we work with $N_2^{as}(t,\epsilon(t))=\min\brc*{B(t),t-1}$, and thus, $\bar{N}_2^{as}(T)=\min\brc*{B(T),T-1}$. Moreover, for this choice, we have that
    \begin{align*}
        \bar{N}_a^g(T,\mu)\le \min\brc*{N_a^g(T,\mu),4V_a(B(T)-1) + \min\brc*{22\sqrt{(B(T)-1)\ln T}, 4(B(T)-1)} + 1}.
    \end{align*}
\end{itemize}
\end{lemma}
\begin{proof}
First note that for any $a\in\brs*{\Narms}$ and $t\ge1$, the width of the confidence interval is 
\begin{align*}
    CI_t(a) = 2\sqrt{\frac{6\hat{V}_{t-1}(a)\ln t}{\nq{t-1}{a}}} + 14\frac{\ln t}{\nq{t-1}{a}-1}\enspace.    
\end{align*}
We now verify each of the regularity requirements:
\begin{enumerate}
    \item The condition that $CI_t(a)\ge0$ trivially holds for any $t\ge1$ when $\nq{t-1}{a}\ge2$, as $\hat{V}_{t-1}(a)\ge0$ and $\ln T\ge0$, and hold by definition when $\nq{t-1}{a}\le1$, as then, $CI_t(a)=+\infty$.
    
    \item By definition, under $\G_t$, $\rwd{a}\in\brs*{LCB_t(a),UCB_t(a)}$ for any $t\ge1$ and $a\in\brs*{\Narms}$.
    
    \item We now bound $\Pr\brc*{\bar{\G}_t}$. For any $a\in\brs*{\Narms}$, let $X_{a}(1),\dots X_a(T)$ be i.i.d. random variables of the same distribution as arm $a$, and we let the observed reward at the $\nq{t}{a}^{th}$ time the arm was played be $R_t=X_a(\nq{t}{a})$. We also denote the unbiased empirical estimator of the variance based on the first $n$ samples by 
    \begin{align*}
        \bar{V}_n(a) = \frac{n}{n-1}\br*{\frac{1}{n}\sum_{k=1}^n X_a^2(k) - \br*{\frac{1}{n}\sum_{k=1}^n X_a(k)}^2} \enspace, 
    \end{align*}
    which equals $\hat{V}_{t-1}(a)$ when $\nq{t-1}{a}=n$. Then, we can write:
    \begin{align*}
        \Pr\brc*{\bar{\G}_t}
        &= \Pr\Biggl\{\exists a\in\brs*{\Narms}: \rwd{a}\notin\brs*{LCB_t(a),UCB_t(a)},\quad \textrm{or}\Biggr.\\
        &\hspace{2.65cm}\left. \abs*{\sqrt{\hat{V}_{t-1}(a)} - \sqrt{V_a}}>\sqrt{\frac{6\ln t}{\nq{t-1}{a}-1}}\right\}\\
        &= \Pr\left\{ \exists a\in\brs*{\Narms}: \abs*{\rEst{t-1}{a}-\rwd{a}}>\sqrt{\frac{6\hat{V}_{t-1}(a)\ln t}{\nq{t-1}{a}}} + \frac{7\ln t}{\nq{t-1}{a}-1},\quad \textrm{or} \right. \\
        &\hspace{2.73cm} \left.\abs*{\sqrt{\hat{V}_{t-1}(a)} - \sqrt{V_a}}>\sqrt{\frac{6\ln t}{\nq{t-1}{a}-1}}\right\}\\
        & \overset{(1)}{\le} \sum_{a=1}^\Narms \sum_{n=2}^{t-1}\Pr\brc*{\abs*{\rEst{t-1}{a}-\rwd{a}}>\sqrt{\frac{6\hat{V}_{t-1}(a)\ln t}{\nq{t-1}{a}}} + \frac{7\ln t}{\nq{t-1}{a}-1},\nq{t-1}{a}=n} \\
        & \quad+\sum_{a=1}^\Narms \sum_{n=2}^{t-1}\Pr\brc*{\abs*{\sqrt{\hat{V}_{t-1}(a)} - \sqrt{V_a}}>\sqrt{\frac{6\ln t}{\nq{t-1}{a}-1}},\nq{t-1}{a}=n} \\
        & \overset{(2)}{=} \sum_{a=1}^\Narms \sum_{n=2}^{t-1} \Pr\brc*{\abs*{\frac{1}{n}\sum_{k=1}^n X_a(k)-\rwd{a}}> \sqrt{\frac{6\bar{V}_n(a)\ln t}{n}} + \frac{7\ln t}{n-1},\nq{t-1}{a}=n}\\
        & \quad+\sum_{a=1}^\Narms \sum_{n=2}^{t-1}\Pr\brc*{\abs*{\sqrt{\bar{V}_n(a)} - \sqrt{V_a}}>\sqrt{\frac{6\ln t}{n-1}},\nq{t-1}{a}=n} \\
        & \le \sum_{a=1}^\Narms \sum_{n=2}^{t-1} \Pr\brc*{\abs*{\frac{1}{n}\sum_{k=1}^n X_a(k)-\rwd{a}}> \sqrt{\frac{6\bar{V}_n(a)\ln t}{n}} + \frac{7\ln t}{n-1}}\\
        & \quad+\sum_{a=1}^\Narms \sum_{n=2}^{t-1}\Pr\brc*{\abs*{\sqrt{\bar{V}_n(a)} - \sqrt{V_a}}>\sqrt{\frac{6\ln t}{n-1}}} \\
        & \overset{(3)}{\le} \sum_{a=1}^\Narms \sum_{n=1}^{t-1} \frac{4}{t^3}+ \sum_{a=1}^\Narms \sum_{n=1}^{t-1} \frac{2}{t^3}\\
        & \le \frac{6\Narms}{t^2}\enspace.
    \end{align*}
    Relations $(1)$ is due to the union bound over the  two events, actions $a\in\brs*{\Narms}$ and  $\nq{t-1}{a}\in\brc*{0,\dots,t-1}$, respectively (as in the more detailed proof of \Cref{lemma: hoeffding regular}). Specifically, notice that the relevant events cannot occur when $\nq{t-1}{a}\le 1$, so we only treat values larger than $2$. In $(2)$ we substituted $\nq{t-1}{a}=n$ and wrote the empirical means  and variance using $X_{a}(n)$ and $\bar{V}_a(n)$. Finally, $(3)$ relies on Theorems 4 and 10 of \citep{maurer2009empirical} for i.i.d. bounded random variables supported by $[0,1]$ of variance $V_a$ (for all $a\in\brs*{\Narms}$). 
    
    Summing over all time indices leads to the desired result:
    \begin{align*}
        \sum_{t=1}^T \Pr\brc*{\bar{\G}_t}
        \le \sum_{t=1}^T \frac{6\Narms}{t^2}
        \le \sum_{t=1}^\infty \frac{6\Narms}{t^2}
        \le 12\Narms
        =C(T)\enspace.
    \end{align*}
    
    \item  If $N_a^g(t,\mu)=t-1$, then the event that $N_a^g(t,\mu)<\nq{t-1}{a}\le t-1$ can never happen, so assume w.l.o.g. that $N_a^g(t,\mu)<t-1$ and, in particular, $t>1$, $\mu>0$ and  $N_a^g(t,\mu)=\frac{24V_a\ln t}{\mu^2}+\frac{52\ln t}{\mu}+1$. Notice that in this case,  $N_a^\mu(t)>1$, and for $\nq{t-1}{a}>N_a^\mu(t)$, all bounds are finite and all denominators are positive. Also note that under $\G_t$, we can bound $\sqrt{\hat{V}_{t-1}(a)} \le \sqrt{V_a} + \sqrt{\frac{6\ln t}{\nq{t-1}{a}-1}}$. Then, substituting to $CI_t(a)$, for any $t>1$ and $a\in\brs*{\Narms}$,
    \begin{align*}
        CI_t(a) & = 2\sqrt{\frac{6\hat{V}_{t-1}(a)\ln t}{\nq{t-1}{a}}} + 14\frac{\ln t}{\nq{t-1}{a}-1} \\
        & \le 2\sqrt{\frac{6V_a\ln t}{\nq{t-1}{a}}} + 12\frac{\ln t}{\sqrt{\nq{t-1}{a}(\nq{t-1}{a}-1)}} + 14\frac{\ln t}{\nq{t-1}{a}-1}  \\
        & \le 2\sqrt{\frac{6V_a\ln t}{\nq{t-1}{a}-1}} + 12\frac{\ln t}{\nq{t-1}{a}-1} + 14\frac{\ln t}{\nq{t-1}{a}-1} \\
        & = \frac{\sqrt{24 V_a\ln t}}{\sqrt{\nq{t-1}{a}-1}} + \frac{26\ln t}{\nq{t-1}{a}-1}\enspace.
    \end{align*}
    Finally, by \Cref{claim: bernstein to confidence} (presented below), see that if $\nq{t-1}{a}>\frac{24V_a\ln t}{\mu^2}+\frac{52\ln t}{\mu}+1=N_a^g(t,\mu)$, then $CI_t(a)<\mu$, as required.
    
    \item As in the proof of $N_a^g(t,\mu)$, we focus on the case when $N_i^{as}(t,\mu)<t-1$, which implies that $t>1$ and $\mu>0$ (for both choices of $i\in\brc*{1,2}$). Thus, for $\nq{t-1}{a}>N_i^{as}(t,\mu)>1$, all bounds are finite and all denominators are positive. Also, for rewards bounded in $[0,1]$, we can bound $\hat{V}_t(a)$ by
    \begin{align*}
        \hat{V}_t(a) 
        &= \frac{\nq{t}{a}}{\nq{t}{a}-1}\br*{\frac{1}{\nq{t}{a}}\sum_{s=1}^t R_t^2\Ind{a_t=a} - \br*{\rEst{t}{a}}^2} \\
        &\le \frac{\nq{t}{a}}{\nq{t}{a}-1}\br*{\frac{1}{\nq{t}{a}}\sum_{s=1}^t R_t\Ind{a_t=a} - \br*{\rEst{t}{a}}^2} \\
        & = \frac{\nq{t}{a}}{\nq{t}{a}-1}\br*{\rEst{t}{a} - \br*{\rEst{t}{a}}^2} \\
        & \le \frac{1}{4}\frac{\nq{t}{a}}{\nq{t}{a}-1}\enspace,
    \end{align*}
    where the last inequality is since the function $x-x^2\le\frac{1}{4}$ for $x\in[0,1]$. In turn, for all $\nq{t-1}{a}>1$,
    \begin{align*}
        CI_t(a) & = 2\sqrt{\frac{6\hat{V}_{t-1}(a)\ln t}{\nq{t-1}{a}}} + 14\frac{\ln t}{\nq{t-1}{a}-1} \\
        & \le \sqrt{\frac{6\ln t}{\nq{t-1}{a}-1}} + 14\frac{\ln t}{\nq{t-1}{a}-1}\enspace.
    \end{align*}
    By elementary algebra, if $\nq{t-1}{a} = \br*{\sqrt{\frac{3\ln t}{2\mu^2}}+\sqrt{\frac{3\ln t}{2\mu^2} + \frac{14\ln t}{\mu}}}^2+1=N_2^{as}(t,\mu)$, then the above bound exactly equals to $\mu$. Moreover, the bound strictly decreases in $\nq{t-1}{a}$; therefore, for any $\nq{t-1}{a}>N_2^{as}(t,\mu)$, we get that $CI_t(a)<\mu$, as required of $N_2^{as}$. Alternatively, by \Cref{claim: bernstein to confidence}, if $\nq{t-1}{a} > \frac{6\ln t}{\mu^2}+\frac{28\ln t}{\mu}+1= N_1^{as}(t,\mu)$, then $CI_t(a)<\mu$, as desired from $N_1^{as}$.
    \end{enumerate}
    We now prove the additional results stated at the end of the lemma. First, notice that 
    \begin{align}
        N_a^g(t,\max\brc*{\mu,\epsilon(t)}) 
        &= \min\brc*{\frac{24V_a\ln t}{\max\brc*{\mu^2,\epsilon^2(t)}}+\frac{52\ln t}{\max\brc*{\mu,\epsilon(t)}}+1,t-1} \nonumber\\
        & = \min\brc*{\frac{24V_a\ln t}{\mu^2}+\frac{52\ln t}{\mu}+1,\frac{24V_a\ln t}{\epsilon^2(t)}+\frac{52\ln t}{\epsilon(t)}+1,t-1} \label{eq: bernstein N minimum}\\
        & \le \min\brc*{\frac{24V_a\ln t}{\mu^2}+\frac{52\ln t}{\mu}+1, t-1} \nonumber\\
        & = N_a^g(t,\mu)\enspace.\nonumber
    \end{align}
    Moreover, $N_a^g(t,\mu)$ is nondecreasing in $t$, and thus, 
    \begin{align*}
        \bar{N}_a^g(T,\mu) 
        = \max_{t\in\brs*{T}} N_a^g(t,\max\brc*{\mu,\epsilon(t)}) 
        \le \max_{t\in\brs*{T}} N_a^g(t,\mu) \le N_a^g(T,\mu)\enspace.
    \end{align*}
    Next, if $\epsilon(t)$ is nonincreasing, then \Cref{eq: bernstein N minimum} expresses $N_a^g(t,\max\brc*{\mu,\epsilon(t)})$ as a minimum on nondecreasing functions, which is by itself nondecreasing (see \Cref{claim:minimum increasing}). Thus, in this case, it holds that \begin{align*}
        \bar{N}_a^g(T,\mu) 
        = \max_{t\in\brs*{T}} N_a^g(t,\max\brc*{\mu,\epsilon(t)}) 
        = N_a^g(T,\max\brc*{\mu,\epsilon(T)}) \enspace.
    \end{align*}
    One can similarly verify that when $\epsilon(t)$ is nonincreasing, $N_i^{as}(t,\epsilon(t))$ are minima of nondecreasing functions, and following similar lines implies that $\bar{N}_i^{as}(T)\le N_i^{as}(T,\epsilon(T))$ for $i=1,2$.
    
    Finally, assume that $B(t)>1$ is nondecreasing and let $\epsilon(t) = \sqrt{\frac{6\ln t}{B(t)-1}}+14\frac{\ln t}{B(t)-1}$. Then, a direct calculation results with $N_2^{as}(t,\epsilon(t))=\min\brc*{B(t),t-1}$, which is also nondecreasing, a property that once again implies that $\bar{N}_2^{as}(T) = N_2^{as}(T,\epsilon(T))$. Moreover, returning to \eqref{eq: bernstein N minimum}, we can write
    \begin{align*}
        N_a^g(t,\max\brc*{\mu,\epsilon(t)}) 
        &= \min\brc*{\frac{24V_a\ln t}{\mu^2}+\frac{52\ln t}{\mu}+1,\frac{24V_a\ln t}{\epsilon^2(t)}+\frac{52\ln t}{\epsilon(t)}+1,t-1}\\
        & =\min\brc*{N_a^g(t,\mu),\frac{24V_a\ln t}{\epsilon^2(t)}+\frac{52\ln t}{\epsilon(t)}+1}\enspace.\nonumber
    \end{align*}
    We now substitute $\epsilon(t)$ and bound the second term by
    \begin{align*}
        \frac{24V_a\ln t}{\epsilon^2(t)}+\frac{52\ln t}{\epsilon(t)}  
        &\le \frac{24V_a\ln t}{\br*{\sqrt{6\ln t/(B(t)-1)}}^2}+\frac{52\ln t}{\max\brc*{\sqrt{6\ln t/(B(t)-1)},14\ln t/(B(t)-1}} \\
        & \le 
        4V_a(B(t)-1) + \min\brc*{22\sqrt{(B(t)-1)\ln t}, 4(B(t)-1)}\enspace.
    \end{align*}
    Importantly, as $B(t)$ is nondecreasing, this bound is nondecreasing. In turn, substituting back to $N_a^g(t,\max\brc*{\mu,\epsilon(t)})$ results with a nondecreasing upper bound, whose maximum is achieved at $t=T$. This leads to the desired bound on $\bar{N}_a^g(T,\mu)$
\end{proof}


\subsubsection{Auxiliary Claims}
We now present two extremely simple results that we repeatedly used and are proven for completeness:
\begin{claim}
\label{claim:minimum increasing}
Let $f_1,\dots,f_n:\R\mapsto\R$ be nondecreasing functions. Then, $f(x)=\min_{j\in\brs*{n}}f_j(x)$ is also nondecreasing in $x$.
\end{claim}
\begin{proof}
Let $x_1,x_2\in\R$ such that $x_1>x_2$ and assume w.l.o.g. that $f(x_1) = f_k(x_1)$ for some $k\in\brs*{n}$. Then,
\begin{align*}
    f(x_1) = f_k(x_1) \ge f_k(x_2) \ge  \min_{j\in\brs*{n}}f_j(x_2) = f(x_2)\enspace.
\end{align*}
\end{proof}
\begin{claim}
\label{claim: bernstein to confidence}
If $c_1,c_2\ge0$ and $\mu>0$, then for any $n>n_0\triangleq\frac{c_1^2}{\mu^2} + \frac{2c_2}{\mu}$, it holds that
\begin{align*}
    \frac{c_1}{\sqrt{n}}+\frac{c_2}{n}<\mu\enspace.
\end{align*}
\end{claim}
\begin{proof}
Notice that the l.h.s. is strictly increasing in $n$. Therefore, it is sufficient to prove that for $n=n_0$, the l.h.s. is (weakly) smaller than $\mu$:
\begin{align*}
    \frac{c_1}{\sqrt{n_0}}+\frac{c_2}{n_0}
    &= \frac{c_1\sqrt{n_0}+c_2}{n_0}
    = \frac{c_1\sqrt{\frac{c_1^2}{\mu^2} + \frac{2c_2}{\mu}}+c_2}{\frac{c_1^2}{\mu^2} + \frac{2c_2}{\mu}}
    = \mu \cdot\frac{c_1\sqrt{c_1^2 + 2\mu c_2}+\mu c_2}{c_1^2 + 2\mu c_2} \\
    &\overset{(*)}{\le}  \mu\cdot \frac{c_1^2 + \mu c_2+\mu c_2}{c_1^2 + 2\mu c_2}
    = \mu
\end{align*}
where $(*)$ is by the inequality $\sqrt{a(a+b)} \le a + \frac{b}{2}$.
\end{proof}
\clearpage

\section{Experimental details}
\subsection{Baseline Algorithms}
\label{appendix:baselines}
\begin{algorithm}[H]
\caption{\underline{Bu}dget-\underline{F}eedback \underline{A}ware \underline{U}pper  Confidence Bound (BuFAU)}  \label{alg: BuFAU}
\begin{algorithmic}[1]
\State \textbf{Define: } $UCB_t(a) = \rEst{t-1}{a} + \sqrt{\frac{3\ln t}{2\nq{t-1}{a}}}$ and $LCB_t(a) = \rEst{t-1}{a} - \sqrt{\frac{3\ln t}{2\nq{t-1}{a}}}$
\For{$t=1,...,\Narms$}
\State Play $a_t = t$ and ask for feedback ($q_t=1$); observe $\rOb{t}$ and update $\nq{t}{a_t}, \rEst{t}{a_t}$
\EndFor
\For{$t=\Narms+1,...,T$}
\State Observe $\epsilon(t)\ge0$ 
\State Set $l_t\in\argmax_a LCB_t(a)$ and $u_t\in\argmax_a UCB_t(a)$
\If{$\max_{a\ne l_t} UCB_t(a)\le LCB_t(l_t)$ or $UCB_t(u_t)-LCB_t(l_t)\le \epsilon(t)$} 
    \State Play $a_t=l_t$ and do not ask for feedback ($q_t=0$) 
\Else
    \State Play $a_t=u_t$ and ask for feedback ($q_t=1$); observe $\rOb{t}$ and update $\nq{t}{a_t}, \rEst{t}{a_t}$
\EndIf
\EndFor
\end{algorithmic}
\end{algorithm}
\vspace{-.5cm}

\begin{algorithm}[H]
\caption{Confidence-Budget Matching Upper Confidence Bound (CBM-UCB)}  \label{alg: CBM}
\begin{algorithmic}[1]
\State \textbf{Define: } $UCB_t(a) = \rEst{t-1}{a} + \sqrt{\frac{3\ln t}{2\nq{t-1}{a}}}$ and $LCB_t(a) = \rEst{t-1}{a} - \sqrt{\frac{3\ln t}{2\nq{t-1}{a}}}$
\For{$t=1,...,\Narms$}
\State Play $a_t = t$ and ask for feedback ($q_t=1$); observe $\rOb{t}$ and update $\nq{t}{a_t}, \rEst{t}{a_t}$
\EndFor
\For{$t=\Narms+1,...,T$}
\State Observe $\epsilon(t)\ge0$ 
\State Play $a_t=u_t\in\argmax_a UCB_t(a)$
\If{$CI_t(u_t)\le \epsilon(t)$} 
    \State Do not ask for feedback ($q_t=0$) 
\Else
    \State Ask for feedback ($q_t=1$); observe $\rOb{t}$ and update $\nq{t}{a_t}, \rEst{t}{a_t}$
\EndIf
\EndFor
\end{algorithmic}
\end{algorithm}
\vspace{-.5cm}

\begin{algorithm}[H]
\caption{Greedy Algorithm}  \label{alg: greedy}
\begin{algorithmic}[1]
\State \textbf{Define: } $UCB_t(a) = \rEst{t-1}{a} + \sqrt{\frac{3\ln t}{2\nq{t-1}{a}}}$, $LCB_t(a) = \rEst{t-1}{a} - \sqrt{\frac{3\ln t}{2\nq{t-1}{a}}}$ and $\Bq{\Narms}=\Narms$
\For{$t=1,...,\Narms$}
\State Play $a_t = t$ and ask for feedback ($q_t=1$) 
\State Observe $\rOb{t}$ and update $\nq{t}{a_t}, \rEst{t}{a_t}$
\EndFor
\For{$t=\Narms+1,...,T$}
\State Observe $\epsilon(t)\ge0$ and calculate effective budget $B(t) = \frac{6\Narms\ln t}{\epsilon^2(t)}+\Narms$
\If{$\Bq{t-1} > B(t) - 1$} 
    \State Play $a_t\in\argmax_a \rEst{t-1}{a}$, do not ask for feedback ($q_t=0$) and set $\Bq{t}=\Bq{t-1}$
\Else
    \State Play $a_t\in\argmax_a UCB_t(a)$, ask for feedback ($q_t=1$) and set $\Bq{t}=\Bq{t-1}+1$
    \State Observe $\rOb{t}$ and update $\nq{t}{a_t}, \rEst{t}{a_t}$
\EndIf
\EndFor
\end{algorithmic}
\end{algorithm}
\clearpage

\subsection{Additional Experiments}
\label{appendix:experiments}
\subsubsection{Experiments in 5-Armed Random Problems}
In this appendix, we present simulation results that demonstrate that the behavior presented in \Cref{subsection:illustrations} still holds when arms are random. Specifically, we experiment on 5-armed problems with either one or two optimal arms, where the optimal mean is $\rOpt=0.5$ and the suboptimal arms are of mean $\rwd{a}=0.25$. All arms are Bernoulli-distributed. As in the main paper, we evaluate the algorithms for $100,000$ time steps using $\epsilon(t)=t^{-1/4}$. The simulation results in \Cref{figure:experiments 5arms} are averaged over $1,000$ seeds, and the statistics at the last time-step are presented in \Cref{table:statistics} (including mean, standard deviation, $90^{th}$ quantile and maximum). The conclusions practically remain the same as in the main paper -- all algorithms perform roughly the same when there are multiple optimal arms (with slight performance degradation and increased querying for the greedy algorithm). On the other hand, when there is a unique optimal arm, BuFALU requires much less feedback but suffers from a regret degradation by a factor of $3$ (similar to the $4$ factor of \Cref{theorem:upper bound}).

\begin{figure}[H]
\centering
\subfigure{
\includegraphics[trim=0 0 0 0,clip,width=0.39\linewidth]{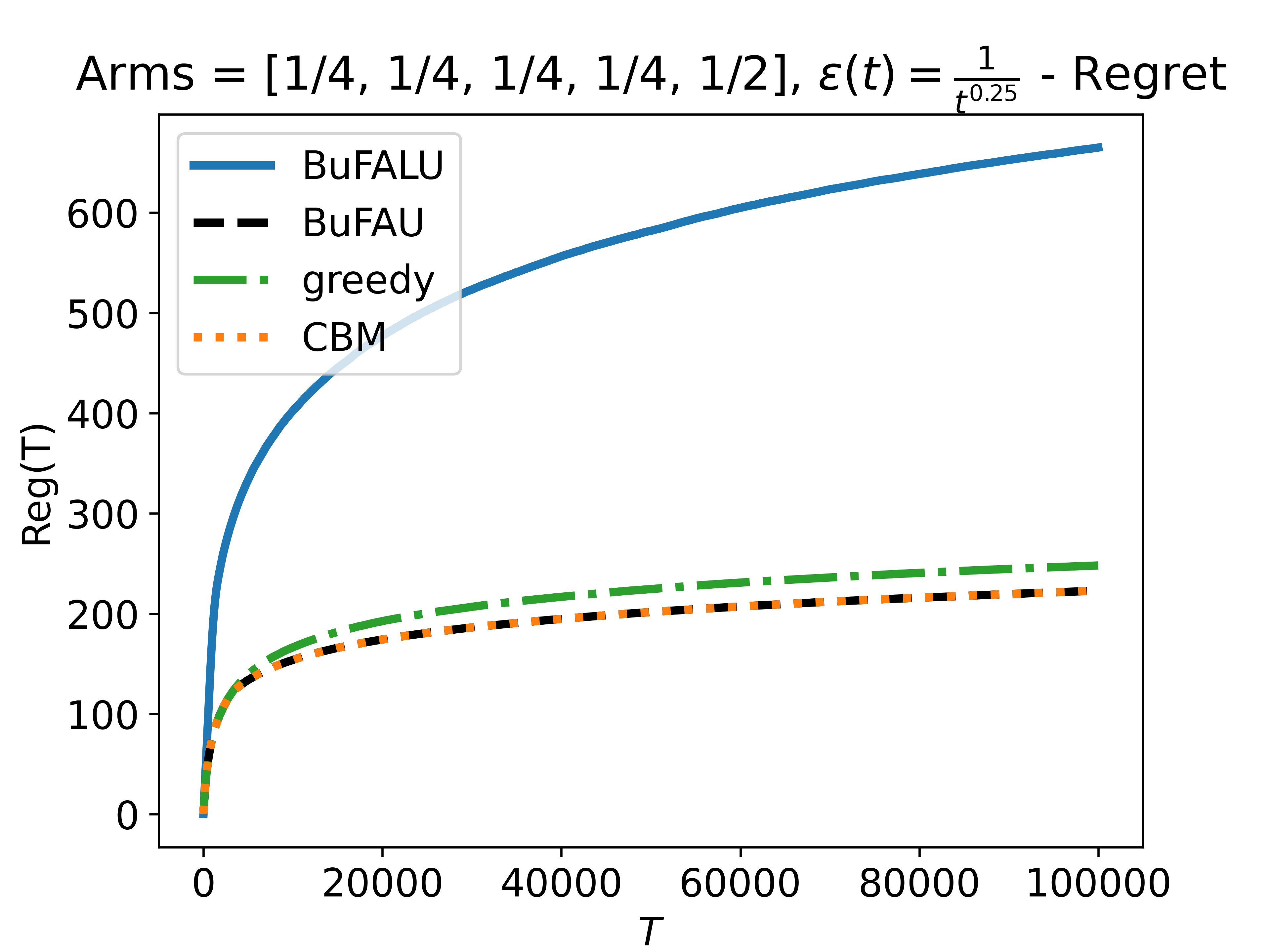}
}
\hspace{0.05\linewidth}
\subfigure{
\includegraphics[trim=0 0 0 0,clip,width=0.39\linewidth]{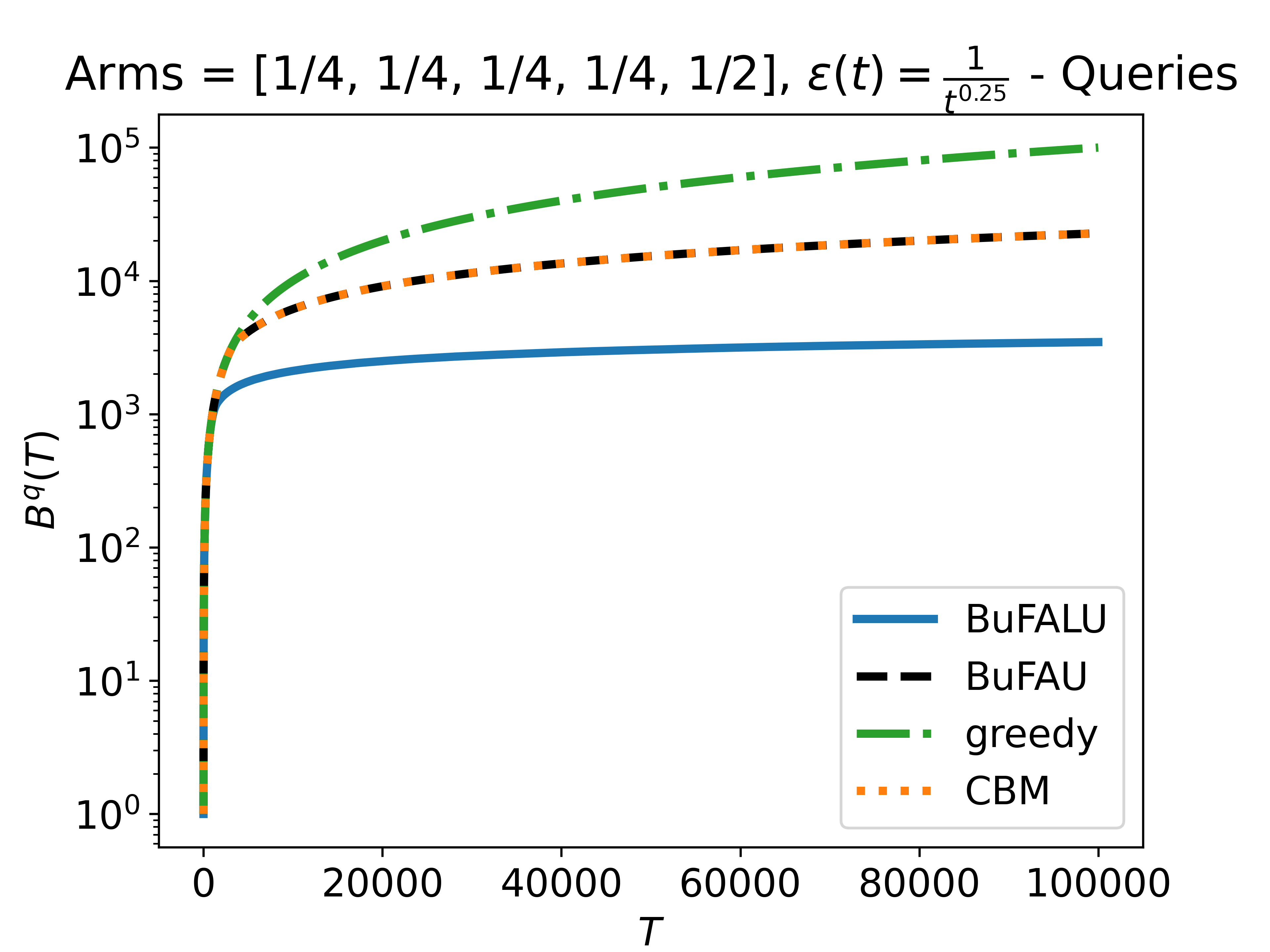}
} 
\subfigure{
\includegraphics[trim=0 0 0 0,clip,width=0.39\linewidth]{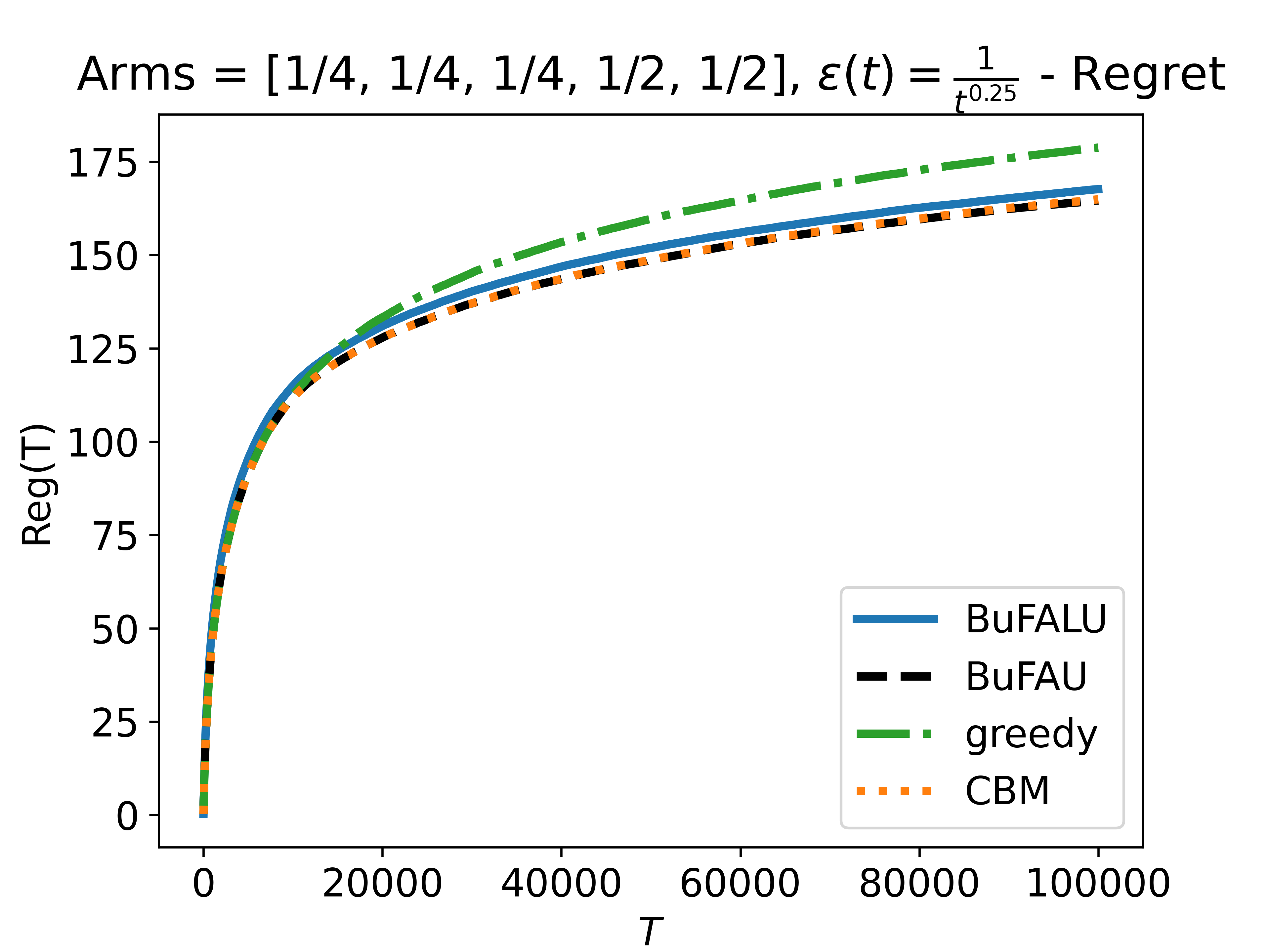}
}
\hspace{0.05\linewidth}
\subfigure{
\includegraphics[trim=0 0 0 0,clip,width=0.39\linewidth]{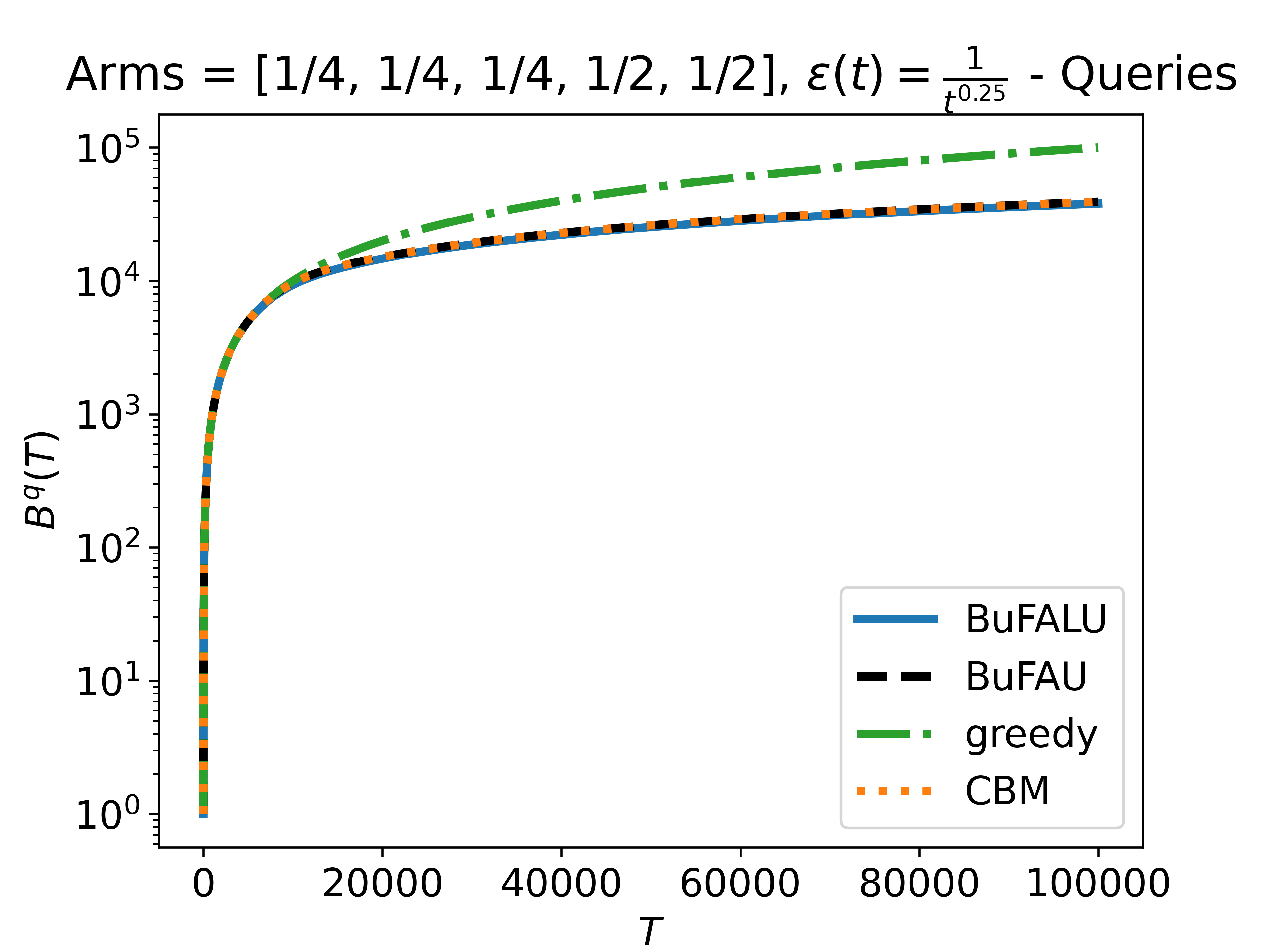}
}
\caption{Evaluation of all algorithms on 5-armed problems over $1,000$ seeds with $\epsilon(t)=t^{-1/4}$.}
\label{figure:experiments 5arms}
\end{figure}

\setlength{\tabcolsep}{3pt}
\begin{table}[ht]
\centering
\caption{Additional statistics of the empirical evaluation in \Cref{figure:experiments 5arms}. All statistics are measured after $T=100,000$ time steps over $1,000$ different seeds. All experiments use $\epsilon(t)=t^{-1/4}$. std is the standard deviation, $90\%$ represents the $90^{th}$ percentile and max is the maximum.}
\label{table:statistics}
\begin{tabular}{|c|c?c|c|c|c?c|c|c|c|}\hline
  &  & \multicolumn{4}{c?}{\textbf{Regret}} &  \multicolumn{4}{c|}{\textbf{Queries}} \\ \hhline {~~--------}
   \multirow{-2}{2.5cm}{\centering\textbf{Bandit Instance}}&  \multirow{-2}{*}{\centering\textbf{Algorithm}} & mean & std & $90\%$ & max & 
    mean & std & $90\%$ & max \\ \hline 
    
  \rowcolor[gray]{0.9} \cellcolor{white}&
  BuFALU & 664.88 & 87.08 & 774.8 & 982.25 & 3471.38 & 437.39 & 4024.1 & 5023 \\ 
  
  \hhline{~>{\arrayrulecolor[gray]{.5}}|*9{-}|}
     & 
   BuFAU & 222.94 & 23.03 & 253.25 & 304.5 & 22736.7 & 92.1 & 22858 & 23063 \\ 
   
   \hhline{~>{\arrayrulecolor[gray]{.5}}|*9{-}|}
   \rowcolor[gray]{0.9} \cellcolor{white} & 
   greedy & 248.08 & 26.18 & 281.77 & 343.75 & 100000 & 0 & 100000 & 100000 \\ 
   
   \hhline{~>{\arrayrulecolor[gray]{.5}}|*9{-}|}
     \multirow{-4}{2.5cm}{\centering 5-arms, unique optimal, values: $\brs*{\frac{1}{4},\frac{1}{4},\frac{1}{4},\frac{1}{4},\frac{1}{2}}$}& 
   CBM & 222.97 & 23.06 & 253.25 & 304.5 & 22736.9 & 92.25 & 22858 & 23063 \\ 
  \arrayrulecolor{black}\hline
  \rowcolor[gray]{0.9} \cellcolor{white}&
   BuFALU & 167.64 & 20.17 & 194.52 & 230.75 & 38049.3 & 3944.21 & 42593.5 & 44135 \\ 
  
  \hhline{~>{\arrayrulecolor[gray]{.5}}|*9{-}|}
     & 
   BuFAU & 164.62 & 19.75 & 190.02 & 230.25 & 39254.9 & 3292.47 & 43298.1 & 44383 \\ 
   
   \hhline{~>{\arrayrulecolor[gray]{.5}}|*9{-}|}
   \rowcolor[gray]{0.9} \cellcolor{white} & 
   greedy & 178.79 & 21.72 & 207.78 & 256.75 & 100000 & 0 & 100000 & 100000\\ 
   
   \hhline{~>{\arrayrulecolor[gray]{.5}}|*9{-}|}
     \multirow{-4}{2.5cm}{\centering 5-arms, multiple optimal, values: $\brs*{\frac{1}{4},\frac{1}{4},\frac{1}{4},\frac{1}{2},\frac{1}{2}}$}& 
   CBM & 164.92 & 19.95 & 190.5 & 245 & 39217.6 & 3277.71 & 43289.4 & 44383 \\ 
  \arrayrulecolor{black}\hline
\end{tabular}
\end{table}
\setlength{\tabcolsep}{6pt}

For completeness, we also present simulation results for two additional profiles of $\epsilon(t)$ -- the first is when $\epsilon(t)=0$ (see \Cref{figure:experiments 5arms eps-zero} and statistics in \Cref{table:statistics eps-zero}) and the second when $\epsilon(t) = \frac{1}{\ln t}$ (see \Cref{figure:experiments 5arms eps poly-log} and statistics in \Cref{table:statistics eps poly-log}). When $\epsilon(t)=0$, algorithms are allowed to ask as much feedback as they want. Then, when there are multiple optimal arms, all algorithms always ask for feedback and perform the same. On the other hand, when the optimal arm is unique, BuFALU asks for feedback on $\sim\frac{1}{20}$ of the rounds and suffers a factor of $4$ in its regret. Therefore, in the unlimited budget scenario, we return to the same behavior as in \Cref{subsection:illustrations}. The case of $\epsilon(t) = \frac{1}{\ln t}$, which corresponds to a querying budget of $\Ocal(\Narms \ln^3 t)$, behaves very similar to the choice of $\epsilon(t) = t^{-1/4}$. Interestingly, when the optimal arm is unique, the regret of BuFALU in both limited cases is lower than the regret when $\epsilon(t)=0$. A possible explanation is that when the budget is unlimited, BuFALU aggressively explores to separate the optimal arm from all other arms. When the budget is limited, BuFALU combines exploration with exploitation of $l_t$, which leads to a slightly lower regret.

\begin{figure}[htbp]
\centering
\subfigure{
\includegraphics[trim=0 0 0 0,clip,width=0.38\linewidth]{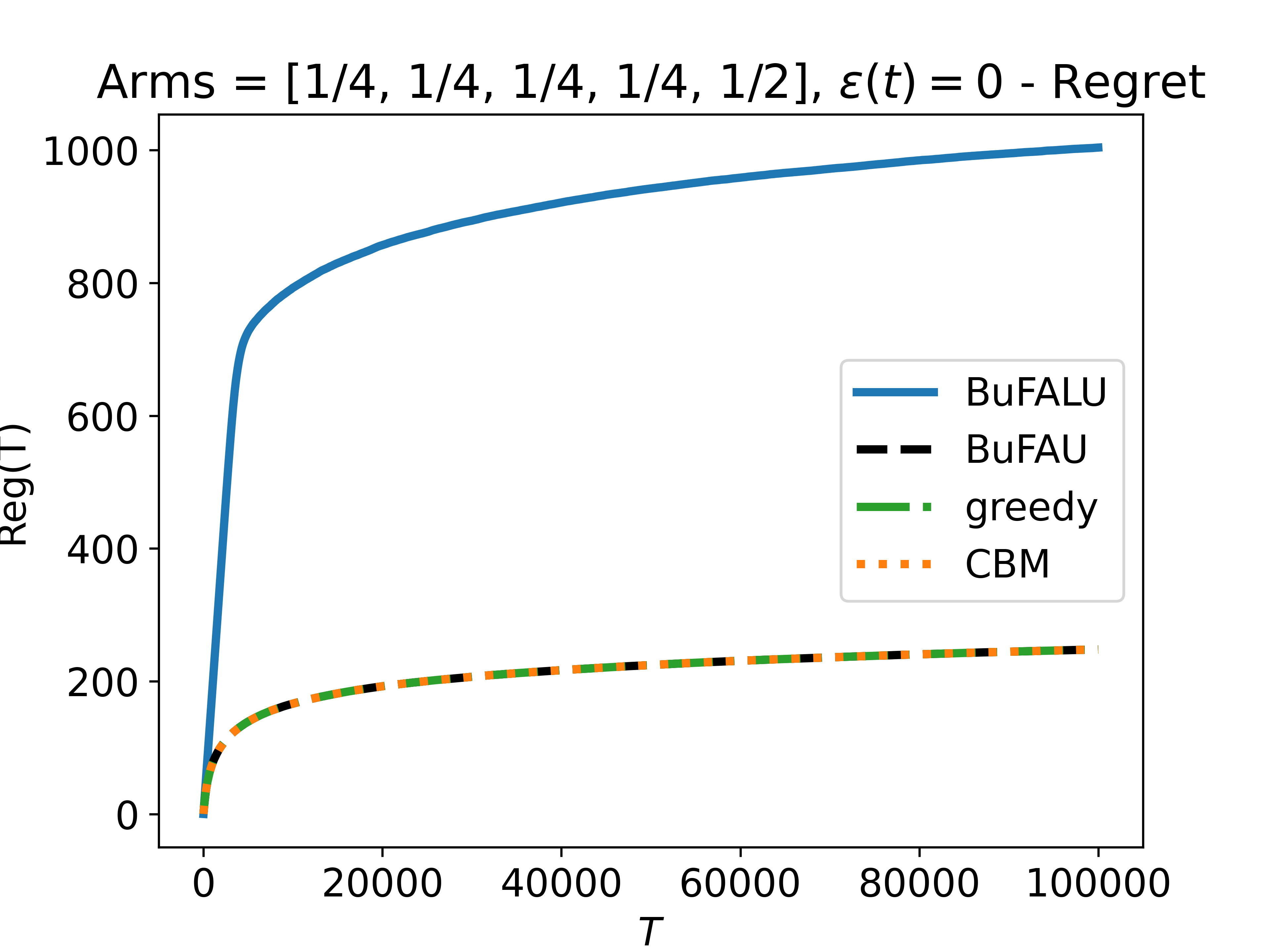}
}
\hspace{0.05\linewidth}
\subfigure{
\includegraphics[trim=0 0 0 0,clip,width=0.38\linewidth]{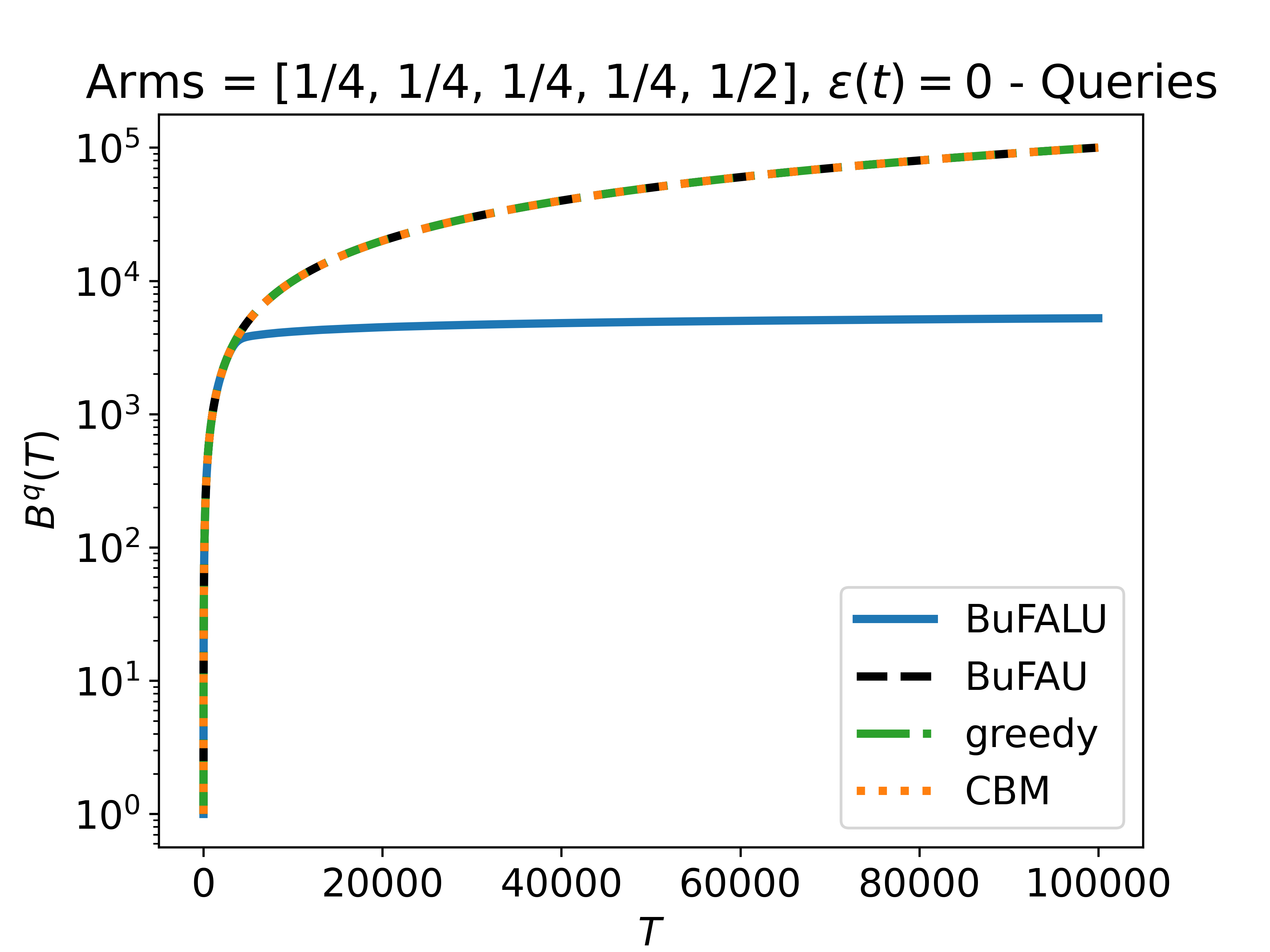}
} 
\subfigure{
\includegraphics[trim=0 0 0 0,clip,width=0.38\linewidth]{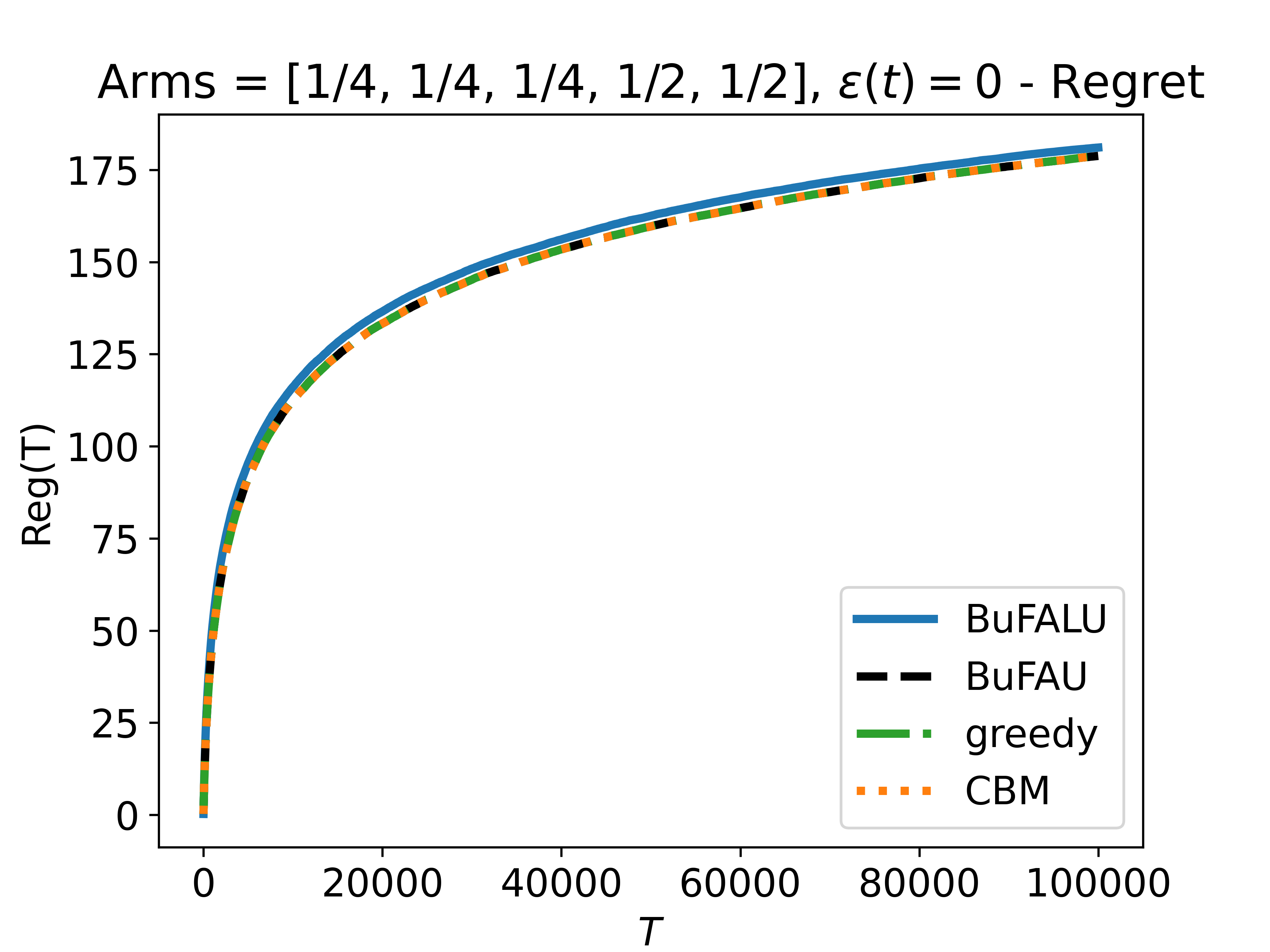}
}
\hspace{0.05\linewidth}
\subfigure{
\includegraphics[trim=0 0 0 0,clip,width=0.38\linewidth]{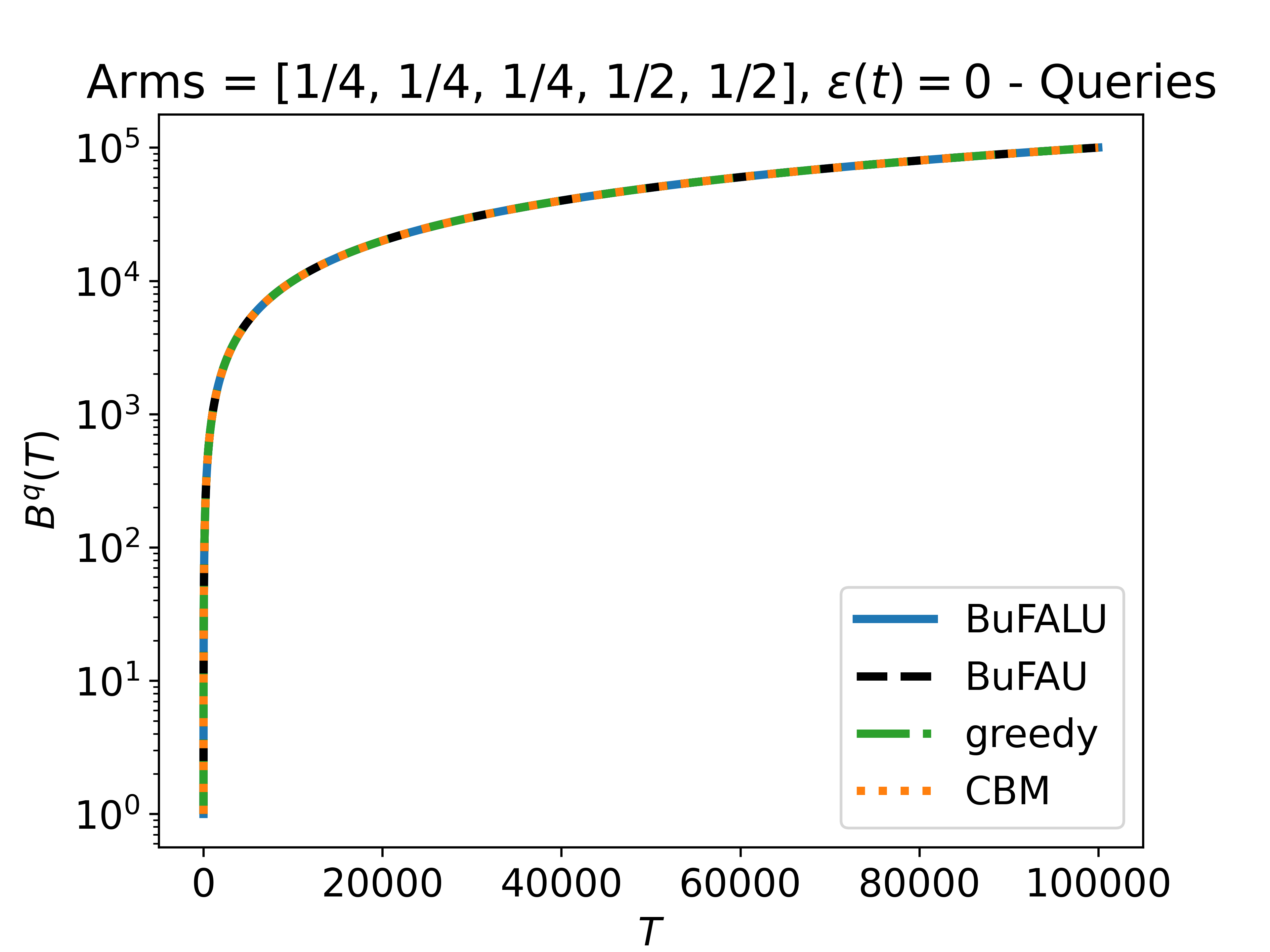}
}
\caption{Evaluation of all algorithms on 5-armed problems over $1,000$ seeds with $\epsilon(t)=0$.}
\label{figure:experiments 5arms eps-zero}
\end{figure}

\setlength{\tabcolsep}{3pt}
\begin{table}[htbp]
\centering
\caption{Additional statistics of the empirical evaluation in \Cref{figure:experiments 5arms eps-zero}. All statistics are measured after $T=100,000$ time steps over $1,000$ different seeds. All experiments use $\epsilon(t)=0$. std is the standard deviation, $90\%$ represents the $90^{th}$ percentile and max is the maximum.}
\label{table:statistics eps-zero}
\begin{tabular}{|c|c?c|c|c|c?c|c|c|c|}\hline
  &  & \multicolumn{4}{c?}{\textbf{Regret}} &  \multicolumn{4}{c|}{\textbf{Queries}} \\ \hhline {~~--------}
   \multirow{-2}{2.5cm}{\centering\textbf{Bandit Instance}}&  \multirow{-2}{*}{\centering\textbf{Algorithm}} & mean & std & $90\%$ & max & 
    mean & std & $90\%$ & max \\ \hline 
    
  \rowcolor[gray]{0.9} \cellcolor{white}&
  BuFALU & 1003.96 & 131.23 & 1180.55 & 1457.25 & 5240.73 & 658.93 & 6130 & 7538 \\ 
  
  \hhline{~>{\arrayrulecolor[gray]{.5}}|*9{-}|}
     & 
   BuFAU & 248.08 & 26.18 & 281.77 & 343.75 & 100000 & 0 & 100000 & 100000 \\ 
   
   \hhline{~>{\arrayrulecolor[gray]{.5}}|*9{-}|}
   \rowcolor[gray]{0.9} \cellcolor{white} & 
   greedy & 248.08 & 26.18 & 281.77 & 343.75 & 100000 & 0 & 100000 & 100000 \\ 
   
   \hhline{~>{\arrayrulecolor[gray]{.5}}|*9{-}|}
     \multirow{-4}{2.5cm}{\centering 5-arms, unique optimal, values: $\brs*{\frac{1}{4},\frac{1}{4},\frac{1}{4},\frac{1}{4},\frac{1}{2}}$}& 
   CBM & 248.08 & 26.18 & 281.77 & 343.75 & 100000 & 0 & 100000 & 100000 \\ 
  \arrayrulecolor{black}\hline
  \rowcolor[gray]{0.9} \cellcolor{white}&
   BuFALU & 181.05 & 21.6 & 208.02 & 259.5 & 100000 & 0 & 100000 & 100000 \\ 
  
  \hhline{~>{\arrayrulecolor[gray]{.5}}|*9{-}|}
     & 
   BuFAU & 178.79 & 21.72 & 207.78 & 256.75 & 100000 & 0 & 100000 & 100000 \\ 
   
   \hhline{~>{\arrayrulecolor[gray]{.5}}|*9{-}|}
   \rowcolor[gray]{0.9} \cellcolor{white} & 
   greedy & 178.79 & 21.72 & 207.78 & 256.75 & 100000 & 0 & 100000 & 100000\\ 
   
   \hhline{~>{\arrayrulecolor[gray]{.5}}|*9{-}|}
     \multirow{-4}{2.5cm}{\centering 5-arms, multiple optimal, values: $\brs*{\frac{1}{4},\frac{1}{4},\frac{1}{4},\frac{1}{2},\frac{1}{2}}$}& 
   CBM & 178.79 & 21.72 & 207.78 & 256.75 & 100000 & 0 & 100000 & 100000 \\ 
  \arrayrulecolor{black}\hline
\end{tabular}
\end{table}
\setlength{\tabcolsep}{6pt}

\begin{figure}[htbp]
\centering
\subfigure{
\includegraphics[trim=0 0 0 0,clip,width=0.38\linewidth]{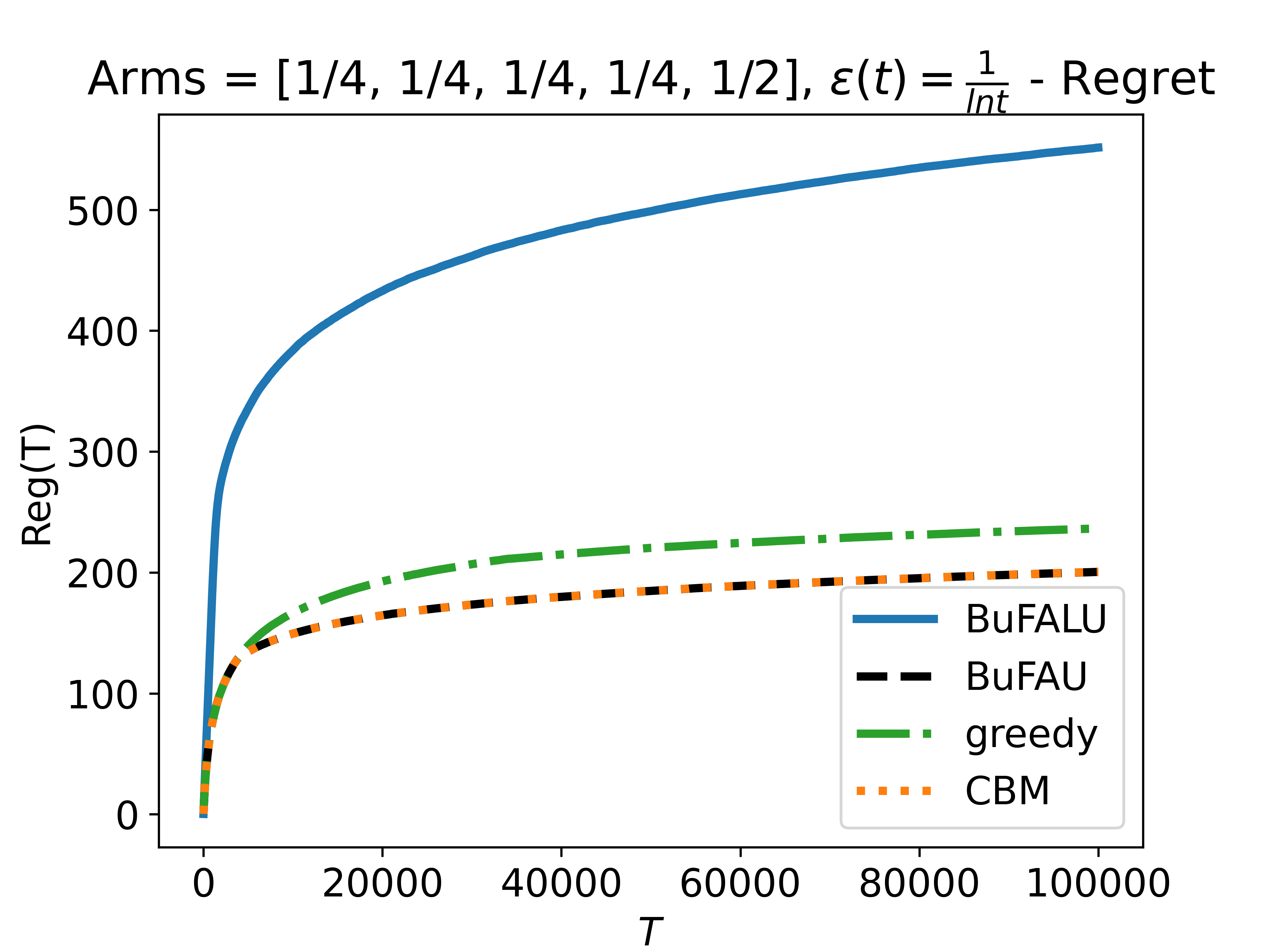}
}
\hspace{0.05\linewidth}
\subfigure{
\includegraphics[trim=0 0 0 0,clip,width=0.38\linewidth]{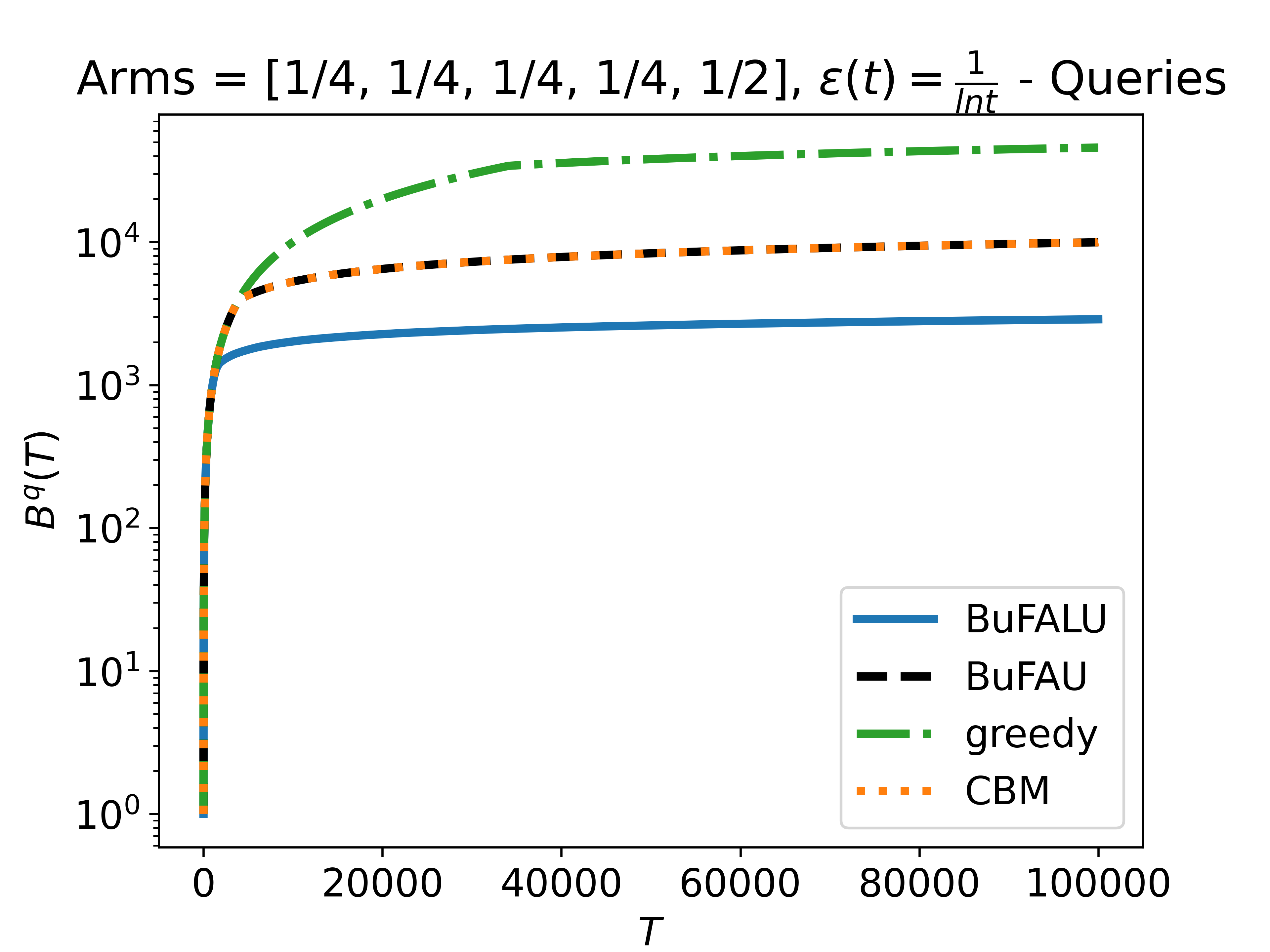}
} 
\subfigure{
\includegraphics[trim=0 0 0 0,clip,width=0.38\linewidth]{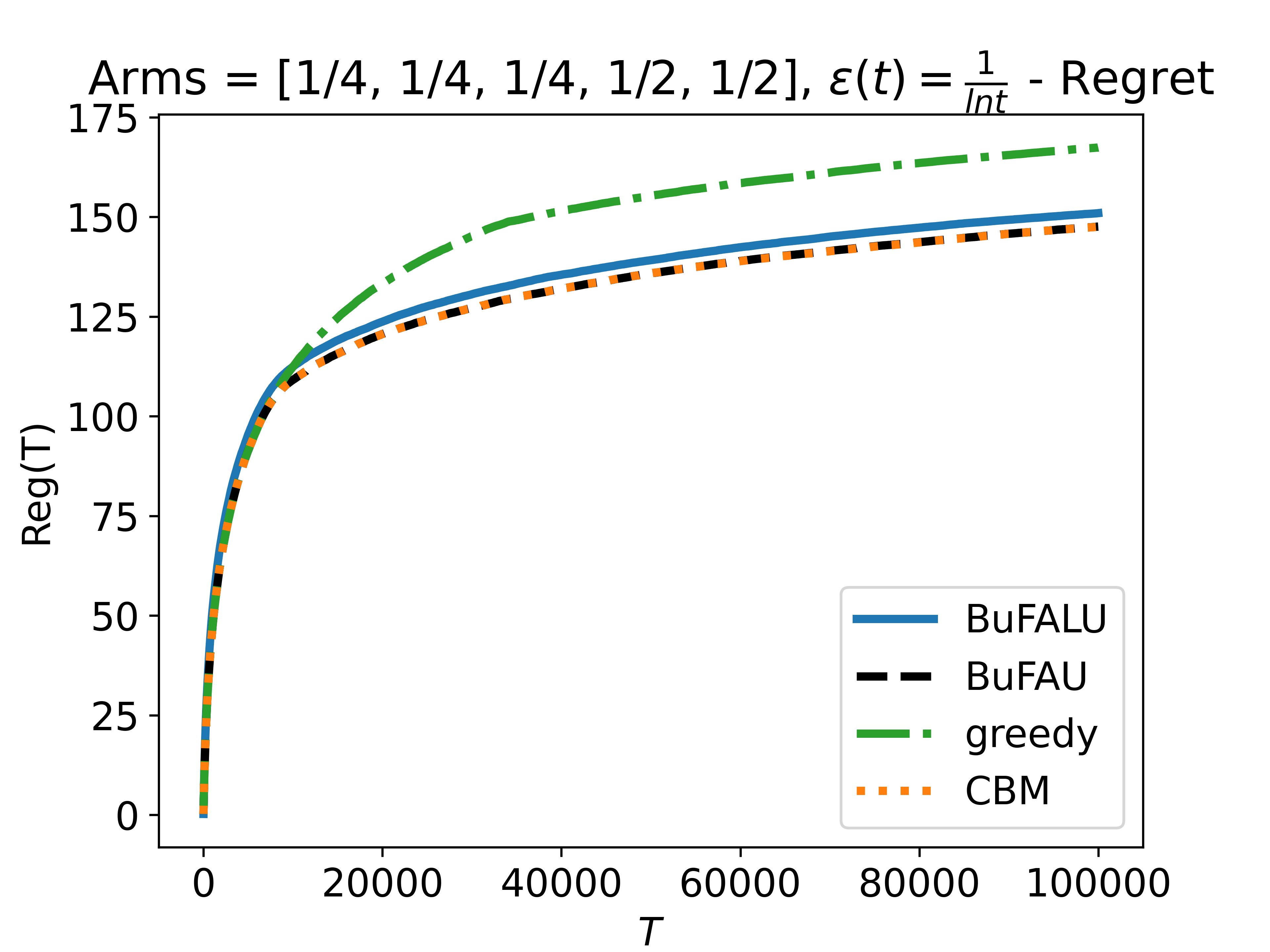}
}
\hspace{0.05\linewidth}
\subfigure{
\includegraphics[trim=0 0 0 0,clip,width=0.38\linewidth]{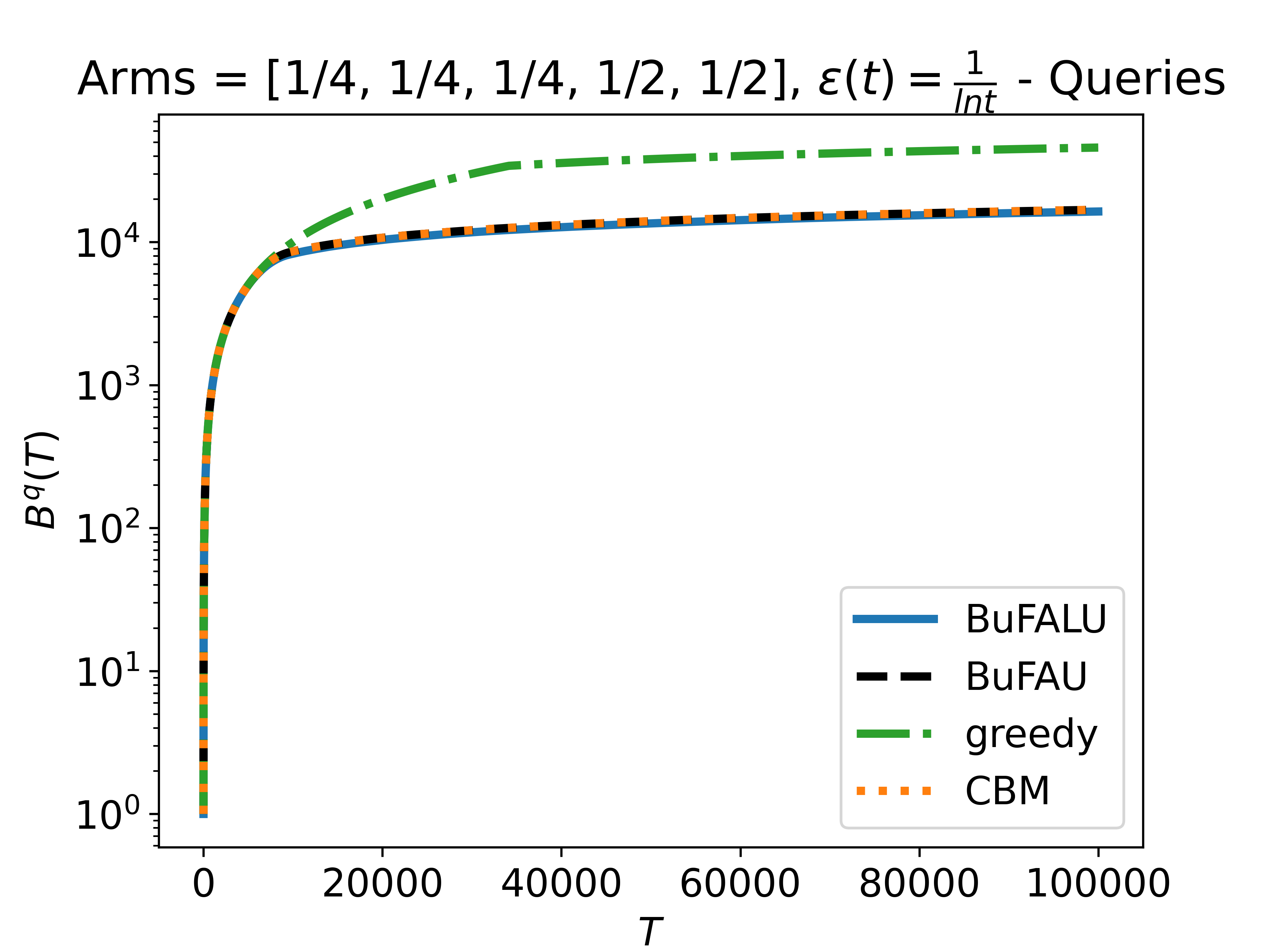}
}
\caption{Evaluation of all algorithms on 5-armed problems over $1,000$ seeds with $\epsilon(t)=\frac{1}{\ln t}$.}
\label{figure:experiments 5arms eps poly-log}
\end{figure}

\setlength{\tabcolsep}{3pt}
\begin{table}[htbp]
\centering
\caption{Additional statistics of the empirical evaluation in \Cref{figure:experiments 5arms eps poly-log}. All statistics are measured after $T=100,000$ time steps over $1,000$ different seeds. All experiments use $\epsilon(t)=\frac{1}{\ln t}$. std is the standard deviation, $90\%$ represents the $90^{th}$ percentile and max is the maximum.}
\label{table:statistics eps poly-log}
\begin{tabular}{|c|c?c|c|c|c?c|c|c|c|}\hline
  &  & \multicolumn{4}{c?}{\textbf{Regret}} &  \multicolumn{4}{c|}{\textbf{Queries}} \\ \hhline {~~--------}
   \multirow{-2}{2.5cm}{\centering\textbf{Bandit Instance}}&  \multirow{-2}{*}{\centering\textbf{Algorithm}} & mean & std & $90\%$ & max & 
    mean & std & $90\%$ & max \\ \hline 
    
  \rowcolor[gray]{0.9} \cellcolor{white}&
  BuFALU & 551.43 & 73.82 & 644.08 & 780.25 & 2881.58 & 373.37 & 3336.5 & 4072 \\ 
  
  \hhline{~>{\arrayrulecolor[gray]{.5}}|*9{-}|}
     & 
   BuFAU & 200.47 & 21.92 & 228.75 & 276.25 & 9958.88 & 87.68 & 10072 & 10262 \\ 
   
   \hhline{~>{\arrayrulecolor[gray]{.5}}|*9{-}|}
   \rowcolor[gray]{0.9} \cellcolor{white} & 
   greedy & 236.37 & 24.63 & 269.02 & 311.5 & 45785 & 0 & 45785 & 45785 \\ 
   
   \hhline{~>{\arrayrulecolor[gray]{.5}}|*9{-}|}
     \multirow{-4}{2.5cm}{\centering 5-arms, unique optimal, values: $\brs*{\frac{1}{4},\frac{1}{4},\frac{1}{4},\frac{1}{4},\frac{1}{2}}$}& 
   CBM & 200.52 & 21.94 & 228.82 & 276.25 & 9959.06 & 87.76 & 10072.3 & 10262 \\ 
  \arrayrulecolor{black}\hline
  \rowcolor[gray]{0.9} \cellcolor{white}&
   BuFALU & 150.97 & 17.9 & 173.78 & 210 & 16349.9 & 1762.22 & 18277.3 & 18954 \\ 
  
  \hhline{~>{\arrayrulecolor[gray]{.5}}|*9{-}|}
     & 
   BuFAU & 147.59 & 18.04 & 170 & 235.5 & 16811.8 & 1387.87 & 18504.5 & 19022 \\ 
   
   \hhline{~>{\arrayrulecolor[gray]{.5}}|*9{-}|}
   \rowcolor[gray]{0.9} \cellcolor{white} & 
   greedy & 167.42 & 20.59 & 194.02 & 238.25 & 45785 & 0 & 45785 & 45785 \\ 
   
   \hhline{~>{\arrayrulecolor[gray]{.5}}|*9{-}|}
     \multirow{-4}{2.5cm}{\centering 5-arms, multiple optimal, values: $\brs*{\frac{1}{4},\frac{1}{4},\frac{1}{4},\frac{1}{2},\frac{1}{2}}$}& 
   CBM & 147.57 & 17.99 & 170.5 & 235.5 & 16827.3 & 1390.31 & 18523.1 & 19022 \\ 
  \arrayrulecolor{black}\hline
\end{tabular}
\end{table}
\setlength{\tabcolsep}{6pt}

\clearpage

\subsubsection{Testing the Effect of the Total Budget}
We now test the effect of the total budget on the regret and number of feedback queries. We do so problems with Bernoulli arms, optimal arm $\mu^*=0.5$, a single suboptimal arm $\mu_a=0.25$ and either a unique or two optimal arms (namely, two and three-armed problems, respectively). The time horizon in the simulations is $T=1,000$ and results were averaged over $1,000$ seeds. To simulate the effect of the total budget, we allocated a fixed budget $B$ throughout the interactions (i.e., set $\epsilon(t) = \sqrt{\frac{6\Narms\ln t}{B}}$) and measured both the regret and number of queries at the last time step, as a function of the (fractional) allocated budget $B/T$. The results are depicted in \Cref{figure:budget effect}, and error bars represent a variation of one standard deviation.

Interestingly, the regret first decreases, achieves a minimum when $B/T\approx 0.08$, and then starts increasing. We believe that this is since the standard UCB bounds are a little loose and can be replaced by \citep{garivier2011kl}:
\begin{equation*}
    UCB_t(a) = \rEst{t-1}{a} + \sqrt{\frac{\ln t + 3\ln\ln t}{2\nq{t-1}{a}}}\enspace.
\end{equation*}
Therefore, the algorithm suffers from some over-exploration. Then, the local minimum of the regret allocates enough budget to adequately explore and then forces the algorithm the exploit, which leads to a better exploration-exploitation tradeoff.

Asides from this phenomenon, the graph behaves as could be expected. As we previously saw, all algorithms achieve roughly the same performance in the presence of multiple optimal arms. When the optimal arm is unique, BuFALU is substantially fewer budget queries but suffers from a small regret degradation (up to a factor of $4$). Notably, in the small-budget regime, BuFALU achieves similar regret to all other baselines.

\begin{figure}[h]
\centering
\subfigure{
\includegraphics[trim=0 0 0 0,clip,width=0.38\linewidth]{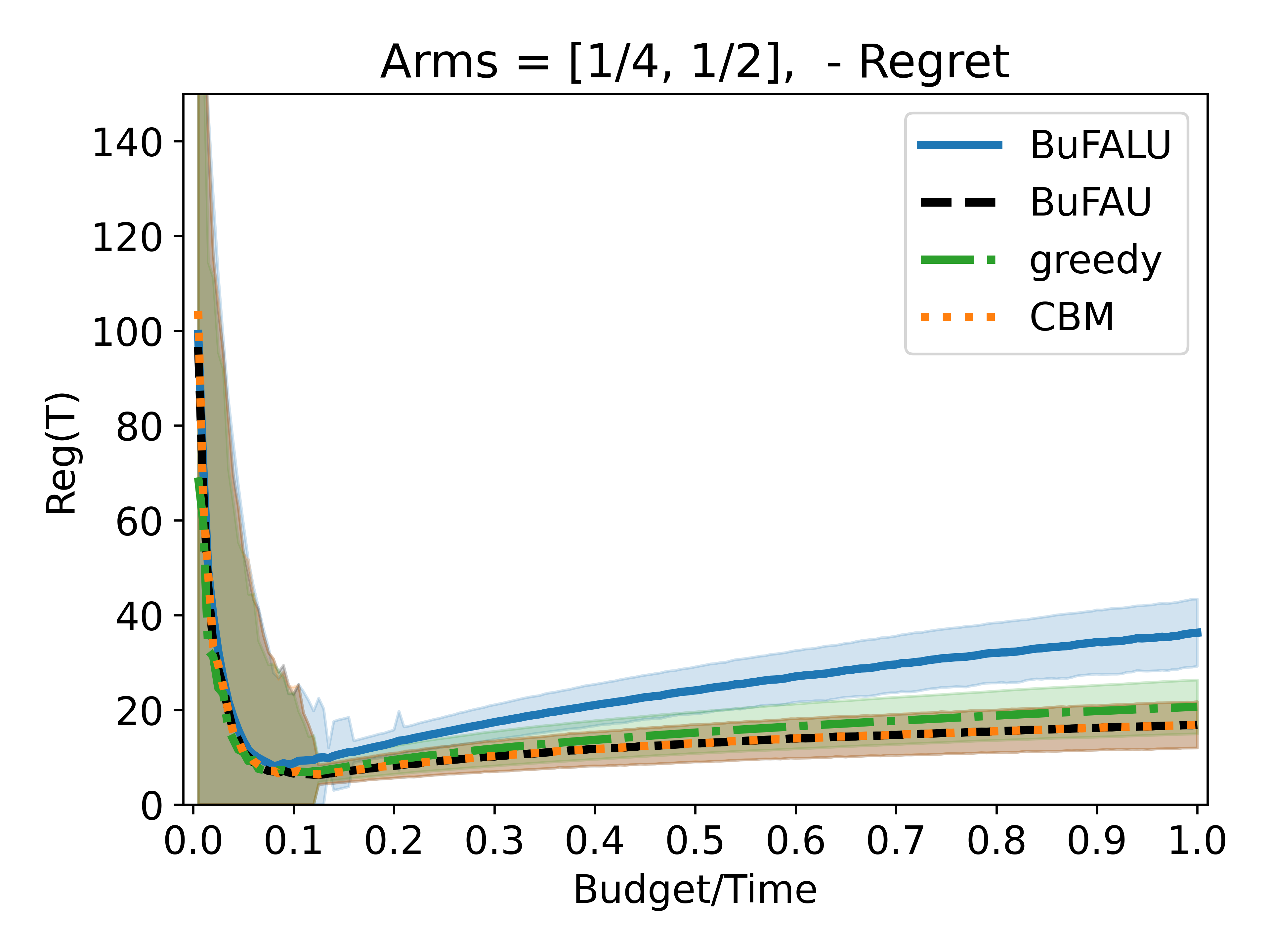}
}
\hspace{0.05\linewidth}
\subfigure{
\includegraphics[trim=0 0 0 0,clip,width=0.38\linewidth]{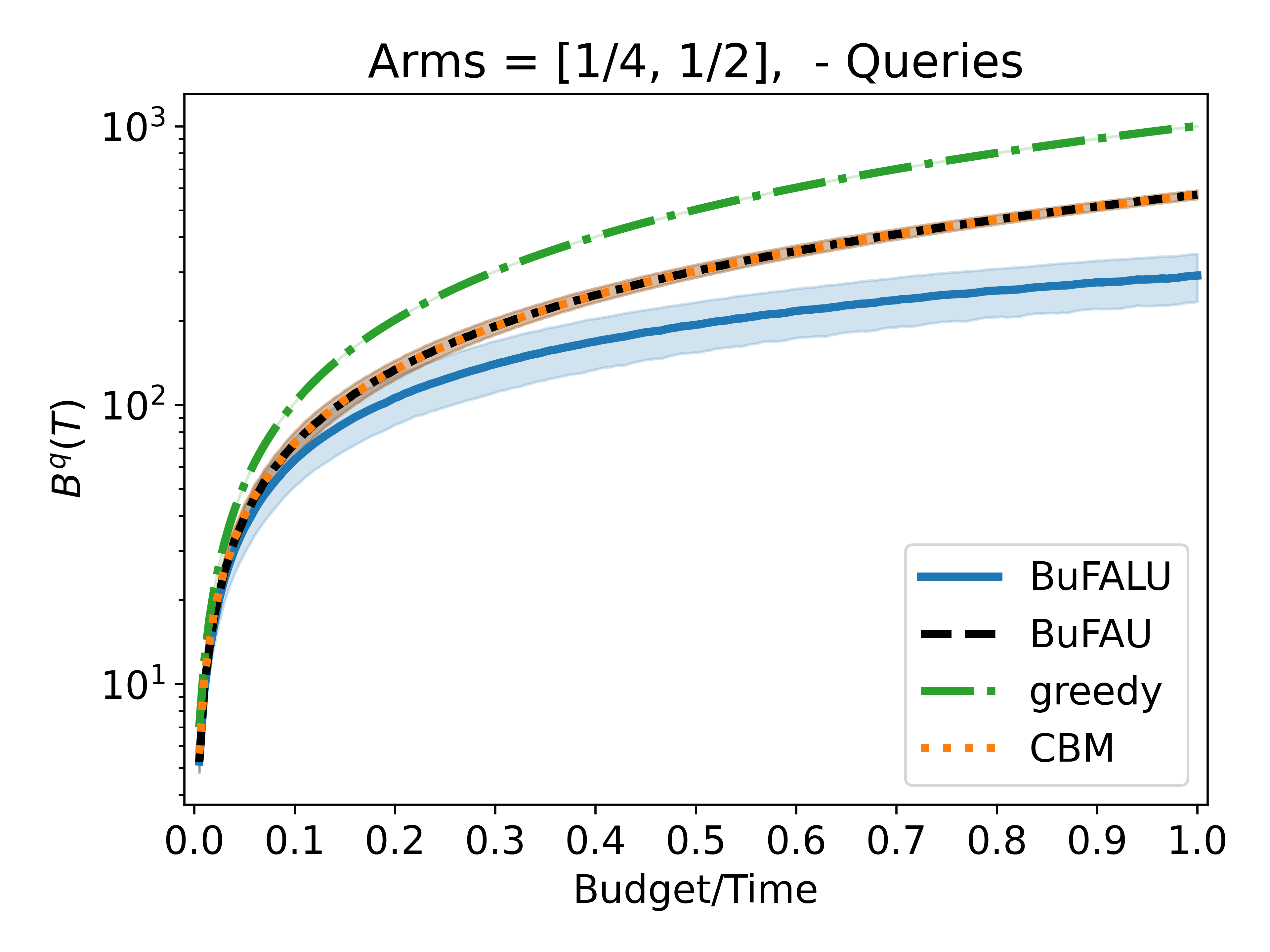}
} 
\subfigure{
\includegraphics[trim=0 0 0 0,clip,width=0.38\linewidth]{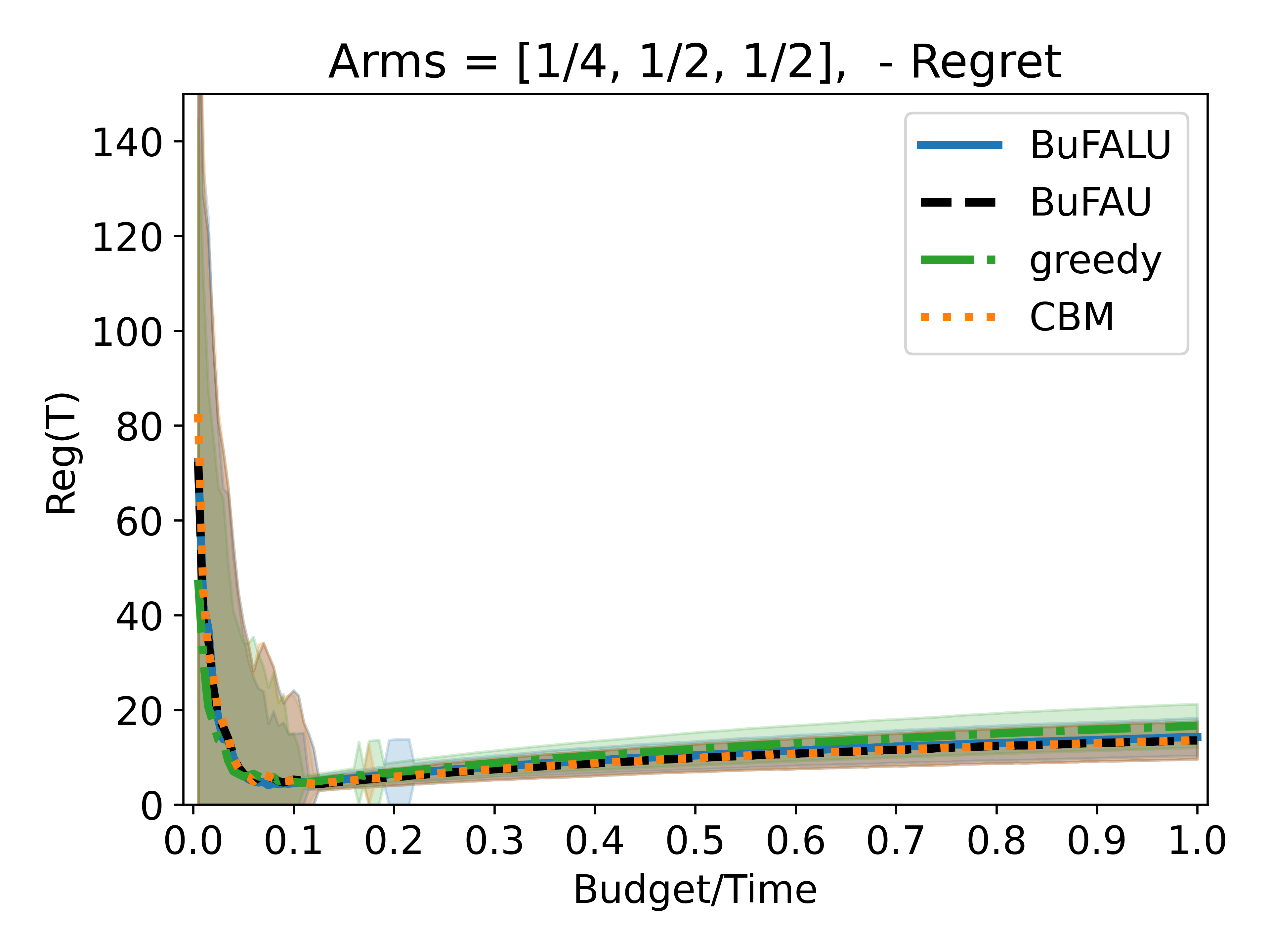}
}
\hspace{0.05\linewidth}
\subfigure{
\includegraphics[trim=0 0 0 0,clip,width=0.38\linewidth]{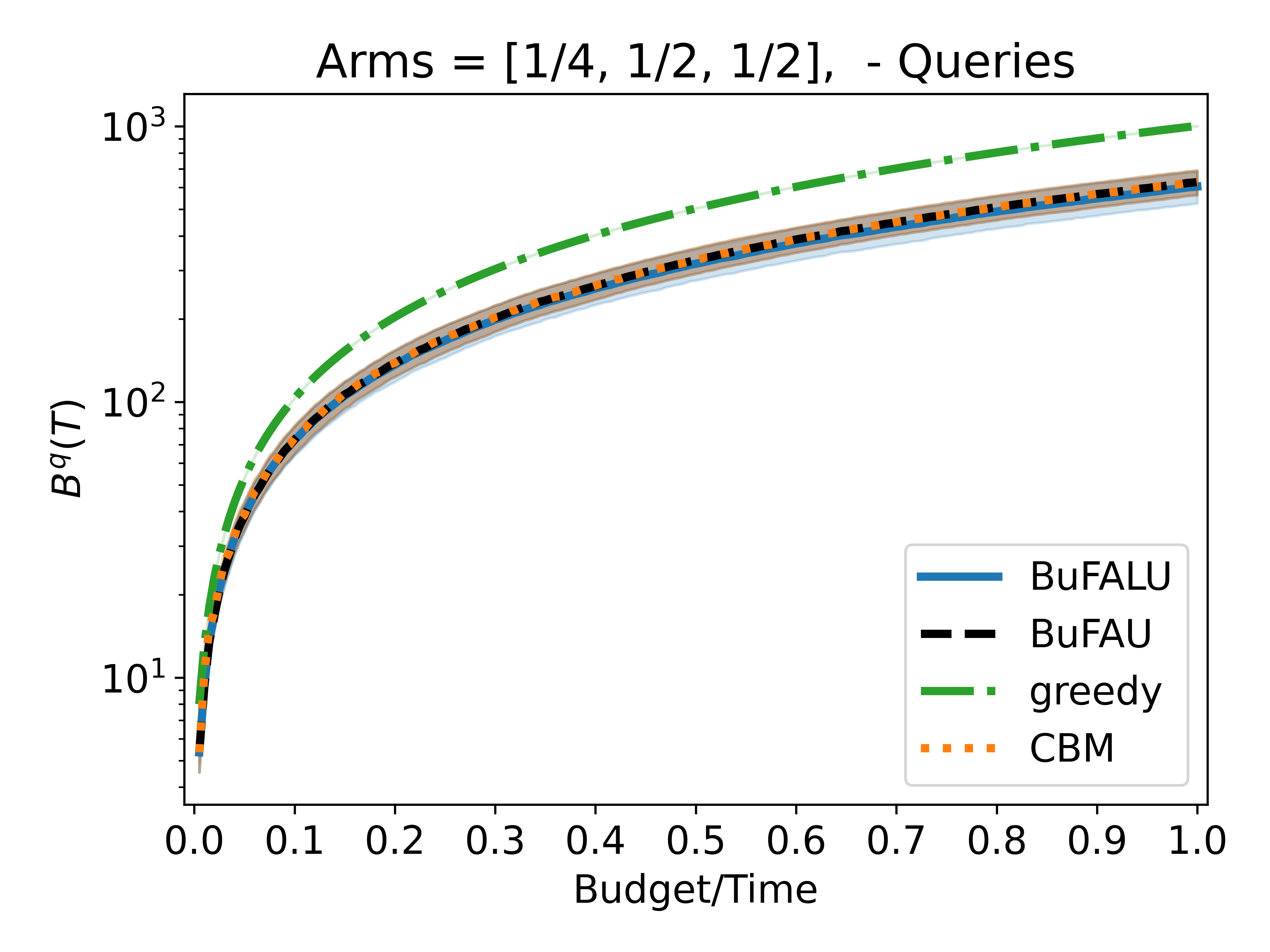}
}
\caption{Regret and Queries as a function of the allocated budget in two and three armed bandit problems.}
\label{figure:budget effect}
\end{figure}

\end{document}